\documentclass{article}
\usepackage{iclr2021/iclr2021_conference,times}
\usepackage{microtype}
\usepackage{graphicx}
\usepackage{subcaption}
\usepackage{booktabs} %
\usepackage{array} 
\newcommand{\head}[2]{\multicolumn{1}{>{\centering\arraybackslash}p{#1}}{\textbf{#2}}}
\usepackage{multirow}
\usepackage{stackengine}\setstackEOL{\cr}%

\usepackage{algorithmic}
\usepackage{algorithm}

\usepackage[utf8]{inputenc}
\usepackage[T1]{fontenc}
\usepackage[american]{babel}

\DeclareUnicodeCharacter{00A0}{~}
\usepackage{hhline}
\usepackage{mathtools}
\usepackage{amssymb}
\usepackage{amsthm}
\usepackage{graphicx}
\usepackage{url}
\usepackage{booktabs}
\usepackage{mathtools}
\usepackage{subcaption} 
\usepackage[subtle]{savetrees}
\usepackage{amssymb,amsmath,amscd,amsfonts,amsthm,bbm,mathrsfs,yhmath}
\usepackage[shortlabels]{enumitem}
\usepackage[colorlinks,linkcolor={red!80!black},citecolor={green!80!black},urlcolor={blue!80!black},hypertexnames=false]{hyperref}

\usepackage[capitalize,noabbrev]{cleveref}
\usepackage{autonum}

\crefformat{equation}{(#2#1#3)}
\crefrangeformat{equation}{(#3#1#4) to~(#5#2#6)}
\crefmultiformat{equation}{(#2#1#3)}{ and~(#2#1#3)}{, (#2#1#3)} { and~(#2#1#3)}
\crefname{appsec}{Appendix}{Appendices}

\newcommand*\diff{\mathop{}\!\mathrm{d}}

\newcommand{\R}{{{\mathbb R}}}

\newcommand{\GEBM}{{{\text{GEBM}}}} 
\newcommand{\BB}{{{\mathbb G}}} 
\newcommand{\QQ}{{{\mathbb Q}}}  
\newcommand{\PP}{{{\mathbb P}}} 
\newcommand{\A}{{{A_{\BB,E}}}} 
\newcommand{\B}{{{G}}}

\newtheorem{theorem}{Theorem}
\newtheorem{lemma}[theorem]{Lemma}

\newtheorem{proposition}[theorem]{Proposition}
\newtheorem{definition}{Definition}

\newtheorem{assumption}{Assumption}

\usepackage{xcolor}

\usepackage[noabbrev,capitalize]{cleveref}
\crefformat{equation}{(#2#1#3)}
\crefrangeformat{equation}{(#3#1#4) to~(#5#2#6)}
\crefmultiformat{equation}{(#2#1#3)}{ and~(#2#1#3)}{, (#2#1#3)}{ and~(#2#1#3)}
\crefname{appsec}{Appendix}{Appendices}

\newlist{assumplist}{enumerate}{1}
\setlist[assumplist]{label=(\textbf{\Alph*})}
\Crefname{assumplisti}{Assumption}{Assumptions}

\newlist{assumplist2}{enumerate}{6}
\setlist[assumplist2]{label=(\textbf{\Roman*})}
\Crefname{assumplist2i}{Assumption}{Assumptions}

\newlist{assumplist3}{enumerate}{6}
\setlist[assumplist3]{label=(\textbf{\roman*})}
\Crefname{assumplist3i}{Assumption}{Assumptions}

\usepackage{mathtools}

\usepackage[colorinlistoftodos]{todonotes}

\newcommand{\qdist}{\BB}
\newcommand{\pdist}{\mathbb{P}}

\title{Generalized Energy Based Models}

\author{Michael Arbel\thanks{Correspondence: \texttt{michael.n.arbel@gmail.com}.} ,  Liang Zhou \& Arthur Gretton  \\
Gatsby Computational Neuroscience Unit, University College London }

\iclrfinalcopy
 
\begin{document}

\maketitle

\begin{abstract}
We introduce the Generalized Energy Based Model (GEBM) for generative modelling. These models combine two  trained components: a base distribution (generally an implicit model), which can learn the support of data with low intrinsic dimension in a high dimensional space; and an energy function, to refine the probability mass on the learned support. 
Both the energy function and base jointly constitute the final model, unlike GANs, which retain only the base distribution (the "generator").  
GEBMs are trained by alternating between learning the energy and the base. 
We show that both training stages are well-defined: the energy is learned by maximising a generalized likelihood, and the resulting energy-based loss provides informative gradients for learning the base.
Samples from the posterior on the latent space of the trained model can be obtained via MCMC, thus finding regions in this space that produce better quality samples.
Empirically, the GEBM samples on image-generation tasks are of much better quality than those from the learned generator alone, indicating that all else being equal, the GEBM will outperform a GAN of the same complexity. When using normalizing flows as base measures, GEBMs succeed on density modelling tasks, returning comparable performance to direct maximum likelihood of the same networks.
 \end{abstract}

\section{Introduction}
Energy-based models (EBMs) have a long history in physics, statistics and machine learning \citep{LecChoHadMArFu06}. They belong to the class of \textit{explicit} models, and can be described by a family of energies $E$ which define probability distributions with density proportional to $\exp(-E)$. Those models are often known up to a normalizing constant  $Z(E)$, also called the \textit{partition function}. 
The learning task  consists of finding an optimal function that best describes a given system or target distribution $\pdist$. This can be achieved using maximum likelihood estimation (MLE), however the intractability of the normalizing partition function makes this learning task challenging. Thus, various methods have been proposed to address this \citep{Hinton:2002,Hyvarinen:2005a,Gutmann:2012,Dai:2018,Dai:2019}.  %
All these methods estimate EBMs that are supported over the whole space. In many applications, however, $\PP$ is believed to be supported on an unknown lower dimensional manifold. This happens in particular when there are strong dependencies between variables in the data \citep{Thiry:2020}, and suggests incorporating a low-dimensionality hypothesis in the model . 
 
Generative Adversarial Networks (GANs) \citep{Goodfellow:2014} are a particular way to enforce low dimensional structure in a model. They rely on an \textit{implicit} model, the generator, to produce samples supported on a low-dimensional manifold by mapping a pre-defined latent noise to the sample space using a trained function. GANs have been very successful in generating high-quality samples on various tasks, especially for unsupervised image generation \citep{Brock:2018}. The generator is trained \textit{adversarially} against a discriminator network whose goal is to distinguish  samples produced by the generator from the target data. This has inspired further research to extend the training procedure to more general losses  \citep{Nowozin:2016,Arjovsky:2017a,Li:2017a,Binkowski:2018,Arbel:2018} and to improve its stability \citep{Miyato:2018,Gulrajani:2017,Nagarajan:2017,Kodali:2017}. While the generator of a GAN has effectively a low-dimensional support, it remains challenging to refine the distribution of mass on that support using pre-defined latent noise. For instance, as shown by \cite{Cornish:2020} for normalizing flows, when the latent distribution is unimodal and the target distribution possesses multiple disconnected low-dimensional components, the generator, as a continuous map, compensates for this mismatch using steeper slopes. In practice, this implies the need for more complicated generators.

In the present work, we propose a new class of models, called \textit{Generalized Energy Based Models} (GEBMs), which can represent distributions supported on low-dimensional manifolds, while  offering more flexibility in refining the  mass on those manifolds. GEBMs combine the strength of both \textit{implicit} and \textit{explicit} models in two separate components: a base distribution (often chosen to be an implicit model) which learns the low-dimensional support of the data, and an energy function that can refine the probability mass on that learned support. 
We propose to train the GEBM by alternating between learning the energy and the base, analogous to $f$-GAN training \citep{Goodfellow:2014,Nowozin:2016}. 
The energy is learned by maximizing a generalized notion of likelihood which we relate to the \textit{Donsker-Varadhan} lower-bound \citep{Donsker:1975} and \textit{Fenchel duality}, as in \citep{nguyen2010estimating,Nowozin:2016}. 
Although the partition function is intractable in general, we propose a method to learn it in an amortized fashion without introducing additional surrogate models, as done in variational inference \citep{kingma+welling:2014,rezende+al:2014:icml} or by \citet{Dai:2018,Dai:2019}. 
The resulting maximum likelihood estimate,  the \textit{KL Approximate Lower-bound Estimate (\textsc{KALE})}, is then used as a loss for training the base. 
When the class of energies is rich and smooth enough,  we show that KALE leads to a meaningful criterion for measuring weak convergence of probabilities. Following recent work by \citet{Chu:2019,Sanjabi:2018}, we  show that KALE possesses well defined gradients w.r.t. the parameters of the base, ensuring well-behaved training. 
We also provide convergence rates for the empirical estimator of KALE when the variational family is sufficiently well behaved, %
 which may be of independent interest.

The main advantage of GEBMs
becomes clear
when sampling from these models: the posterior over the latents of the base distribution incorporates the learned energy,  putting greater mass on regions in this latent space that lead to better quality samples.  Sampling from the GEBM can thus be achieved by first sampling from the posterior distribution of the latents via MCMC in the low-dimensional latent space, then mapping those latents to the input space using the implicit map of the base. This is in contrast to standard GANs, where the latents of the base have a fixed distribution.  We focus on a class of samplers that exploit gradient information, and show that these samplers enjoy fast convergence properties by leveraging the recent work of \citet{Eberle:2018}. 
While there has been recent interest in using the discriminator to improve the quality of the generator during sampling \citep{Azadi:2019,Turner:2019a,Neklyudov:2019,Grover:2019,Tanaka:2019,Wu:2019c}, our approach emerges naturally from the model we consider.

We begin in \cref{sec:background} by introducing the GEBM  model. In \cref{sec:learning}, we describe the learning procedure using KALE, then derive a method for sampling from the learned model in \cref{sec:sampling}. In \cref{sec:related_work} we discuss related work. Finally, experimental results are presented in \cref{sec:experiments} with code available at \url{https://github.com/MichaelArbel/GeneralizedEBM}.
\section{Generalized Energy-Based Models}\label{sec:background}
\begin{figure}[H]
\centering
  	\includegraphics[width=\linewidth]{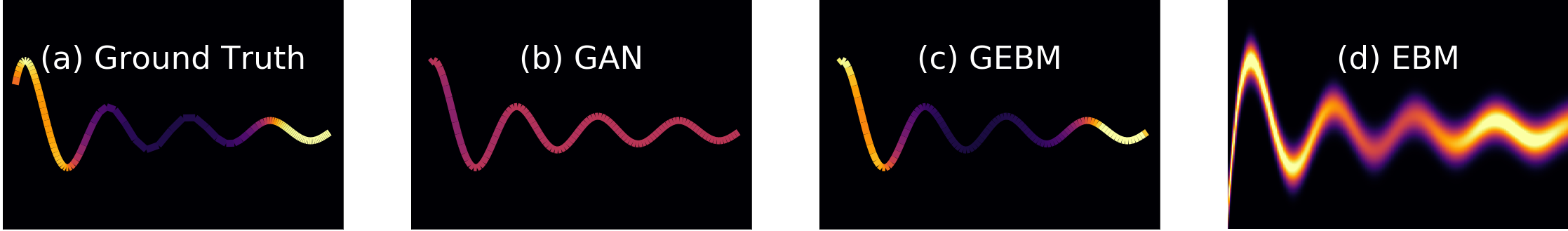}
\caption{Data generating distribution supported on a line and with higher density at the extremities. Models are learned using either a GAN, GEBM, or EBM. More details are provided in \cref{sec:illustrative_example}.} 
\label{fig:image_samples}
\end{figure}
In this section, we introduce generalized energy based models (GEBM), that combine the strengths of both energy-based models and implicit generative models, and admit the first of these as a special case.
An {\bf energy-based model} (EBM) is defined by a set $\mathcal{E}$ of real valued functions called \textit{energies}, where each $E\in\mathcal{E}$  specifies a probability density over the data space $\mathcal{X} \subset \mathbb{R}^d$ up to a normalizing constant,
\begin{align}\label{eq:energy_model}
	\mathbb{Q}(\diff x) =  \exp\left(-E(x)-A\right)\diff x,\qquad  A = \log\left(\int \exp(-E(x))\diff x  \right).
\end{align} 
While EBMs have been shown recently to be powerful models for representing complex high dimensional data distributions, they still unavoidably lead to a blurred model whenever data are concentrated on a lower-dimensional manifold.  This is the case in \cref{fig:image_samples}(a), where the ground truth distribution is supported on a 1-D line and embedded in a 2-D space. The EBM in \cref{fig:image_samples}(d) learns to give higher density to a halo surrounding the data, and thus provides a blurred representation. That is a consequence of EBM having a density defined over the whole space, and can result in blurred samples for image models.

An {\bf implicit generative model} (IGM) is a family of probability distributions $\BB_{\theta}$ parametrized by a learnable \textit{generator} function $\B:\mathcal{Z}\mapsto \mathcal{X}$ that maps latent samples $z$ from a fixed latent distribution $\eta$ to the data space $\mathcal{X}$. The latent distribution $\eta$ is required to have a density over the latent space $\mathcal{Z}$ and is often easy to sample from. Thus, Sampling from $\BB$ is simply achieved by first sampling $z$ from $\eta$ then applying $\B$,
\begin{align}\label{eq:sampling_IGM}
	x\sim \BB \quad \iff \quad x = \B(z), \quad z \sim \eta.
\end{align}
GANs are popular instances of these models, and are trained \textit{adversarially} \citep{Goodfellow:2014}. When the latent space $\mathcal{Z}$ has a smaller dimension than the input space $\mathcal{X}$, the IGM will be supported on a lower dimensional manifold
of $\mathcal{X}$, and thus will  not possess a Lebesgue density on $\mathcal{X}$ \citep{Bottou:2017}. IGMs are therefore good candidates for modelling low dimensional distributions. While GANs can accurately learn the low-dimensional support of the data, they can have limited power for representing the distribution of mass on the support. This is illustrated in \cref{fig:image_samples}(b).

A {\bf generalized energy-based model} (GEBM) $\mathbb{Q}$ is defined by a combination of a \textit{base}  $\BB$ and an \textit{energy} $E$ defined over a subset $\mathcal{X}$  of $\R^d$. The \textbf{base} component can typically be  chosen to be an IGM as in \cref{eq:sampling_IGM}. 
The \textbf{generalized energy} component can refine the mass on the support defined by the \textit{base}. It belongs to a class $\mathcal{E}$ of real valued functions defined on the input space $\mathcal{X}$, and represents the negative log-density of a sample from the GEBM  with respect to the base $\BB$,
\begin{align}\label{eq:GEBM}
	\QQ(\diff x) = \exp\left(-E(x)-\A\right) \BB(\diff x),\qquad \A = \log\left(\int \exp(-E(x))\BB(\diff x)\right),
\end{align}
where $\A$ is the logarithm of the normalizing constant of the model w.r.t. $\BB$. Thus, a GEBM $\QQ$ re-weights samples from the base according to the un-normalized importance weights $\exp(-E(x))$. 
Using the latent structure of the base $\BB$, this importance weight can be pulled-back to the latent space to define a \textit{posterior latent} distribution $\nu$,
	\begin{align}\label{eq:posterior_latent}
	 	\nu(z) := \eta(z)\exp\left(-E\left(\B(z)\right)-\A\right).
	\end{align}
Hence, the \textit{posterior latent} $\nu$ can be used instead of the latent noise $\eta$ for sampling from $\QQ$, as summarized by \cref{prop:sampling}:
\begin{proposition}\label{prop:sampling}
Sampling from $\QQ$ requires sampling a latent $z$ from $\nu$  \cref{eq:posterior_latent} then applying the map $\B$,
\begin{align}\label{eq:sampling_GEBM}
	x\sim \QQ \quad \iff \quad x = \B(z), \quad z \sim \nu.
\end{align} 
\end{proposition}
In order to hold, \cref{prop:sampling} does not need the generator $\B$ to be invertible. We provide a proof in \cref{proof:prop:sampling} which relies on a characterization of probability distribution using generalized moments. We will see later in \cref{sec:sampling} how equation \cref{eq:sampling_GEBM} can be used to provide practical sampling algorithms from the GEBM. Next we discuss the advantages of GEBMs.

{\bf Advantages of Generalized Energy Based Models.}
The GEBM defined by \cref{eq:GEBM} can be related to exponential tilting (re-weighting) \citep{Siegmund:1976,Xie:2016} of the base $\BB$. The  important difference over classical EBMs is that the base $\BB$ is allowed to change its support and shape in space. By learning the base $\BB$, GEBMs can accurately learn the low-dimensional support of data, just like IGMs do. They also benefit from the flexibility of EBMs for representing densities using an energy $E$ to refine distribution of mass on the support defined by $\BB$, as seen in \cref{fig:image_samples}(c). 

{\bf Compared to EBMs}, that put mass on the whole space by construction (positive density), 
GEBMs have the additional flexibility to concentrate the probability mass on a low-dimensional support learned by the base $\BB$, provided that the dimension of the latent space $\mathcal{Z}$ is smaller than the dimension of the ambient space $\mathcal{X}$: see \cref{fig:image_samples}(c) vs  \cref{fig:image_samples}(d). In the particular case when the dimension of $\mathcal{Z}$ is equal to the ambient dimension and $\B$ is invertible, the base $\BB$ becomes supported over the whole space $\mathcal{X}$, and GEBM recover usual EBMs. The next proposition further shows that any EBM can be viewed as a particular cases of GEBMs, as proved in \cref{proof:prop:ebm_as_gebm}.
\begin{proposition}\label{prop:ebm_as_gebm}
	Any EBM with energy $E$ (as in \cref{eq:energy_model}) can be expressed as a GEBM with  base $\BB$ given as a \textit{normalizing flow} with density $exp(-r(x))$ and a generalized energy $\tilde E(x)= E(x)-r(x)$. In this particular case, the dimension of the latent is necessarily equal to the data dimension, i.e.  $dim(\mathcal{Z}) = dim(\mathcal{X})$. 
\end{proposition} 

{\bf Compared to IGMs}, that rely on a fixed pre-determined latent noise distribution $\eta$, GEBMs offer the additional flexibility of learning a richer latent noise distribution.
This is particularly useful when the data is multimodal. In IGMs, such a GANs, the latent noise $\eta$ is usually unimodal thus requiring a more sophisticated generator to distort a unimodal noise distribution into a distribution with multiple modes, as shown by \cite{Cornish:2020}.
Instead, GEBMs allow to sample from a \textit{posterior} $\nu$ over the latent noise defined in \cref{eq:posterior_latent}. This posterior noise can be multimodal in latent space (by incorporating information from the energy) and thus can put more or less mass in specific regions of the manifold defined by the base $\BB$. This allows GEBMs  to capture multimodality in data, provided the support of the base is broad enough to subsume the data support \cref{fig:image_samples}(c). The base can be simpler, compared to GANs, as it doesn’t need to distort the input noise too much to produce multimodal samples (see \cref{fig:gen_complexity} in \cref{sec:gen_complexity}).
This additional flexibility comes at no additional training cost compared to GANs. Indeed, GANs still require another model during training, the discriminator network,  but do not use it for sampling. Instead, GEBMs avoid this waist since the base and energy can be trained jointly, with no other additional model, and then both are used for sampling. 

\section{Learning \GEBM s }\label{sec:learning} 
In this section we describe a general procedure for learning GEBMs. We decompose the learning procedure into two steps: an \textit{energy learning} step and a \textit{base learning} step. The overall learning procedure alternates between these two steps, as done in GAN  training \citep{Goodfellow:2014}.

\subsection{Energy learning}\label{sec:energy_KALE}
When the base $\BB$ is fixed, varying the energy $E$ leads to a family of models that all admit a density $\exp(-E-\A)$ w.r.t. $\BB$. When the base $\BB$ admits a density $\exp(-r)$ defined over the whole space, it is possible to learn the energy $E$ by maximizing the likelihood of the model $-\int (E+r)\diff\PP -\A$. However, in general $\BB$ is supported on a lower-dimensional manifold so that $r$ is ill-defined and the usual notion of likelihood cannot be used. Instead, we introduce a generalized notion of likelihood which  does not require a well defined density $\exp(-r)$ for $\BB$:
\begin{definition}[Generalized Likelihood]\label{def:generalized_likelihood}
The expected $\BB$-log-likelihood under a target distribution  $\PP$ of a GEBM model $\QQ$ with base $\BB$ and energy $E$ is defined as
	\begin{align}
	\label{eq:q_likelihood}
		\mathcal{L}_{\PP,\BB}(E) := -\int E(x) d\PP(x) - \A.
	\end{align}
\end{definition}
To provide intuitions about the generalized likelihood in \cref{def:generalized_likelihood}, we start by discussing the particular case where $KL(\mathbb{P}|| \BB)<+\infty$. We then present the training method in the general case where $\PP$ and $\BB$ might not share the same support, i.e. $KL(\mathbb{P}|| \BB)=+\infty$.

{\bf Special case of finite $KL(\mathbb{P}|| \BB)$.}
When the Kullback-Leibler divergence between $\PP$ and $\BB$ is well defined, \cref{eq:q_likelihood} corresponds to the Donsker-Varadhan (DV) lower bound on the KL \citep{Donsker:1975}, meaning that $ \text{KL}(\PP||\BB) \geq  \mathcal{L}_{\PP,\BB}(E)$ for all $E$. Moreover, the following proposition holds:
\begin{proposition}\label{prop:kl_improvement}
	Assume that $KL(\mathbb{P}|| \BB) < +\infty $ and $0\in \mathcal{E}$. If, in addition, $E^{\star}$ maximizes \cref{eq:q_likelihood},  then: 
\begin{align}\label{eq:energy_improvement}
KL(\mathbb{P}|| \QQ )\leq KL(\mathbb{P}|| \BB ).
\end{align}
In addition, we have that $KL(\mathbb{P}|| \QQ )= 0$ when  $E^{\star}$ is the negative log-density ratio of $\PP$ w.r.t. $\BB$.
\end{proposition}
We refer to  \cref{proof:prop:kl_improvement} for a proof. According to \cref{eq:energy_improvement}, the GEBM systematically improves over the IGM defined by $\BB$, with no further improvement possible in the limit case when $\BB=\mathbb{P}$.  
Hence as long as there is an error in mass on the common support of $\PP$ and $\BB$, the GEBM improves over the base $\BB$. 

{\bf Estimating the likelihood in the General setting.}
\cref{def:generalized_likelihood} can be used to learn a maximum likelihood energy $E^{\star}$ by maximizing $\mathcal{L}_{\PP,\BB}(E)$ w.r.t. $E$ even when the $\text{KL}(\PP||\BB)$ is infinite and when $\PP$ and $\BB$ don't necessarily share the same support. Such an optimal solution is well defined whenever the set of energies is suitably constrained. This is the case if the energies are parametrized by a compact set $\Psi$ with $\psi\mapsto E_{\psi}$ continuous over $\Psi$. Estimating the likelihood is then achieved using i.i.d. samples $(X_n)_{1:N},(Y_m)_{1:M}$ from $\PP$ and $\BB$ \citep{Tsuboi:2009,Sugiyama:2012,Liu:2017b}:
\begin{align}\label{eq:DV}
	\hat{\mathcal{L}}_{\PP,\BB}(E) = -\frac{1}{N}\sum_{n=1}^N E(X_n) - \log\left(  \frac{1}{M} \sum_{m=1}^M\exp(-E(Y_m))  \right).
\end{align}
In the context of mini-batch stochastic gradient methods, however,
$M$ typically ranges from $10$ to $1000,$ which can lead to a poor estimate for the log-partition function $\A$ . Moreover, \cref{eq:DV} doesn't exploit estimates of $\A$ from previous gradient iterations. Instead, we propose an estimator which introduces a variational parameter $A\in \R$ meant to estimate $\A$ in an amortized fashion. The key idea is to exploit the convexity of the exponential which directly implies $ -\A \geq -A - \exp(-A+\A)+1$ for any $A\in \mathbb{R}$, with equality only when $A = \A$. Therefore, \cref{eq:q_likelihood} admits a lower-bound of the form 
\begin{align}\label{eq:lower_bound_likelihood}
	\mathcal{L}_{\PP,\BB}(E) \geq  -\int (E+A) \diff \PP - \int  \exp(-(E+A))\diff \BB+1 := \mathcal{F}_{\PP,\BB}(E+A), 
\end{align}
where we introduced the functional $\mathcal{F}_{\PP,\BB}$  for concision. Maximizing $\mathcal{F}_{\PP,\BB}(E+A)$ over $A$ recovers the likelihood $\mathcal{L}_{\PP,\BB}(E)$. Moreover, jointly maximizing over $E$ and $A$ yields the maximum likelihood energy $E^{\star}$ and its corresponding log-partition function $A^{\star} = A_{\BB,E^{\star}}$. This optimization is well-suited for stochastic gradient methods using the following estimator \cite{Kanamori:2011}:
\begin{align}\label{eq:estimator_variational_lower_bound}
	\hat{\mathcal{F}}_{\PP,\BB}(E+A)= -\frac{1}{N}\sum_{n=1}^N (E(X_n)+A)-\frac{1}{M} \sum_{m=1}^M \exp(-(E(Y_m)+A)) + 1.
\end{align}
\subsection{Base learning}\label{sec:horizontal_movement}
Unlike in \cref{sec:energy_KALE}, varying the base $\BB$ does not need to preserve the same support.  
Thus, it is generally not possible to use maximum likelihood methods for learning $\BB$. %
Instead, we propose to use the  generalized likelihood  \cref{eq:q_likelihood} evaluated at the optimal energy $E^{\star}$ as a meaningful loss for learning $\BB$, and refer to it as the \textit{KL Approximate Lower-bound Estimate} (KALE),
\begin{align}\label{eq:KALE}
	\text{KALE}(\PP||\BB) = \sup_{ (E ,A)\in \mathcal{E}\times \R }  \mathcal{F}_{\PP,\BB}(E + A).
\end{align}
From \cref{sec:energy_KALE}, $\text{KALE}(\PP||\BB)$ is always a lower bound on $\text{KL}(\PP,\BB)$. The bound becomes tight whenever the negative log density of $\PP$ w.r.t. $\BB$ is well-defined and belongs to $\mathcal{E}$ (\cref{sec:kale}).
Moreover, \cref{prop:kale_extension} shows that KALE is a reliable criterion for measuring convergence, and is a consequence of \cite[Theorem B.1]{Zhang:2017}, with a proof in \cref{proof:prop:kale_smooth}:
\begin{proposition}\label{prop:kale_extension} 
Assume all energies in $\mathcal{E}$ are $L$-Lipschitz and that any continuous function can be well approximated by linear combinations of energies in $\mathcal{E}$ (\cref{densness,smoothness} of \cref{sec:proof_smoothness_kale}), then $\text{KALE}(\PP||\BB)\geq 0$ with equality only if $ \PP = \BB $ and $\text{KALE}(\PP||\BB^{n}) \rightarrow 0$ iff $ \BB^{n}\rightarrow \PP$ in distribution.
\end{proposition}
The universal approximation assumption holds in particular when $\mathcal{E}$ contains feedforward networks. In fact networks with a single neuron are enough, as shown in \cite[Theorem 2.3]{Zhang:2017}. The Lipschitz assumption holds when additional regularization of the energy is enforced during training by methods such as  \textbf{spectral normalization} \citep{Miyato:2018} or additional regularization $I(\psi)$ on the energy $E_{\psi}$ such as the \textbf{gradient penalty} \citep{Gulrajani:2017} as done in \cref{sec:experiments}.

{\bf Estimating KALE.}  According to \cite{Arora:2017}, accurate finite sample estimates of divergences that result from an optimization procedures (such as in \cref{eq:KALE}) depend on the richness of the class $\mathcal{E}$; and richer energy classes can result in slower convergence.  Unlike divergences such as Jensen-Shannon, KL and the Wasserstein distance, which result from optimizing over a non-parametric and rich class of functions, KALE is restricted to a class of parametric energies $E_{\psi}$. Thus, \cite[Theorem 3.1]{Arora:2017} applies, and guarantees good finite sample estimates, provided optimization is solved accurately. In \cref{sec:convergence_rates_kale}, we  provide an analysis for the more general case where energies are not necessarily parametric but satisfy some further smoothness properties; we emphasize that our rates do not require the strong assumption that the density ratio is bounded above and below as in \citep{nguyen2010estimating}.

{\bf Smoothness of KALE.} 
Learning the base is achieved by minimizing $\mathcal{K}(\theta) :=  \text{KALE}(\PP||\BB_{\theta})$ over the set of parameters $\Theta$ of the generator $\B_{\theta}$ using first order methods \citep{Duchi:2011,Kingma:2014,Arbel:2019a}.  
 This requires $\mathcal{K}(\theta)$ to be smooth enough so that gradient methods converge to local minima and avoid instabilities during training \citep{Chu:2019}. Ensuring smoothness of losses that result from an optimization procedure, as in \cref{eq:KALE}, can be challenging. Results for the regularized Wasserstein are provided by  \cite{Sanjabi:2018}, while more general losses are considered by \cite{Chu:2019}, albeit under stronger conditions than for our setting.\cref{thm:kale_gan_convergence} shows that when $E$, $\B_{\theta}$ and their gradients are all Lipschitz then $\mathcal{K}(\theta)$ is smooth enough.We provide a proof for \cref{thm:kale_gan_convergence} in \cref{proof:prop:kale_smooth}.
\begin{theorem}\label{thm:kale_gan_convergence}
	Under \cref{assump:compactness,assump:l_smooth_critic,assump:l_smooth_gen} of \cref{sec:proof_smoothness_kale}, 
	sub-gradient methods on  $\mathcal{K}$ converge to local optima. Moreover, $\mathcal{K}$ is Lipschitz and differentiable for almost all $\theta\in \Theta$  with: 
		\begin{align}\label{eq:grad_kale} 
			\nabla \mathcal{K}(\theta) = \exp(-A_{\B_{\theta},E^{\star}})\int \nabla_x E^{\star}(\B_{\theta}(z))\nabla_{\theta}\B_{\theta}(z)  \exp(-E^{\star}(\B_{\theta}(z))) \eta(z) \diff z.
		\end{align}  
\end{theorem}
\vspace{-.7cm}
\begin{minipage}{.38\linewidth}
{\bf Estimating the gradient} in \cref{eq:grad_kale} is achieved by first optimizing over $E_{\psi}$ and $A$ using  \cref{eq:estimator_variational_lower_bound}, with additional regularization $I(\psi)$. The resulting estimators $\hat{E}^{\star}$ and $ \hat{A}^{\star}$ are plugged in \cref{eq:empirical_gradient} to estimate  $\nabla \mathcal{K}(\theta)$ using  samples $(Z_m)_{1:M}$ from $\eta$. Unlike for learning the energy $E^{\star}$, which benefits from using the amortized estimator of the log-partition function, we found that using the empirical log-partition for learning the base was more stable. We summarize the training procedure in \cref{alg:kale_gan}, which alternates between learning the energy and the base in a similar fashion to \textit{adversarial training}.
\end{minipage}\hspace{.02\linewidth}
\begin{minipage}{.60\linewidth}
\begin{algorithm}[H]
	\caption{Training GEBM}\label{alg:kale_gan}
	\begin{algorithmic}[1] 
		\STATE \textbf{Input} $\PP$,  $N$,$M$, $n_{b}$, $n_{e}$
		\STATE \textbf{Output} Trained generator $\B_{\theta}$ and energy $E_{\psi}$. 
		\STATE \textit{Initialize $\theta$ , $\psi$ and $A$.}
		\FOR{$k=1,\dots, n_{b}$}
			\FOR{$j=1,\dots, n_{e}$}
				\STATE Sample $\{X_n\}_{1:N} \sim \PP$ and $\{Y_n\}_{1:N} \sim  \BB_{\theta}$
				\STATE $g_{\psi} \leftarrow  -\nabla_{\psi} \hat{\mathcal{F}}_{\PP,\BB_{\theta}}(E_{\psi}+A) + I(\psi)  $
				\STATE $\tilde{A}  \leftarrow \log\left(\frac{1}{M} \sum_{m=1}^M \exp(-E_{\psi}(Y_m))\right)$
				\STATE $g_A \leftarrow \exp(A-\tilde{A})-1$  
				\STATE \textit{Update $\psi$ and $A$ using $g_{\psi}$ and $g_A$.}
			\ENDFOR
			\STATE Set $\hat{E}^{\star} \leftarrow E_{\psi}$ and $\hat{A}^{\star}\leftarrow A$.
			\STATE \textit{Update $\theta$ using $\widehat{\nabla \mathcal{K}(\theta) }$ from \cref{eq:empirical_gradient}}
		\ENDFOR
	\end{algorithmic}
\end{algorithm}
\end{minipage}
\begin{align}\label{eq:empirical_gradient}
	\widehat{\nabla \mathcal{K}(\theta) } = \frac{\exp(-\hat{A}^{\star})}{M}\sum_{m=1}^M  \nabla_x \hat{E}^{\star}(\B_{\theta}(Z_m))\nabla_{\theta}\B_{\theta}(Z_m)  \exp(-\hat{E}^{\star}(\B_{\theta}(Z_m))). 
\end{align}
\section{Sampling from \GEBM s}\label{sec:sampling}
A simple estimate of the empirical  distribution of observations under the GEBM is via importance sampling (IS). This consists in first sampling multiple points from the base $\BB$, and then re-weighting the samples according to the energy $E$. Although straightforward, this approach can lead to highly unreliable estimates, a well known problem in the Sequential Monte Carlo (SMC) literature which employs IS extensively \citep{Doucet:2001,Del-Moral:2006}. Other methods such as rejection sampling are known to be inefficient in high dimensions \cite{Haugh:2017}. Instead, we propose to sample from the posterior $\nu$ using MCMC. Recall from \cref{eq:sampling_GEBM} that a sample $x$ from $\QQ$ is of the form $x=\B(z)$ with $z$ sampled from the \textit{posterior latent} $\nu$ of \cref{eq:posterior_latent} instead of the prior $\eta$. While sampling from $\eta$ is often straightforward (for instance if $\eta$ is a Gaussian), sampling from $\nu$ is generally harder, due to dependence of its density on complex functions $E$ and $G$. It is still possible to use MCMC methods to sample from $\nu$, however, since we have access to its density up to a normalizing constant \cref{eq:posterior_latent}. In particular, we are interested in methods that exploit the gradient of $\nu$, and consider two classes of samplers:   \textit{Overdamped samplers} and \textit{Kinetic samplers}.

{\bf Overdamped samplers} are obtained as a time-discretization of the \textit{Overdamped Langevin dynamics}:
\begin{align}\label{eq:overdamped_equation}
	dz_t = \left( \nabla_z \log\eta(z_t) - \nabla_z E(G(z_t)) \right) + \sqrt{2}\diff w_t,
\end{align}
where $w_t$ is a standard Brownian motion. The simplest sampler arising from \cref{eq:overdamped_equation} is the Unadjusted Langevin Algorithm (ULA):
\begin{align}\label{eq:ULA}
	Z_{k+1} = Z_{k} + \lambda\left( \nabla_z \log \eta(Z_k) -\nabla_z E(G(Z_k)) \right) + \sqrt{2\lambda}W_{k+1},\qquad Z_0\sim \eta,
\end{align}
where $(W_{k})_{k\geq 0}$ are i.i.d. standard Gaussians and $\lambda$ is the step-size. For large $k$, $Z_k$ is an approximate sample from $\nu$ \cite[Proposition 3.3]{Raginsky:2017}. Hence, setting $X=G(Z_k)$ for a large enough $k$ provides an approximate sample from the GEBM $\QQ$, as summarized in \cref{alg:overdamped_langevin} of \cref{sec:Algorithms}. 

{\bf Kinetic samplers}  arise from  the \textit{Kinetic Langevin dynamics} which introduce a momentum variable: 
\begin{align}\label{eq:latent_langevin_diffusion}
	\diff z_t = v_t\diff t, \qquad \diff v_t= -\gamma v_t \diff t + u\left(\nabla \log \eta(z_t) - \nabla E(\B(z_t))\right)\diff t + \sqrt{2\gamma u} \diff w_t.
\end{align}
with friction coefficient $\gamma\geq 0$, inverse mass $u\geq0$, \textbf{momentum} vector $v_t$ and standard Brownian motion $w_t$. When the mass $u^{-1}$ becomes negligible compared to the friction coefficient $\gamma$, i.e. $u\gamma^{-2}\approx 0$, standard results show that  \cref{eq:latent_langevin_diffusion} recovers the Overdamped dynamics \cref{eq:overdamped_equation}.
Discretization in time of \cref{eq:latent_langevin_diffusion} leads to Kinetic samplers similar to Hamiltonian Monte Carlo \citep{Cheng:2017,Sachs:2017}. We consider a particular algorithm from \cite{Sachs:2017} which we call Kinetic Langevin Algorithm (KLA) (see  \cref{alg:langevin} in \cref{sec:Algorithms}).  
Kinetic samplers were shown to better explore the modes of the invariant distribution $\nu$ compared to Overdamped ones (see \citep{Neal:2012,Betancourt:2017} for empirical results and \citep{Cheng:2017} for theory), as also confirmed empirically in \cref{sec:image_generation_appendix} for image generation tasks using GEBMs. Next, we provide the following convergence result:%
\begin{proposition}
\label{prop:convergence_langevin}
	Assume that $\log \eta(z)$ is strongly concave and has a Lipschitz gradient, that $E$, $\B$ and their gradients are all $L$-Lipschitz. Set $x_t=\B(z_t)$, where $z_t$ is given by \cref{eq:latent_langevin_diffusion} and call $\PP_t$ the probability distribution of $x_t$. Then  $\PP_t$ converges to $\QQ$ in the Wasserstein sense, 
	\begin{align}
		W_2(\PP_t,\QQ)\leq L C e^{-c\gamma t} , 
	\end{align}
	where $c$ and $C$ are positive constants independent of $t$, with $c = O(\exp(-dim(\mathcal{Z}))) $.
\end{proposition}
\cref{prop:convergence_langevin} is proved in \cref{poof:prop:convergence_langevin} using \cite[Corollary 2.6]{Eberle:2018}, and implies that $(x_t)_{t\geq 0}$ converges at the same speed as $(z_t)_{t\geq 0 }$. 
When the dimension $q$ of $\mathcal{Z}$ is orders of magnitude smaller than the input space dimension $d$, the process $(x_t)_{t\geq 0}$ converges faster than typical sampling methods on $\mathcal{X}$, for which the exponent controlling the convergence rate is of order $O(\exp(-d))$.
\section{Related work}\label{sec:related_work}
{\bf Energy based models.}
Usually, energy based models are required to have a density w.r.t. to a Lebesgue measure, and do not use a learnable base measure; in other words, models are supported on the whole space. Various methods have been proposed in the literature to learn EBMs. 
\textit{Contrastive Divergence} \citep{Hinton:2002} approximates the gradient of the log-likelihood by sampling from the energy model with MCMC. More recently, \citep{Belanger:2016,Xie:2016,Xie:2017,Xie:2018a,Xie:2019,Tu:2018,Du:2019,Deng:2020} extend the idea using more sophisticated models and MCMC sampling strategies that lead to higher quality estimators.
\textit{Score Matching} \citep{Hyvarinen:2005a} calculates an alternative objective (the \textit{score}) to the log-likelihood which is independent of the partition function, and was recently used in the context non-parametric energy functions to provide estimators of the energy that are provably consistent \citep{Sriperumbudur:2013,Sutherland:2018,Arbel:2018a,Wenliang:2018}. 
In \textit{Noise-Contrastive Estimation} \citep{Gutmann:2012}, a classifier is trained to distinguish between samples from a fixed proposal distribution and the target $\PP$. This provides an estimate for the density ratio between the optimal energy model and the proposal distribution. 
In a similar spirit, \citet{Cranmer:2016} uses a classifier to learn likelihood ratios. Conversely, \citet{Grathwohl:2020} interprets the logits of a classifier as an energy model obtained after marginalization over the classes.  The resulting model is then trained using Contrastive Divergence.
In more recent work, \cite{Dai:2018,Dai:2019} exploit a dual formulation of the logarithm of the partition function as a supremum over the set of all probability distributions of some functional objective. %
\cite{Yu:2020} explore methods for using general f-divergences, such as Jensen-Shannon, to train EBMs.

{\bf Generative Adversarial Networks.}
Recent work proposes using the discriminator of a trained GAN to improve the generator quality. Rejection sampling \citep{Azadi:2019}  and  Metropolis-Hastings correction \citep{Turner:2019a,Neklyudov:2019} perform sampling directly on the high-dimensional input space without using gradient information provided by the discriminator. 
Moreover, the data distribution is assumed to admit a density w.r.t. the generator.
\cite{Ding:2019} perform sampling on the feature space of some auxiliary pre-trained network; while  \cite{Lawson:2020} treat the sampling procedure  as a model on its own, learned by maximizing the ELBO. In our case, no auxiliary model is needed.
In the present work, sampling doesn't interfere with training, in contrast to recently considered methods to optimize over the latent space during training \cite{Wu:2019c,Wu:2019b}. In \cite{Tanaka:2019}, the discriminator is viewed as an optimal transport map between the generator and the data distribution and is used to compute optimized samples from latent space.
This is in contrast to the diffusion-based sampling that we  consider. In \citep{Xie:2018b,Xie:2018c}, two independent models, a full support EBM and a generator network, are trained cooperatively using MCMC. By contrast, in the present work,  the energy and base are part of the same model, and the model support is lower-dimensional than the target space $\mathcal{X}$. 
While we do not address the mode collapse problem, \cite{Xu:2018a,Nguyen:2017} showed that KL-based losses are resilient to it thanks to the zero-avoiding property of the KL, a good sign for KALE which is derived from KL by Fenchel duality. 

The closest related approach appears in a study concurrent to the present work \citep{Che:2020}, where the authors propose to use Langevin dynamics on the latent space of a GAN generator, but with a different discriminator to ours (derived from the Jensen-Shannon divergence or a Wasserstein-based divergence). Our theory results showing the existence of the loss gradient (\cref{thm:kale_gan_convergence}), establishing weak convergence of distributions under KALE (\cref{prop:kale_extension}), and demonstrating consistency of the KALE estimator (\cref{sec:convergence_rates_kale}) should transfer to the JS and Wasserstein criteria used in that work. 
Subsequent to the present work, an alternative approach has been recently proposed, based on normalising flows, to learn both the low-dimensional support of the data and the density on this support \citep{Brehmer:2020}. This approach maximises the explicit likelihood of a  data projection  onto a learned manifold, and may be considered complementary to our approach.
\vspace{-0.3cm}
\section{Experiments}\label{sec:experiments} 
\vspace{-.6cm}
 \begin{figure}[H]
\begin{minipage}[c]{.3\textwidth}
    \includegraphics[width=1.\linewidth]{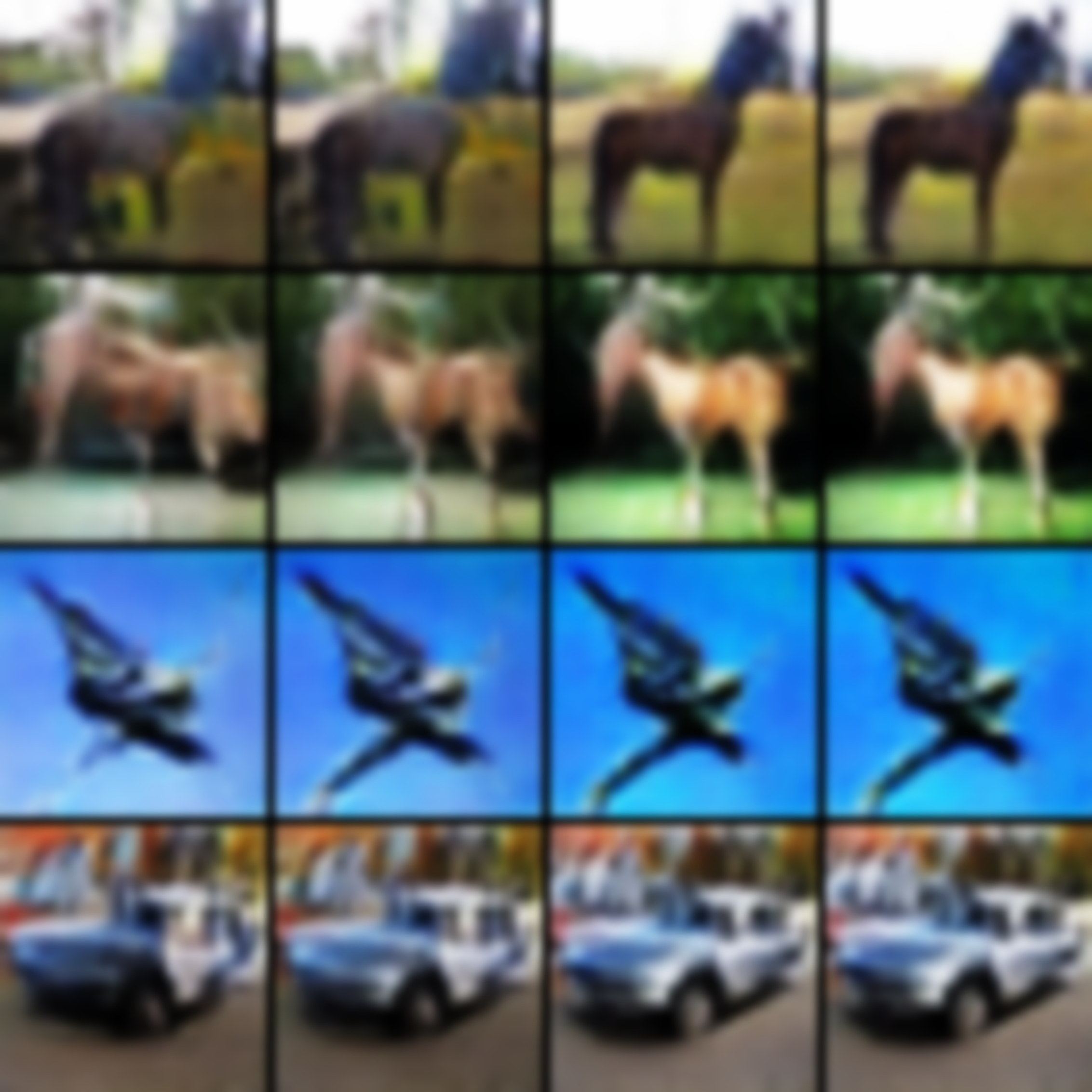}
     \captionof{figure}{\small 
    Samples at different iterations of the MCMC chain of \cref{alg:langevin} (left to right).}
     \label{fig:image_samples_MCMC}
\end{minipage}\hspace{.01\linewidth}
\begin{minipage}{.68\textwidth}
\subsection{Image generation.} 
	{\bf Experimental setting.} We train a GEBM on \textbf{unsupervised} image generation tasks, and compare the quality of generated samples with other methods using the FID score \citep{Heusel:2017} computed on $5\times 10^{4}$ generated samples. We consider \verb+CIFAR-10+ \citep{Krizhevsky:2009},  LSUN \citep{Yu:2015}, CelebA \citep{Liu:2015} and ImageNet \citep{Russakovsky:2014} all downsampled to 32x32 resolution to reduce computational cost. 
		We consider two network architectures for each of the base and energy, a smaller one (SNGAN ConvNet) and a larger one (SNGAN ResNet), both of which are from \cite{Miyato:2018}.
		For the base we used the SNGAN generator networks from \cite{Miyato:2018} with a $100$-dimensional Gaussian for the latent noise $\eta$. For the energy we used the SNGAN discriminator networks from \cite{Miyato:2018}. (Details of the networks in \cref{appendix:image_gen}).
	\end{minipage}
\end{figure}
\vspace{-.65cm}
We train the models for 150000 generator iterations using  \cref{alg:kale_gan}. 
	After training is completed, we rescale the energy by $\beta=100$ to get a \textbf{colder} version of the GEBM and  sample from it  using either  \cref{alg:overdamped_langevin} (ULA) or  \cref{alg:langevin} (KLA) with parameters $(\gamma=100, u=1)$.
         This colder temperature  leads to an improved FID score, and needs relatively few MCMC iterations, as shown in \cref{fig:sampling} of \cref{sec:image_generation_appendix}.  Sampler convergence to visually plausible  modes at low tempteratures is  demonstrated in \cref{fig:image_samples_MCMC}.
	We perform  $1000$ MCMC iterations with initial step-size of $\lambda=10^{-4}$ decreased by $10$ every $200$ iterations. As a baseline  we consider samples generated  from the base of the GEBM only (without using information from the energy) and call this  KALE-GAN. More details are given in \cref{sec:exp_details}.
	
       {\bf Results: }\cref{tab:fid_cifar10_2} shows that GEBM outperforms both KALE and standard GANs when using the same networks for the base/generator and energy/critic. Moreover, KALE-GAN matches the performance of a standard GAN (with Jensen-Shannon critic), showing that the improvement of GEBM cannot be explained by the switch from Jensen-Shannon to a KALE-based critic. Rather, the improvement is largely due to incorporating the energy function into the model, and sampling using \cref{alg:langevin}.
      
       This finding experimentally validates our claim that incorporating the energy improves the model, and that all else being equal, a GEBM outperforms a GAN with the same generator and critic architecture. Indeed, if the critic is not zero at convergence, then by definition it contains information on the remaining mismatch between the generator (base) and data mass, which the GEBM  incorporates, but the GAN does not.
       The GEBM also outperforms an EBM even when the latter was trained using a larger network (ResNet) with supervision (S) on ImageNet, which is an easier task ( \cite{Chen:2019c}). More comparisons on Cifar10 and ImageNet are provided in \cref{tab:further_exps} of \cref{sec:image_generation_appendix}.

\begin{table}[h]
\center
\resizebox{.9\width}{!}{
\begin{tabular}{cccc|cccc}
\midrule
\multirow{2}{*}{} & \multicolumn{3}{c|}{SNGAN (ConvNet)}   & \multicolumn{4}{c}{SNGAN (ResNet)}  \\ \cmidrule{2-8}
                  & GEBM       & KALE-GAN   & GAN  &   GEBM & KALE-GAN & GAN    &  EBM  \\\midrule
Cifar10     & 23.02      & 32.03      & 29.9 & \textbf{19.31} & 20.19 & $21.7$  & 38.2                                        \\ 
ImageNet          & \textbf{13.94} & 19.37 & 20.66    & 20.33 & 21.00 & 20.50      & 14.31 (S)                \\ \midrule 
\end{tabular} 
}
\caption{\small FID scores for two versions of SNGAN from \citep{Miyato:2018} on Cifar10  and ImageNet. GEBM:  training using \cref{alg:kale_gan} and sampling using \cref{alg:langevin}. KALE-GAN: Only the base of a GEBM is retained for sampling. GAN: training as in \citep{Miyato:2018} with $q=128$ for the latent dimension as it worked best. EBM: results from \cite{Du:2019} with {\em supervised} training on ImageNet (S).}
\label{tab:fid_cifar10_2}
\end{table} 

\vspace{-.5cm}
\begin{minipage}{.39\textwidth}
	\cref{tab:fid_cifar10} shows different sampling methods using the same trained networks (generator and critic), with KALE-GAN as a baseline. All energy-exploiting methods outperform the unmodified KALE-GAN with the same architecture.  That said, our method (both ULA and KLA) outperforms both (IHM) \citep{Turner:2019a} and (DOT) \citep{Tanaka:2019}, which both use the energy information.
\end{minipage}\hspace{.01\linewidth}
\begin{minipage}{.6\textwidth}
\begin{center}
	\label{fig:one}
\resizebox{.9\width}{!}{
\begin{tabular}{lcccc} 
\toprule
             &Cifar10 & LSUN & CelebA & ImageNet  \\
\midrule
KALE-GAN         & 32.03	& 21.67	& 6.91	& 19.37 \\
IHM 			& 30.47	& 20.63	& 6.39	& 18.15 \\
DOT 		& 26.35 & 20.41 & 5.93  & 16.21 \\
GEBM (ULA) 	& $\mathbf{23.02}$ & $16.23$	& $\textbf{5.21}$	& $14.00$\\
GEBM (KLA) & $24.29$ & $\mathbf{15.25}$ 	& $5.38$	& $\mathbf{13.94}$\\
\bottomrule 
\end{tabular}
}
\captionof{table}{\small FID scores for different sampling methods using the same trained SNGAN (ConvNet):  KALE-GAN as a baseline w/o critic information.}%
\label{tab:fid_cifar10}
\end{center}
\end{minipage}

In \cref{tab:fid_cifar10}, KLA was used in the high friction regime $\gamma=100$ and thus behaves like ULA. This allows to obtain sharper samples concentrated around the modes of the GEBM thus improving the FID score. %
If, instead, the goal is to encourage more exploration of the modes of the GEBM, then KLA with a smaller $\gamma$ is a better alternative than  ULA, as the former can explore multiple modes/images within the same MCMC chain, unlike (ULA): see \cref{fig:image_samples_lmc_1,fig:image_samples_lmc_2,fig:image_samples_all} of \cref{sec:image_generation_appendix}. Moving from one mode to another results in an increased FID score while between modes, however, which can be avoided by decreasing $\lambda$.

\subsection{Density Estimation}
 {\bf Motivation.} We next consider the particular setting where the likelihood of the model is well-defined, and admits a closed form expression. This is intended principally as a sanity check that our proposed training method in \cref{alg:kale_gan} succeeds in learning maximum likelihood solutions. 
Outside of this setting, closed form expressions of the normalizing constant are not available for generic GEBMs. While this is not an issue (since the proposed method doesn’t require a closed form expression for the normalizing constant), in this experiment only, we want to have access to closed form expressions, as they enable a direct comparison with other density estimation methods. 

{\bf Experimental setting.} To have a closed-form likelihood, we consider the case where the dimension of the latent space is equal to data-dimension, and choose the base $\BB$ of the GEBM to be a Real NVP (\cite{Ding:2019} ) with density $\exp(-r(x))$ and energy $E(x) = h(x)-r(x)$.
Thus, in this particular case, the GEBM has a well defined likelihood over the whole space, and we are precisely in the setting of \cref{prop:ebm_as_gebm}, which shows that this GEBM is equal to an EBM with density proportional to $exp(-h)$. We further require the EBM to be a second Real NVP so that its density has a closed form expression. 
We consider 5 UCI datasets \citep{Dheeru:2017} for which we use the same pre-processing as in \citep{Wenliang:2018}. For comparison, we train the EBM by direct maximum likelihood (ML) and contrastive divergence (CD). To train the GEBM,  we use \cref{alg:kale_gan}, which doesn't  directly exploit the closed-form expression of the likelihood (unlike direct ML). We thus use either \cref{eq:DV} (KALE-DV) or \cref{eq:estimator_variational_lower_bound} (KALE-F) to estimate the normalizing constant. More details are given in \cref{sec:density_estimation}.

{\bf Results.}
\cref{tab:UCI_dataset} reports the Negative Log-Likelihood (NLL) evaluated on the test set and corresponding to the best performance on the validation set. Training the GEBM using \cref{alg:kale_gan} leads to comparable performance to (CD) and (ML). As shown in \cref{fig:UCI} of \cref{sec:density_estimation_appendix},  (KALE-DV) and (KALE-F) maintain a small error gap between the training and test NLL and, as discussed in \cref{sec:energy_KALE,sec:Algorithms}, (KALE-F) leads to more accurate estimates of the log-partition function, with a relative error of order $0.1\%$ compared to $10\%$ for (KALE-DV).
\begin{table*}[h] 
\let\center\empty
\let\endcenter\relax
\centering 
\resizebox{.9\width}{!}{ %
\begin{tabular}{llllll}
\toprule
          &  \head{2.0cm}{RedWine   \small $ d=11, N\sim 10^3$} &  \head{2.0cm}{Whitewine \small $ d=11, N\sim 10^3$} & \head{2.0cm}{Parkinsons \small$d=15, N\sim 10^3$} & \head{2.0cm}{Hepmass \small $d=22, N\sim 10^5$} & \head{2.0cm}{Miniboone  \small $d=43, N\sim 10^4$} \\
\midrule
 \textbf{NVP w ML} &                                      11.98 &                                       13.05 &                                         14.5 &                                      24.89 &                                        42.28 \\
     \textbf{NVP w CD} &                                      11.88 &                                       13.01 &                                        14.06 &                                      $\mathbf{22.89}$ &                                        39.36 \\
    \textbf{NVP w KALE (DV)} &                                       11.6 &                                       12.77 &                                        $\mathbf{13.26}$ &                                      26.56 &                                        46.48 \\
     \textbf{NVP w KALE (F)} &                                      $\mathbf{11.19}$ &                                       $\mathbf{12.66}$ &                                        $\mathbf{13.26}$ &                                      24.66 &                                        $\mathbf{38.35}$ \\
\bottomrule
\end{tabular}

}
\caption{\small UCI datasets: Negative log-likelihood  computed on the test set and corresponding to the best performance on the validation set. Best method in boldface.}   
\label{tab:UCI_dataset} 
\end{table*}
\vspace{-.5cm}
\section{Acknowledgments}
\vspace{-.3cm}
We thank Mihaela Rosca for insightful discussions and Song Liu, Bo Dai and Hanjun Dai for pointing us to important related work.  

\clearpage

 \subsubsection*{References} 
\renewcommand\refname{\vskip -1cm}
\bibliographystyle{apalike}
\bibliography{references}

\begin{thebibliography}{}

\bibitem[Arbel and Gretton, 2018]{Arbel:2018a}
Arbel, M. and Gretton, A. (2018).
\newblock Kernel {Conditional} {Exponential} {Family}.
\newblock In {\em International {Conference} on {Artificial} {Intelligence} and
  {Statistics}}, pages 1337--1346.

\bibitem[Arbel et~al., 2019]{Arbel:2019a}
Arbel, M., Gretton, A., Li, W., and Montufar, G. (2019).
\newblock Kernelized {Wasserstein} {Natural} {Gradient}.

\bibitem[Arbel et~al., 2018]{Arbel:2018}
Arbel, M., Sutherland, D., Binkowski, M., and Gretton, A. (2018).
\newblock On gradient regularizers for mmd gans.
\newblock In {\em Advances in Neural Information Processing Systems 31}. Curran
  Associates, Inc.

\bibitem[Arjovsky et~al., 2017]{Arjovsky:2017a}
Arjovsky, M., Chintala, S., and Bottou, L. (2017).
\newblock {W}asserstein generative adversarial networks.
\newblock In {\em Proceedings of the 34th International Conference on Machine
  Learning}, volume~70 of {\em Proceedings of Machine Learning Research},
  International Convention Centre, Sydney, Australia. PMLR.

\bibitem[Arora et~al., 2017]{Arora:2017}
Arora, S., Ge, R., Liang, Y., Ma, T., and Zhang, Y. (2017).
\newblock Generalization and equilibrium in generative adversarial nets
  ({GAN}s).
\newblock In {\em Proceedings of the 34th International Conference on Machine
  Learning}, volume~70 of {\em Proceedings of Machine Learning Research}, pages
  224--232. PMLR.

\bibitem[Azadi et~al., 2019]{Azadi:2019}
Azadi, S., Olsson, C., Darrell, T., Goodfellow, I., and Odena, A. (2019).
\newblock Discriminator rejection sampling.
\newblock In {\em International Conference on Learning Representations}.

\bibitem[Belanger and McCallum, 2016]{Belanger:2016}
Belanger, D. and McCallum, A. (2016).
\newblock Structured prediction energy networks.
\newblock In {\em International Conference on Machine Learning}, pages
  983--992.

\bibitem[Betancourt et~al., 2017]{Betancourt:2017}
Betancourt, M., Byrne, S., Livingstone, S., and Girolami, M. (2017).
\newblock The geometric foundations of {Hamiltonian} {Monte} {Carlo}.
\newblock {\em Bernoulli}, 23(4A):2257--2298.

\bibitem[Bi{\'n}kowski et~al., 2018]{Binkowski:2018}
Bi{\'n}kowski, M., Sutherland, D.~J., Arbel, M., and Gretton, A. (2018).
\newblock Demystifying {MMD} {GAN}s.
\newblock In {\em International Conference on Learning Representations}.

\bibitem[Bottou et~al., 2017]{Bottou:2017}
Bottou, L., Arjovsky, M., Lopez-Paz, D., and Oquab, M. (2017).
\newblock Geometrical insights for implicit generative modeling.
\newblock In {\em Braverman Readings in Machine Learning}.

\bibitem[Brehmer and Cranmer, 2020]{Brehmer:2020}
Brehmer, J. and Cranmer, K. (2020).
\newblock Flows for simultaneous manifold learning and density estimation.
\newblock {\em arXiv preprint arXiv:2003.13913}.

\bibitem[Brock et~al., 2018]{Brock:2018}
Brock, A., Donahue, J., and Simonyan, K. (2018).
\newblock Large scale gan training for high fidelity natural image synthesis.
\newblock {\em arXiv preprint arXiv:1809.11096}.

\bibitem[Che et~al., 2020]{Che:2020}
Che, T., Zhang, R., Sohl-Dickstein, J., Larochelle, H., Paull, L., Cao, Y., and
  Bengio, Y. (2020).
\newblock Your {GAN} is secretly an energy-based model and you should use
  discriminator driven latent sampling.

\bibitem[Chen et~al., 2019]{Chen:2019c}
Chen, T., Zhai, X., Ritter, M., Lucic, M., and Houlsby, N. (2019).
\newblock Self-supervised gans via auxiliary rotation loss.
\newblock In {\em Proceedings of the IEEE Conference on Computer Vision and
  Pattern Recognition}, pages 12154--12163.

\bibitem[Cheng et~al., 2017]{Cheng:2017}
Cheng, X., Chatterji, N.~S., Bartlett, P.~L., and Jordan, M.~I. (2017).
\newblock Underdamped langevin mcmc: A non-asymptotic analysis.
\newblock {\em arXiv preprint arXiv:1707.03663}.

\bibitem[Chu et~al., 2020]{Chu:2019}
Chu, C., Minami, K., and Fukumizu, K. (2020).
\newblock Smoothness and stability in gans.
\newblock In {\em International Conference on Learning Representations}.

\bibitem[Cornish et~al., 2020]{Cornish:2020}
Cornish, R., Caterini, A.~L., Deligiannidis, G., and Doucet, A. (2020).
\newblock Relaxing bijectivity constraints with continuously indexed
  normalising flows.

\bibitem[Cranmer et~al., 2016]{Cranmer:2016}
Cranmer, K., Pavez, J., and Louppe, G. (2016).
\newblock Approximating likelihood ratios with calibrated discriminative
  classifiers.

\bibitem[Dai et~al., 2019a]{Dai:2018}
Dai, B., Dai, H., Gretton, A., Song, L., Schuurmans, D., and He, N. (2019a).
\newblock Kernel exponential family estimation via doubly dual embedding.
\newblock In {\em Proceedings of Machine Learning Research}, volume~89 of {\em
  Proceedings of Machine Learning Research}, pages 2321--2330. PMLR.

\bibitem[Dai et~al., 2019b]{Dai:2019}
Dai, B., Liu, Z., Dai, H., He, N., Gretton, A., Song, L., and Schuurmans, D.
  (2019b).
\newblock Exponential {Family} {Estimation} via {Adversarial} {Dynamics}
  {Embedding}.
\newblock {\em arXiv:1904.12083 [cs, stat]}.
\newblock arXiv: 1904.12083.

\bibitem[Davis and Drusvyatskiy, 2018]{Davis:2018}
Davis, D. and Drusvyatskiy, D. (2018).
\newblock Stochastic subgradient method converges at the rate
  \${O}(k{\textasciicircum}\{-1/4\})\$ on weakly convex functions.
\newblock {\em arXiv:1802.02988 [cs, math]}.

\bibitem[Del~Moral et~al., 2006]{Del-Moral:2006}
Del~Moral, P., Doucet, A., and Jasra, A. (2006).
\newblock Sequential monte carlo samplers.
\newblock {\em Journal of the Royal Statistical Society: Series B (Statistical
  Methodology)}, 68(3):411--436.

\bibitem[Deng et~al., 2020]{Deng:2020}
Deng, Y., Bakhtin, A., Ott, M., Szlam, A., and Ranzato, M. (2020).
\newblock Residual energy-based models for text generation.
\newblock {\em arXiv preprint arXiv:2004.11714}.

\bibitem[Dheeru and Taniskidou, 2017]{Dheeru:2017}
Dheeru, D. and Taniskidou, E.~K. (2017).
\newblock Uci machine learning repository.

\bibitem[Ding et~al., 2019]{Ding:2019}
Ding, X., Wang, Z.~J., and Welch, W.~J. (2019).
\newblock Subsampling {Generative} {Adversarial} {Networks}: {Density} {Ratio}
  {Estimation} in {Feature} {Space} with {Softplus} {Loss}.

\bibitem[Dinh et~al., 2016]{Dinh:2016}
Dinh, L., Sohl-Dickstein, J., and Bengio, S. (2016).
\newblock Density estimation using real nvp.

\bibitem[Donahue and Simonyan, 2019]{Donahue:2019}
Donahue, J. and Simonyan, K. (2019).
\newblock Large {Scale} {Adversarial} {Representation} {Learning}.
\newblock {\em arXiv:1907.02544 [cs, stat]}.
\newblock arXiv: 1907.02544.

\bibitem[Donsker and Varadhan, 1975]{Donsker:1975}
Donsker, M.~D. and Varadhan, S. R.~S. (1975).
\newblock Asymptotic evaluation of certain markov process expectations for
  large time, i.
\newblock 28(1):1--47.
\newblock \_eprint:
  https://onlinelibrary.wiley.com/doi/pdf/10.1002/cpa.3160280102.

\bibitem[Doucet et~al., 2001]{Doucet:2001}
Doucet, A., Freitas, N.~d., and Gordon, N. (2001).
\newblock {\em Sequential {Monte} {Carlo} {Methods} in {Practice}}.
\newblock Information {Science} and {Statistics}. Springer-Verlag, New York.

\bibitem[Du and Mordatch, 2019]{Du:2019}
Du, Y. and Mordatch, I. (2019).
\newblock Implicit generation and modeling with energy based models.
\newblock In {\em Advances in Neural Information Processing Systems 32}, pages
  3608--3618. Curran Associates, Inc.

\bibitem[Duchi et~al., 2011]{Duchi:2011}
Duchi, J., Hazan, E., and Singer, Y. (2011).
\newblock Adaptive {Subgradient} {Methods} for {Online} {Learning} and
  {Stochastic} {Optimization}.
\newblock {\em Journal of Machine Learning Research}, 12(Jul):2121--2159.

\bibitem[Eberle et~al., 2017]{Eberle:2018}
Eberle, A., Guillin, A., and Zimmer, R. (2017).
\newblock Couplings and quantitative contraction rates for {Langevin} dynamics.
\newblock {\em The Annals of Probability}.

\bibitem[Ekeland and T{\'e}mam, 1999]{Ekeland:1999}
Ekeland, I. and T{\'e}mam, R. (1999).
\newblock {\em Convex {Analysis} and {Variational} {Problems}}.
\newblock Classics in {Applied} {Mathematics}. Society for Industrial and
  Applied Mathematics.

\bibitem[Feydy et~al., 2019]{Feydy:2019}
Feydy, J., S{\'e}journ{\'e}, T., Vialard, F.-X., Amari, S.-i., Trouv{\'e}, A.,
  and Peyr{\'e}, G. (2019).
\newblock Interpolating between optimal transport and mmd using sinkhorn
  divergences.
\newblock In {\em The 22nd International Conference on Artificial Intelligence
  and Statistics}, pages 2681--2690.

\bibitem[Goodfellow et~al., 2014]{Goodfellow:2014}
Goodfellow, I., Pouget-Abadie, J., Mirza, M., Xu, B., Warde-Farley, D., Ozair,
  S., Courville, A., and Bengio, Y. (2014).
\newblock Generative adversarial nets.
\newblock In {\em Advances in Neural Information Processing Systems 27}, pages
  2672--2680. Curran Associates, Inc.

\bibitem[Grathwohl et~al., 2020]{Grathwohl:2020}
Grathwohl, W., Wang, K.-C., Jacobsen, J.-H., Duvenaud, D., Norouzi, M., and
  Swersky, K. (2020).
\newblock Your classifier is secretly an energy based model and you should
  treat it like one.

\bibitem[Grover et~al., 2019]{Grover:2019}
Grover, A., Song, J., Kapoor, A., Tran, K., Agarwal, A., Horvitz, E.~J., and
  Ermon, S. (2019).
\newblock Bias correction of learned generative models using likelihood-free
  importance weighting.
\newblock In {\em Advances in Neural Information Processing Systems 32}. Curran
  Associates, Inc.

\bibitem[Gulrajani et~al., 2017]{Gulrajani:2017}
Gulrajani, I., Ahmed, F., Arjovsky, M., Dumoulin, V., and Courville, A. (2017).
\newblock Improved training of wasserstein gans.
\newblock In {\em Proceedings of the 31st International Conference on Neural
  Information Processing Systems}, Red Hook, NY, USA. Curran Associates Inc.

\bibitem[Gutmann and Hyv{\"a}rinen, 2012]{Gutmann:2012}
Gutmann, M.~U. and Hyv{\"a}rinen, A. (2012).
\newblock Noise-contrastive estimation of unnormalized statistical models, with
  applications to natural image statistics.
\newblock {\em The Journal of Machine Learning Research}, 13(null):307--361.

\bibitem[Haugh, 2017]{Haugh:2017}
Haugh, M. (2017).
\newblock Mcmc and bayesian modeling.
\newblock {\em IEOR E4703 Monte-Carlo Simulation, Columbia University}.

\bibitem[Heusel et~al., 2017]{Heusel:2017}
Heusel, M., Ramsauer, H., Unterthiner, T., Nessler, B., and Hochreiter, S.
  (2017).
\newblock Gans trained by a two time-scale update rule converge to a local nash
  equilibrium.
\newblock In {\em Advances in Neural Information Processing Systems 30}, pages
  6626--6637. Curran Associates, Inc.

\bibitem[Hinton, 2002]{Hinton:2002}
Hinton, G.~E. (2002).
\newblock Training products of experts by minimizing contrastive divergence.
\newblock {\em Neural Computation}, 14(8):1771--1800.

\bibitem[Ho and Ermon, 2016]{Ho:2016}
Ho, J. and Ermon, S. (2016).
\newblock Generative adversarial imitation learning.
\newblock In {\em Advances in neural information processing systems}, pages
  4565--4573.

\bibitem[Hyv{\"a}rinen, 2005]{Hyvarinen:2005a}
Hyv{\"a}rinen, A. (2005).
\newblock Estimation of {Non}-{Normalized} {Statistical} {Models} by {Score}
  {Matching}.
\newblock {\em The Journal of Machine Learning Research}, 6:695--709.

\bibitem[Kanamori et~al., 2011]{Kanamori:2011}
Kanamori, T., Suzuki, T., and Sugiyama, M. (2011).
\newblock $ f $-divergence estimation and two-sample homogeneity test under
  semiparametric density-ratio models.
\newblock {\em IEEE Transactions on Information Theory}, 58(2):708--720.

\bibitem[Kingma and Ba, 2014]{Kingma:2014}
Kingma, D.~P. and Ba, J. (2014).
\newblock Adam: {A} {Method} for {Stochastic} {Optimization}.
\newblock {\em arXiv:1412.6980 [cs]}.
\newblock arXiv: 1412.6980.

\bibitem[Kingma and Welling, 2014]{kingma+welling:2014}
Kingma, D.~P. and Welling, M. (2014).
\newblock Auto-encoding variational {B}ayes.
\newblock {\em ICLR}.

\bibitem[Klenke, 2008]{Klenke:2008}
Klenke, A. (2008).
\newblock {\em Probability Theory: A Comprehensive Course}.
\newblock World Publishing Corporation.

\bibitem[Kodali et~al., 2017]{Kodali:2017}
Kodali, N., Abernethy, J., Hays, J., and Kira, Z. (2017).
\newblock On {Convergence} and {Stability} of {GANs}.
\newblock {\em arXiv:1705.07215 [cs]}.
\newblock arXiv: 1705.07215.

\bibitem[Krizhevsky, 2009]{Krizhevsky:2009}
Krizhevsky, A. (2009).
\newblock Learning multiple layers of features from tiny images.
\newblock Technical report, University of Toronto.

\bibitem[Lawson et~al., 2019]{Lawson:2020}
Lawson, J., Tucker, G., Dai, B., and Ranganath, R. (2019).
\newblock Energy-inspired models: Learning with sampler-induced distributions.
\newblock In {\em Advances in Neural Information Processing Systems 32}, pages
  8501--8513. Curran Associates, Inc.

\bibitem[LeCun et~al., 2006]{LecChoHadMArFu06}
LeCun, Y., Chopra, S., Hadsell, R., Ranzato, M., and Huang, F.-J. (2006).
\newblock {\em Predicting Structured Data}, chapter A Tutorial on Energy-Based
  Learning.
\newblock MIT Press.

\bibitem[Li et~al., 2017]{Li:2017a}
Li, C.-L., Chang, W.-C., Cheng, Y., Yang, Y., and Poczos, B. (2017).
\newblock Mmd gan: Towards deeper understanding of moment matching network.
\newblock In {\em Advances in Neural Information Processing Systems 30}, pages
  2203--2213. Curran Associates, Inc.

\bibitem[Liu et~al., 2017]{Liu:2017b}
Liu, S., Bousquet, O., and Chaudhuri, K. (2017).
\newblock Approximation and {Convergence} {Properties} of {Generative}
  {Adversarial} {Learning}.

\bibitem[Liu et~al., 2015]{Liu:2015}
Liu, Z., Luo, P., Wang, X., and Tang, X. (2015).
\newblock Deep learning face attributes in the wild.

\bibitem[Milgrom and Segal, 2002]{Milgrom:2002}
Milgrom, P. and Segal, I. (2002).
\newblock Envelope {Theorems} for {Arbitrary} {Choice} {Sets}.
\newblock {\em Econometrica}, 70.

\bibitem[Miyato et~al., 2018]{Miyato:2018}
Miyato, T., Kataoka, T., Koyama, M., and Yoshida, Y. (2018).
\newblock Spectral normalization for generative adversarial networks.
\newblock In {\em International Conference on Learning Representations}.

\bibitem[Nagarajan and Kolter, 2017]{Nagarajan:2017}
Nagarajan, V. and Kolter, J.~Z. (2017).
\newblock Gradient descent gan optimization is locally stable.

\bibitem[Neal, 2010]{Neal:2012}
Neal, R.~M. (2010).
\newblock Mcmc using hamiltonian dynamics.
\newblock {\em Handbook of Markov Chain Monte Carlo}.

\bibitem[Neklyudov et~al., 2019]{Neklyudov:2019}
Neklyudov, K., Egorov, E., and Vetrov, D. (2019).
\newblock The implicit metropolis-hastings algorithm.

\bibitem[Nguyen et~al., 2017]{Nguyen:2017}
Nguyen, T., Le, T., Vu, H., and Phung, D. (2017).
\newblock Dual discriminator generative adversarial nets.
\newblock In {\em Advances in Neural Information Processing Systems}, pages
  2670--2680.

\bibitem[Nguyen et~al., 2010]{nguyen2010estimating}
Nguyen, X., Wainwright, M.~J., and Jordan, M.~I. (2010).
\newblock Estimating divergence functionals and the likelihood ratio by convex
  risk minimization.
\newblock {\em IEEE Transactions on Information Theory}, 56(11):5847--5861.

\bibitem[Nowozin et~al., 2016]{Nowozin:2016}
Nowozin, S., Cseke, B., and Tomioka, R. (2016).
\newblock f-gan: Training generative neural samplers using variational
  divergence minimization.
\newblock In {\em Advances in Neural Information Processing Systems 29}, pages
  271--279. Curran Associates, Inc.

\bibitem[Oord et~al., 2016]{Oord:2016}
Oord, A. v.~d., Kalchbrenner, N., and Kavukcuoglu, K. (2016).
\newblock Pixel recurrent neural networks.
\newblock {\em arXiv preprint arXiv:1601.06759}.

\bibitem[Ostrovski et~al., 2018]{Ostrovski:2018}
Ostrovski, G., Dabney, W., and Munos, R. (2018).
\newblock Autoregressive quantile networks for generative modeling.
\newblock {\em arXiv preprint arXiv:1806.05575}.

\bibitem[Papamakarios et~al., 2017]{Papamakarios:2017}
Papamakarios, G., Pavlakou, T., and Murray, I. (2017).
\newblock Masked autoregressive flow for density estimation.
\newblock {\em NIPS}.

\bibitem[Radford et~al., 2015]{Radford:2015a}
Radford, A., Metz, L., and Chintala, S. (2015).
\newblock Unsupervised representation learning with deep convolutional
  generative adversarial networks.
\newblock {\em arXiv preprint arXiv:1511.06434}.

\bibitem[Raginsky et~al., 2017]{Raginsky:2017}
Raginsky, M., Rakhlin, A., and Telgarsky, M. (2017).
\newblock Non-convex learning via stochastic gradient langevin dynamics: a
  nonasymptotic analysis.

\bibitem[Retherford, 1978]{Retherford:1978}
Retherford, J.~R. (1978).
\newblock Review: J. diestel and j. j. uhl, jr., vector measures.
\newblock {\em Bull. Amer. Math. Soc.}, 84(4):681--685.

\bibitem[Rezende and Mohamed, 2015]{Rezende:2015}
Rezende, D.~J. and Mohamed, S. (2015).
\newblock Variational inference with normalizing flows.
\newblock In {\em Proceedings of the 32nd {International} {Conference} on
  {International} {Conference} on {Machine} {Learning} - {Volume} 37},
  {ICML}'15, pages 1530--1538. JMLR.org.

\bibitem[Rezende et~al., 2014]{rezende+al:2014:icml}
Rezende, D.~J., Mohamed, S., and Wierstra, D. (2014).
\newblock Stochastic backpropagation and approximate inference in deep
  generative models.
\newblock In {\em ICML}, pages 1278--1286.

\bibitem[Rockafellar, 1970]{Rockafellar:1997}
Rockafellar, R.~T. (1970).
\newblock {\em Convex analysis}.
\newblock Princeton Mathematical Series. Princeton University Press, Princeton,
  N. J.

\bibitem[Russakovsky et~al., 2014]{Russakovsky:2014}
Russakovsky, O., Deng, J., Su, H., Krause, J., Satheesh, S., Ma, S., Huang, Z.,
  Karpathy, A., Khosla, A., Bernstein, M., Berg, A.~C., and Fei-Fei, L. (2014).
\newblock {ImageNet} {Large} {Scale} {Visual} {Recognition} {Challenge}.
\newblock {\em arXiv:1409.0575 [cs]}.
\newblock arXiv: 1409.0575.

\bibitem[Sachs et~al., 2017]{Sachs:2017}
Sachs, M., Leimkuhler, B., and Danos, V. (2017).
\newblock Langevin {Dynamics} with {Variable} {Coefficients} and
  {Nonconservative} {Forces}: {From} {Stationary} {States} to {Numerical}
  {Methods}.
\newblock {\em Entropy}, 19.

\bibitem[Sanjabi et~al., 2018]{Sanjabi:2018}
Sanjabi, M., Ba, J., Razaviyayn, M., and Lee, J.~D. (2018).
\newblock On the convergence and robustness of training gans with regularized
  optimal transport.
\newblock In {\em Advances in Neural Information Processing Systems 31}, pages
  7091--7101. Curran Associates, Inc.

\bibitem[Siegmund, 1976]{Siegmund:1976}
Siegmund, D. (1976).
\newblock Importance sampling in the monte carlo study of sequential tests.
\newblock {\em The Annals of Statistics}, pages 673--684.

\bibitem[Simon-Gabriel and Scholkopf, 2018]{Simon-Gabriel:2016}
Simon-Gabriel, C.-J. and Scholkopf, B. (2018).
\newblock Kernel distribution embeddings: Universal kernels, characteristic
  kernels and kernel metrics on distributions.
\newblock {\em Journal of Machine Learning Research}, 19(44):1--29.

\bibitem[Simsekli et~al., 2020]{Simsekli:2020}
Simsekli, U., Zhu, L., Teh, Y.~W., and Gurbuzbalaban, M. (2020).
\newblock Fractional {Underdamped} {Langevin} {Dynamics}: {Retargeting} {SGD}
  with {Momentum} under {Heavy}-{Tailed} {Gradient} {Noise}.
\newblock {\em arXiv:2002.05685 [cs, stat]}.
\newblock arXiv: 2002.05685.

\bibitem[Sriperumbudur et~al., 2017]{Sriperumbudur:2013}
Sriperumbudur, B., Fukumizu, K., Kumar, R., Gretton, A., and Hyv{\"a}rinen, A.
  (2017).
\newblock Density estimation in infinite dimensional exponential families.
\newblock {\em Journal of Machine Learning Research}.

\bibitem[Sugiyama et~al., 2012]{Sugiyama:2012}
Sugiyama, M., Suzuki, T., and Kanamori, T. (2012).
\newblock {\em Density ratio estimation in machine learning}.
\newblock Cambridge University Press.

\bibitem[Sutherland et~al., 2018]{Sutherland:2018}
Sutherland, D., Strathmann, H., Arbel, M., and Gretton, A. (2018).
\newblock Efficient and principled score estimation with {Nystrom} kernel
  exponential families.
\newblock In {\em International {Conference} on {Artificial} {Intelligence} and
  {Statistics}}, pages 652--660.

\bibitem[Tanaka, 2019]{Tanaka:2019}
Tanaka, A. (2019).
\newblock Discriminator optimal transport.
\newblock In {\em Advances in {Neural} {Information} {Processing} {Systems}
  32}. Curran Associates, Inc.

\bibitem[Thekumparampil et~al., 2019]{Thekumparampil:2019}
Thekumparampil, K.~K., Jain, P., Netrapalli, P., and Oh, S. (2019).
\newblock Efficient algorithms for smooth minimax optimization.
\newblock In {\em Advances in Neural Information Processing Systems 32}, pages
  12680--12691. Curran Associates, Inc.

\bibitem[Thiry et~al., 2021]{Thiry:2020}
Thiry, L., Arbel, M., Belilovsky, E., and Oyallon, E. (2021).
\newblock The unreasonable effectiveness of patches in deep convolutional
  kernels methods.
\newblock In {\em International {Conference} on {Learning} {Representations}}.

\bibitem[Tsuboi et~al., 2009]{Tsuboi:2009}
Tsuboi, Y., Kashima, H., Hido, S., Bickel, S., and Sugiyama, M. (2009).
\newblock Direct density ratio estimation for large-scale covariate shift
  adaptation.
\newblock {\em Journal of Information Processing}, 17:138--155.

\bibitem[Tu and Gimpel, 2018]{Tu:2018}
Tu, L. and Gimpel, K. (2018).
\newblock Learning approximate inference networks for structured prediction.
\newblock {\em arXiv preprint arXiv:1803.03376}.

\bibitem[Turner et~al., 2019]{Turner:2019a}
Turner, R., Hung, J., Frank, E., Saatchi, Y., and Yosinski, J. (2019).
\newblock {M}etropolis-{H}astings generative adversarial networks.
\newblock In {\em Proceedings of the 36th International Conference on Machine
  Learning}, volume~97 of {\em Proceedings of Machine Learning Research}, pages
  6345--6353, Long Beach, California, USA. PMLR.

\bibitem[Villani, 2009]{Villani:2009}
Villani, C. (2009).
\newblock Optimal transport: Old and new.
\newblock Technical report.

\bibitem[Wenliang et~al., 2019]{Wenliang:2018}
Wenliang, L., Sutherland, D., Strathmann, H., and Gretton, A. (2019).
\newblock Learning deep kernels for exponential family densities.
\newblock In {\em International {Conference} on {Machine} {Learning}}, pages
  6737--6746.

\bibitem[Wu et~al., 2019a]{Wu:2019b}
Wu, Y., Donahue, J., Balduzzi, D., Simonyan, K., and Lillicrap, T. (2019a).
\newblock {LOGAN}: {Latent} {Optimisation} for {Generative} {Adversarial}
  {Networks}.
\newblock {\em arXiv:1912.00953 [cs, stat]}.
\newblock arXiv: 1912.00953.

\bibitem[Wu et~al., 2019b]{Wu:2019c}
Wu, Y., Rosca, M., and Lillicrap, T. (2019b).
\newblock Deep compressed sensing.
\newblock In {\em Proceedings of the 36th International Conference on Machine
  Learning}, volume~97 of {\em Proceedings of Machine Learning Research}, pages
  6850--6860, Long Beach, California, USA. PMLR.

\bibitem[Xie et~al., 2018a]{Xie:2018c}
Xie, J., Lu, Y., Gao, R., and Wu, Y.~N. (2018a).
\newblock Cooperative learning of energy-based model and latent variable model
  via mcmc teaching.
\newblock In {\em AAAI}, volume~1, page~7.

\bibitem[Xie et~al., 2018b]{Xie:2018b}
Xie, J., Lu, Y., Gao, R., Zhu, S.-C., and Wu, Y.~N. (2018b).
\newblock Cooperative training of descriptor and generator networks.
\newblock {\em IEEE transactions on pattern analysis and machine intelligence},
  42(1):27--45.

\bibitem[Xie et~al., 2016]{Xie:2016}
Xie, J., Lu, Y., Zhu, S.-C., and Wu, Y. (2016).
\newblock A theory of generative convnet.
\newblock In {\em International Conference on Machine Learning}, pages
  2635--2644.

\bibitem[Xie et~al., 2018c]{Xie:2018a}
Xie, J., Zheng, Z., Gao, R., Wang, W., Zhu, S.-C., and Nian~Wu, Y. (2018c).
\newblock Learning descriptor networks for 3d shape synthesis and analysis.
\newblock In {\em Proceedings of the IEEE conference on computer vision and
  pattern recognition}, pages 8629--8638.

\bibitem[Xie et~al., 2017]{Xie:2017}
Xie, J., Zhu, S.-C., and Nian~Wu, Y. (2017).
\newblock Synthesizing dynamic patterns by spatial-temporal generative convnet.
\newblock In {\em Proceedings of the ieee conference on computer vision and
  pattern recognition}, pages 7093--7101.

\bibitem[Xie et~al., 2019]{Xie:2019}
Xie, J., Zhu, S.-C., and Wu, Y.~N. (2019).
\newblock Learning energy-based spatial-temporal generative convnets for
  dynamic patterns.
\newblock {\em IEEE transactions on pattern analysis and machine intelligence}.

\bibitem[Xu et~al., 2018]{Xu:2018a}
Xu, K., Du, C., Li, C., Zhu, J., and Zhang, B. (2018).
\newblock Learning implicit generative models by teaching density estimators.
\newblock {\em arXiv preprint arXiv:1807.03870}.

\bibitem[Yu et~al., 2015]{Yu:2015}
Yu, F., Seff, A., Zhang, Y., Song, S., Funkhouser, T., and Xiao, J. (2015).
\newblock {LSUN}: Construction of a large-scale image dataset using deep
  learning with humans in the loop.

\bibitem[Yu et~al., 2020]{Yu:2020}
Yu, L., Song, Y., Song, J., and Ermon, S. (2020).
\newblock Training deep energy-based models with f-divergence minimization.

\bibitem[Zenke et~al., 2017]{Zenke:2017}
Zenke, F., Poole, B., and Ganguli, S. (2017).
\newblock Continual learning through synaptic intelligence.
\newblock {\em Proceedings of machine learning research}, 70:3987.

\bibitem[Zhang et~al., 2017]{Zhang:2017}
Zhang, P., Liu, Q., Zhou, D., Xu, T., and He, X. (2017).
\newblock On the {Discrimination}-{Generalization} {Tradeoff} in {GANs}.
\newblock {\em arXiv:1711.02771 [cs, stat]}.
\newblock arXiv: 1711.02771.

\end{thebibliography}

\clearpage 
 
\appendix

\section{ KL Approximate Lower-bound Estimate}\label{sec:kale}
We discuss the relation between KALE \cref{eq:KALE} and  the Kullback-Leibler divergence via Fenchel duality. Recall that a distribution $\pdist$ is said to admit a density w.r.t. $\qdist$ if there exists a real-valued measurable function $ r_0$ that is integrable w.r.t. $\qdist$ and satisfies  $d\pdist =r_0 d\qdist$.  Such a density is also called the \textit{Radon-Nikodym derivative} of $\pdist$ w.r.t. $\qdist$. In this case, we have:
\begin{align}\label{eq:KL}
	 \textsc{KL}(\pdist||\qdist) = \int r_0\log( r_0 )d\qdist.
\end{align}
\cite{nguyen2010estimating,Nowozin:2016} derived a variational formulation for the KL using Fenchel duality. By the duality theorem \citep{Rockafellar:1997}, the convex and lower semi-continuous function $\zeta:u\mapsto u\log(u)$ that appears in \cref{eq:KL} can be expressed as the supremum of a concave function:
\begin{align}\label{eq:fenchel_duality}
	\zeta(u) = \sup_{v} uv - \zeta^{\star}(v).
\end{align}
The function $\zeta^{\star}$ is called the \textit{Fenchel dual} and is defined as $\zeta^{\star}(v) = \sup_{u} uv - \zeta(u)$. By convention, the value of the objective is set to $-\infty$ whenever $u$ is outside of the domain of definition of $\zeta^{\star}$.  When $\zeta(u)=u\log(u)$, the Fenchel dual $\zeta^{\star}(v)$ admits a closed form expression of the form $\zeta^{\star}(v) = \exp(v-1)$. Using the expression of $\zeta$ in terms of its Fenchel dual $\zeta^{\star}$, it is possible to express $\textsc{KL}(\pdist||\qdist)$ as the supremum of the variational objective \cref{eq:variational_objective} over all  measurable functions $h$. 
\begin{align}\label{eq:variational_objective}
	\mathcal{F}(h) :=  -\int h d\pdist - \int \exp(-h) d\qdist  +1.
\end{align}
 \cite{nguyen2010estimating} provided the variational formulation  for the reverse KL using a different choice for $\zeta$: ($\zeta(u)=-\log(u)$). We refer to \citep{Nowozin:2016} for general $f$-divergences. 
Choosing a smaller set of functions $\mathcal{H}$ in the variational objective \cref{eq:variational_objective} will lead to a lower bound on the KL.
This is the \textit{KL Approximate Lower-bound Estimate} (KALE):  
\begin{align}\label{eq:KALE_appendix}
	\textsc{KALE}(\pdist||\qdist)=  \sup_{h\in \mathcal{H}}\mathcal{F}(h)
\end{align}
In general, $\textsc{KL}(\pdist||\qdist) \geq \textsc{KALE}(\pdist||\qdist)$. The bound is tight whenever the negative log-density $h_0= -\log r_0 $ belongs to $\mathcal{H}$; however, we do not require $r_0$ to be well-defined in general.  
Equation \cref{eq:KALE_appendix} has the advantage that it can be estimated using samples from $\pdist$ and $\qdist$. Given i.i.d. samples  $(X_1,...,X_N)$ and $(Y_1,...,Y_M)$ from $\pdist$ and $\qdist$, we denote by $\hat{\pdist}$ and $\hat{\qdist}$ the corresponding empirical distributions. A simple approach to estimate $\textsc{KALE}(\pdist||\qdist)$ is to use an $M$-estimator. This is achieved by optimizing the penalized objective 
\begin{align}\label{eq:KALE_log_density}
	\hat{h} :=\arg\max_{h\in \mathcal{H}} \widehat{\mathcal{F}}(h) -\frac{\lambda}{2}I^2(h),
\end{align}
where $\widehat{\mathcal{F}}$ is an empirical version of $\mathcal{F}$ and $I^{2}(h)$ is a penalty term that prevents overfitting due to finite samples. The penalty $I^{2}(h)$ acts as a regularizer favoring smoother solutions while the parameter $\lambda$ determines the strength of the smoothing and is chosen to decrease as the sample size $N$ and $M$ increase.
The $M$-estimator of $\textsc{KALE}(\pdist||\qdist)$ is obtained simply by plugging in $\hat{h}$ into the empirical objective $ \widehat{\mathcal{F}}(h)$: 
\begin{align}\label{eq:KALE_estimator}
	\widehat{\textsc{KALE}}(\pdist||\qdist) := \widehat{\mathcal{F}}(\hat{h}).
\end{align}
We defer the consistency analysis of \cref{eq:KALE_estimator} to \cref{sec:convergence_rates_kale} where we provide convergence rates in a setting where the set of functions  $\mathcal{H}$  is a Reproducing Kernel Hilbert Space and under weaker assumptions that were not covered by the framework of \cite{nguyen2010estimating}. 

\section{Convergence rates of KALE}\label{sec:convergence_rates_kale}
In this section, we provide a convergence rate for the estimator in \cref{eq:KALE_estimator} when $\mathcal{H}$ is an RKHS. The theory remains the same whether $\mathcal{H}$ contains constants or not. With this choice, the Representer Theorem allows us to reduce the potentially infinite-dimensional optimization problem in \cref{eq:KALE_log_density} to a convex finite-dimensional one. We further restrict ourselves to the \textit{well-specified} case where the density $r_0$ of $\pdist$ w.r.t. $\qdist$ is well-defined and belongs to $\mathcal{H},$ so that $\textsc{KALE}$ matches the KL.
While \cite{nguyen2010estimating} (Theorem 3) provides a convergence rate of $1/\sqrt{N}$ for a related $M$-estimator, this requires the density $r_0$ to be lower-bounded by 0 as well as (generally) upper-bounded. This can be quite restrictive if, for instance, $r_0$ is the density ratio of two gaussians. 
In \cref{thm:consistency}, we provide a similar convergence rate for the estimator defined in \cref{eq:KALE_estimator} without requiring $r_0$ to be bounded.  
We start by briefly introducing some notations, the working assumptions and the statement of the convergence result in \cref{sec:notation} and provide the proofs in \cref{sec:kale_proofs}.
\subsection{Statement of the result}\label{sec:notation}
We recall that an RKHS $\mathcal{H}$ of functions defined on a domain $\mathcal{X}\subset \mathbb{R}^d $ and with kernel $k$ is a Hilbert space with dot product $\langle ., .\rangle$, such that $y\mapsto k(x,y)$ belongs to $\mathcal{H}$ for any $x\in \mathcal{X},$ and
\begin{align}
	k(x,y) = \langle k(x,.),k(y,.)\rangle ,\qquad \forall x,y \in \mathcal{X}.
\end{align}
Any function $h$ in $\mathcal{H}$ satisfies the reproducing property $f(x) = \langle f,k(x,.)\rangle $ for any $x\in \mathcal{X}$.

Recall that $\text{KALE}(\pdist||\qdist)$ is obtained as an optimization problem 
\begin{align}\label{eq:general_problem}
	\text{KALE}(\pdist||\qdist) = \sup_{h\in \mathcal{H}  }  \mathcal{F}(h)    
\end{align}
  where $\mathcal{F}$ is given by:
\begin{align}
	\mathcal{F}(h) := -\int h d \pdist - \int \exp(-h) d\qdist +1.
\end{align}
Since the negative log density ratio $h_0$ is assumed to belong to $\mathcal{H}$, this directly implies that the supremum of $\mathcal{F}$ is achieved at $h_0$ and $\mathcal{F}(h_0) = \textsc{KALE}(\pdist||\qdist)$.
We are interested in estimating $\text{KALE}(\pdist||\qdist)$ using the empirical distributions $\hat{\pdist}$ and $\hat{\qdist}$, 
 \begin{align}
	\hat{\pdist} := \frac{1}{N}\sum_{n=1}^N \delta_{X_n},\qquad
	\hat{\qdist} := \frac{1}{N}\sum_{n=1}^N \delta_{Y_n},
\end{align}
where $(X_n)_{1\leq n\leq N}$ and $(Y_n)_{1\leq n\leq N}$ are i.i.d. samples from $\pdist$ and $\qdist$.  For this purpose we introduce the empirical objective functional,
\begin{align}
	\widehat{\mathcal{F}}(h) := -\int h d \hat{\pdist} - \int \exp(-h) d\hat{\qdist} +1.
\end{align}
The proposed estimator is obtained by solving a regularized empirical problem,
\begin{align}\label{eq:reg_population}
	 \sup_{h\in\mathcal{H}}  \widehat{\mathcal{F}}(h) - \frac{\lambda}{2} \Vert h \Vert^2,
\end{align}
with a corresponding population version,
\begin{align}\label{eq:reg_population_2}
	\sup_{h\in \mathcal{H}} ~ \mathcal{F}(h) - \frac{\lambda}{2}\Vert h \Vert^2.  
\end{align}
Finally, we introduce $D(h,\delta)$ and $\Gamma(h,\delta)$:
\begin{align}
	D(h,\delta) &=  \int \delta \exp(-h)d\qdist -\int \delta d\pdist,\\
	\Gamma(h,\delta) &= -\int \int_{0}^{1}(1-t) \delta^2\exp(-(h+t\delta)) d\qdist.
	\end{align}
The empirical versions of $D(h,\delta)$ and $\Gamma(h,\delta)$ are denoted $\hat{D}(h,\delta)$ and $\hat{\Gamma}(h,\delta)$. Later, we will show that $D(h,\delta)$ $\hat{D}(h,\delta)$ are  in fact the gradients of $\mathcal{F}(h)$ and $\widehat{\mathcal{F}}(h)$ along the direction $\delta$.

We state now the working assumptions:
\begin{assumplist3}
	\item \label{assump:hilbert} The supremum of $\mathcal{F}$ over $\mathcal{H}$ is attained at $h_0$.
\item\label{assump:continuity_differentiable} The following quantities are finite for some positive $\epsilon$:
\begin{align}
	&\int \sqrt{k(x,x)}~d\pdist(x),\\
	& \int \sqrt{k(x,x)} \exp((\Vert h_0\Vert+\epsilon)\sqrt{k(x,x)}) ~d\qdist(x),\\
	&\int k(x,x) \exp((\Vert h_0\Vert+\epsilon)\sqrt{k(x,x)}) ~d\qdist(x).
\end{align}
	\item \label{assum:uniqueness}  For any $h\in \mathcal{H}$, if $D(h,\delta)=0$ for all $\delta$ then $h=h_0$. 
\end{assumplist3} 

\begin{theorem}\label{thm:consistency}
Fix any $1>\eta>0$. Under \cref{assump:hilbert,assump:continuity_differentiable,assum:uniqueness}, and provided that $\lambda= \frac{1}{\sqrt{N}}$, it holds with probability at least $1-2\eta$ that
\begin{align}
	\vert \widehat{\mathcal{F}}(\hat{h}) - \mathcal{F}(h_0) \vert  \leq  \frac{M'(\eta, h_0)}{\sqrt{N}}
\end{align}
	for a constant $M'(\eta,h_0)$ that depends only on $\eta$ and $h_0$.
\end{theorem}
The assumptions in \cref{thm:consistency} essentially state that
the kernel associated to the RKHS $\mathcal{H}$ needs to satisfy some integrability requirements. That is to guarantee that the gradient $ \delta \mapsto \nabla\mathcal{F}(h)(\delta)$ and its empirical version are well-defined and continuous. In addition, the optimality condition  $\nabla\mathcal{F}(h)=0$ is assumed to characterize the global solution  $h_0$. This will be the case if the kernel is characteristic \cite{Simon-Gabriel:2016}.
The proof of \cref{thm:consistency}, in \cref{sec:kale_proofs}, takes advantage of the Hilbert structure of the set $\mathcal{H}$, the convexity of the functional $\mathcal{F}$ and the optimality condition $\nabla \widehat{\mathcal{F}}(\hat{h}) =\lambda \hat{h}$ of the regularized problem, all of which turn out to be sufficient for controlling the error of \cref{eq:KALE_estimator}.

\subsection{Proofs}\label{sec:kale_proofs}
We state now the proof of \cref{thm:consistency} with subsequent lemmas and propositions.
\begin{proof}[Proof of \cref{thm:consistency}]
	We begin with the following inequalities:
	\begin{align}
		  \frac{\lambda}{2}(\Vert \hat{h} \Vert^2- \Vert h_0 \Vert^2) \leq \widehat{\mathcal{F}}(\hat{h})- \widehat{\mathcal{F}}(h_0)\leq \langle \nabla \widehat{\mathcal{F}}(h_0), \hat{h}-h_0\rangle.
	\end{align}
	The first inequality is by definition of $\hat{h}$ while the second is obtained by concavity of $\widehat{\mathcal{F}}$.	 For simplicity we write $\mathcal{B} = \Vert \hat{h}-h_0 \Vert $ and  $\mathcal{C} = \Vert \nabla \widehat{\mathcal{F}}(h_0) -  \mathcal{L}(h_0) \Vert$. Using Cauchy-Schwarz and triangular inequalities, it is easy to see that
	\begin{align}
		-\frac{\lambda}{2}\left(\mathcal{B}^2 +2\mathcal{B}\Vert h_0\Vert \right)\leq \widehat{\mathcal{F}}(\hat{h})- \widehat{\mathcal{F}}(h_0) \leq \mathcal{C} \mathcal{B}.
	\end{align}
	Moreover, by triangular inequality, it holds that
	\[\mathcal{B} \leq \Vert h_{\lambda} - h_0 \Vert + \Vert \hat{h}-h_{\lambda} \Vert. \]
        \cref{prop:consistency_reg_population} ensures that $\mathcal{A}(\lambda) = \Vert h_{\lambda} - h_0 \Vert$ converges to $0$ as $\lambda\rightarrow 0$. Furthermore, by \cref{prop:estimator_consistency}, we have $ \Vert  \hat{h} -h_{\lambda} \Vert \leq \frac{1}{\lambda} \mathcal{D} $ where $\mathcal{D}(\lambda) =  \Vert \nabla\widehat{\mathcal{F}}(h_{\lambda})- \nabla\mathcal{L}(h_{\lambda}) \Vert$.
	Now choosing $\lambda = \frac{1}{\sqrt{N}}$ and applying Chebychev inequality in \cref{lem:chebychev}, it follows that for any $1>\eta>0,$ we have with probability greater than  $1-2\eta$ that both
	\begin{align}
		\mathcal{D}(\lambda) \leq \frac{C(\Vert h_0\Vert \eta)}{\sqrt{N}},\qquad \mathcal{C}\leq \frac{C(\Vert h_0\Vert,\eta)}{\sqrt{N}},
	\end{align}
	where $C(\Vert h_0\Vert,\eta)$ is defined in  \cref{lem:chebychev}. This allows to conclude that for any $\eta>0$, it holds with probability at least $1-2\eta$ that $\vert \widehat{\mathcal{F}}(\hat{h}) - \widehat{\mathcal{F}}(h_0) \vert  \leq  \frac{M'(\eta, h_0)}{\sqrt{N}}$  where $M'(\eta, h_0)$  depends only on $\eta$ and $h_0$.

\end{proof}

We proceed using the following lemma, which provides an expression for $D(h,\delta)$ and $\hat{D}(h,\delta)$ along with a probabilistic bound:
\begin{lemma}\label{lem:chebychev}
	Under \cref{assump:hilbert,assump:continuity_differentiable}, for any $h\in \mathcal{H}$ such that $\Vert h\Vert\leq \Vert h_0\Vert+\epsilon$, there exists $\mathcal{D}(h) $ in $\mathcal{H}$ satisfying
	\begin{align}
		D(h,\delta) = \langle \delta,\mathcal{D}(h)\rangle,	
	\end{align}
	and for any $h\in \mathcal{H}$, there exists $\widehat{\mathcal{D}}(h)$ satisfying
	\begin{align}
	\widehat{D}(h,\delta) = \langle \delta, \widehat{\mathcal{D}}(h)\rangle.
	\end{align}
		Moreover, for any $0<\eta<1$
		 and any $h\in \mathcal{H}$ such that $\Vert h\Vert \leq \Vert h_0\Vert +\epsilon :=M$, it holds with probability greater than $1-\eta$ that
		 \begin{align}
		 	\Vert \mathcal{D}(h)-  \widehat{\mathcal{D}}(h)\Vert \leq \frac{C(M,\eta)}{\sqrt{N}},
		 \end{align}
		 where $C(M,\eta)$ depends only on $M$ and $\eta$.
\end{lemma}
\begin{proof}
		First, we show that $\delta \mapsto D(h,\delta) $ is a bounded linear operator. Indeed, \cref{assump:continuity_differentiable} ensures that $ k(x,.)$  and $k(x,.)\exp(-h(x))  $ are Bochner integrable w.r.t. $\pdist$  and $\qdist$ (\cite{Retherford:1978}), hence $D(h,\delta)$ is obtained as
	\begin{align}
		D(h,\delta) := \langle \delta, \mu_{\exp(-h)\qdist} - \mu_{\pdist}\rangle,
	\end{align}
	where $\mu_{\exp(-h)\qdist} = \int k(x,.)\exp(-h(x))d\qdist$ and $\mu_{\pdist} = \int k(x,.)d\pdist$. Defining  $\mathcal{D}(h)$ to be $= \mu_{\exp(-h)\qdist} - \mu_{\pdist}$ leads to the desired result. $\widehat{\mathcal{D}}(h)$ is simply obtained by taking the empirical version of $\mathcal{D}(h)$.
	
	Finally, the probabilistic inequality is a simple consequence of Chebychev's inequality.
\end{proof}
The next lemma states that $\mathcal{F}(h)$ and $\widehat{\mathcal{F}}(h)$ are Frechet differentiable.
\begin{lemma}\label{lem:frechet_diff}
	Under \cref{assump:hilbert,assump:continuity_differentiable} 
	,  $h\mapsto \mathcal{F}(h)$ is Frechet differentiable on the open ball of radius $\Vert h_0\Vert +\epsilon $ while $h\mapsto \widehat{\mathcal{F}}(h)$ is Frechet differentiable on $\mathcal{H}$. Their gradients are given by $\mathcal{D}(h) $ and $ \widehat{\mathcal{D}}(h)$ as defined in  \cref{lem:chebychev},
	\begin{align}
		 \nabla \mathcal{F}(h) = \mathcal{D}(h),\qquad  \nabla \widehat{\mathcal{F}}(h)=  \widehat{\mathcal{D}}(h)	
		 \end{align}
\end{lemma}
\begin{proof}
	The empirical functional $\widehat{\mathcal{F}}(h)$ is differentiable since it is a finite sum of differentiable functions, and its gradient is simply given by $\widehat{\mathcal{D}}(h)$. For the population functional, we use second order Taylor expansion of $\exp$ with integral remainder, which gives
	\begin{align}
		\mathcal{F}(h+\delta)=\mathcal{F}(h)-D(h,\delta) + \Gamma(h,\delta).
	\end{align}
	By \cref{assump:continuity_differentiable} we know that $\frac{\Gamma(h,\delta)}{\Vert \delta\Vert}$  converges to $0$ as soon as $\Vert \delta \Vert \rightarrow 0$. This allows to directly conclude that $\mathcal{F}$ is Frechet differentiable, with differential given by $\delta\mapsto D(h,\delta)$. By \cref{lem:chebychev}, we conclude the existence of a gradient $\nabla \mathcal{F}(h)$ which is in fact given by  $\nabla \mathcal{F}(h) = \mathcal{D}(h)$.

\end{proof}
From now on, we will only use the notation $\nabla \mathcal{F}(h)$ and $\nabla \widehat{\mathcal{F}}(h)$ to refer to the gradients of $\mathcal{F}(h)$ and $\widehat{\mathcal{F}}(h)$.
The following lemma states that \cref{eq:reg_population,eq:reg_population_2} have a unique global optimum, and gives a first order optimality condition. 
\begin{lemma}
The problems \cref{eq:reg_population,eq:reg_population_2} admit  unique  global solutions $\hat{h}$ and $h_{\lambda}$ in $\mathcal{H}$. Moreover, the following first order optimality conditions hold:
\begin{align}
	\lambda \hat{h} = \nabla \widehat{\mathcal{F}}(\hat{h}),\qquad \lambda h_{\lambda} = \nabla \mathcal{F}(h_{\lambda}).
\end{align}
\end{lemma}
\begin{proof}
	For \cref{eq:reg_population}, existence and uniqueness of a minimizer $\hat{h}$ is a simple consequence of continuity and strong concavity of the regularized objective. We now show the existence result for \cref{eq:reg_population_2}.
	Let's introduce $\mathcal{G}_{\lambda}(h) = -\mathcal{F}(h)+\frac{\lambda}{2}\Vert h\Vert^2$ for simplicity. Uniqueness is a  consequence of the strong convexity of $\mathcal{G}_{\lambda}$. For the existence,  consider a sequence of elements $f_k\in \mathcal{H}$ such that $\mathcal{G}_{\lambda}(f_k)\rightarrow \inf_{h\in \mathcal{H}} \mathcal{G}_{\lambda}(h)$.   If $h_0$ is not the global solution, then it must hold for $k$ large enough that $\mathcal{G}_{\lambda}(f_k)\leq \mathcal{G}_{\lambda}(h_0)$.
	We also know that $\mathcal{F}(f_k)\leq \mathcal{F}(h_0)$, hence, it is easy to see that $ \Vert f_k\Vert \leq \Vert  h_0\Vert $ for $k$ large enough. This implies that $f_k$ is a bounded sequence, therefore it admits a weakly convergent sub-sequence by weak compactness. Without loss of generality we assume that $f_k$ weakly converges to some element $h_{\lambda}\in \mathcal{H}$ and that $\Vert f_k\Vert\leq  \Vert h_0 \Vert $. Hence, $\Vert h_{\lambda} \Vert \leq \lim\inf_{k} \Vert f_k\Vert\leq \Vert h_0\Vert $. Recall now that by definition of weak convergence, we have $f_k(x)\rightarrow_k h_{\lambda}(x)$ for all $x\in \mathcal{X}$.  By \cref{assump:continuity_differentiable}, we can apply the dominated convergence theorem to ensure that $\mathcal{F}(f_k)\rightarrow \mathcal{F}(h_{\lambda})$. Taking the limit of $\mathcal{G}_{\lambda}{f_k}$, the following inequality holds:
	\begin{align}
		\sup_{h\in \mathcal{H}} \mathcal{G}_{\lambda}(h) = \lim\sup_{k}\mathcal{G}_{\lambda}(f_k)
\leq \mathcal{G}_{\lambda}(h_{\lambda}).
	\end{align}
	Finally, by \cref{lem:frechet_diff} we know that $\mathcal{F}$ is Frechet differentiable, hence we can use \cite{Ekeland:1999} (Proposition 2.1) to conclude that $\nabla \mathcal{F}(h_{\lambda}) = \lambda h_{\lambda}$.  We use exactly the same arguments for \cref{eq:reg_population}. 
\end{proof}
Next, we show that $h_{\lambda}$ converges towards $h_0$ in $\mathcal{H}$.
\begin{lemma}\label{prop:consistency_reg_population}
	Under \cref{assump:hilbert,assum:uniqueness,assump:continuity_differentiable} it holds that:
		\begin{align}
			\mathcal{A}(\lambda):= \Vert h_{\lambda}-h_0 \Vert \rightarrow 0.
		\end{align}
\end{lemma}
\begin{proof}
We will first prove that $h_{\lambda}$ converges weakly towards $h_0,$ and then conclude that it must also converge strongly. We  start with the following inequalities:
	\begin{align}
		0\geq \mathcal{F}(h_{\lambda})-\mathcal{F}(h_0) \geq \frac{\lambda}{2}(\Vert h_{\lambda}\Vert^2 - \Vert h_0\Vert^2).
	\end{align}
	These are simple consequences of the definitions of $h_{\lambda}$ and $h_0$ as optimal solutions to \cref{eq:reg_population,eq:general_problem}. This implies that $\Vert h_{\lambda} \Vert$ is always bounded by $\Vert h_0 \Vert$. Consider now an arbitrary sequence $(\lambda_m)_{m\geq 0}$ converging to $0$. Since $\Vert h_{\lambda_m} \Vert$ is bounded by $\Vert h _0 \Vert$, it follows by weak-compactness of balls in $\mathcal{H}$ that $h_{\lambda_m}$ admits a weakly convergent sub-sequence. Without loss of generality we can assume that  $h_{\lambda_m}$ is itself weakly converging towards an element $h^{*}$. We will show now that $h^{*}$ must be equal to $h_0$. Indeed, by optimality of $h_{\lambda_{m}}$, it must hold that
	\begin{align}
		\lambda_{m} h_{\lambda_{m}} = \nabla \mathcal{F}(h_{m}).
	\end{align}
	This implies that $\nabla \mathcal{F}(h_{m})$ converges weakly to $0$. On the other hand, by \cref{assump:continuity_differentiable}, we can conclude that  $\nabla \mathcal{F}(h_{m})$ must also converge weakly towards $\nabla \mathcal{F}(h^{*})$, hence $\nabla \mathcal{F}(h^{*})= 0$. Finally by \cref{assum:uniqueness} we know that $h_0$ is the unique solution to the equation $\nabla \mathcal{F}(h)= 0$ , hence $h^{*} = h_0$.
	We have shown so far that any subsequence of $h_{\lambda_m}$ that converges weakly, must converge weakly towards $h_0$. This allows to conclude that $h_{\lambda_m}$ actually converges weakly towards $h_0$. Moreover, we also have by definition of weak convergence that:
	\begin{align}
		\Vert h_0 \Vert \leq \lim\inf_{m\rightarrow\infty} \Vert  h_{\lambda_m}\Vert.
	\end{align}
	Recalling now that $\Vert h_{\lambda_m} \Vert\leq \Vert h_0 \Vert$ it follows that $\Vert h_{\lambda_m} \Vert $ converges towards $\Vert h_0 \Vert$. Hence, we have the following two properties:
	\begin{itemize}
		\item $h_{\lambda_m}$ converges weakly towards $h_0$,
		\item $\Vert h_{\lambda_m} \Vert$ converges towards $\Vert h_0 \Vert$.
	\end{itemize}
	This allows to directly conclude that $\Vert h_{\lambda_m} - h_0  \Vert$ converges to $0$.
\end{proof}

\begin{proposition}\label{prop:estimator_consistency}
	We have that:
	\begin{align}
		\Vert \hat{h}-h_{\lambda}\Vert \leq \frac{1}{\lambda} \Vert \nabla \hat{\mathcal{F}}(h_{\lambda}) - \nabla \mathcal{F}(h_{\lambda}) \Vert
	\end{align}	
\end{proposition}
\begin{proof}
	By definition of $\hat{h}$ and $h_{\lambda}$ the following optimality conditions hold:
	\begin{align}
		\lambda \hat{h} = \nabla \widehat{\mathcal{F}}(\hat{h}),\qquad \lambda h_{\lambda} = \nabla \mathcal{F}( h_{\lambda} ).
	\end{align}
	We can then simply write:
	\begin{align}
		\lambda(\hat{h} - h_{\lambda}) - (\nabla \widehat{\mathcal{F}}(\hat{h}) - \nabla \widehat{\mathcal{F}}(h_{\lambda}))  = \nabla \widehat{\mathcal{F}}(h_{\lambda}) - \nabla \mathcal{F}( h_{\lambda} ).
	\end{align}
	Now introducing $\delta := \hat{h}- h_{\lambda} $ and $E :=  \nabla \widehat{\mathcal{F}}(\hat{h}) - \nabla \widehat{\mathcal{F}}(h_{\lambda})$ for simplicity and taking the squared norm of the above equation, it follows that
	\begin{align}
		\lambda^2 \Vert \delta \Vert^2 + \Vert E \Vert^2 - 2\lambda \langle \delta, E \rangle = \Vert  \nabla \widehat{\mathcal{F}}(h_{\lambda}) - \nabla \mathcal{F}( h_{\lambda} ) \Vert^2 .
	\end{align}
	By concavity of $\widehat{\mathcal{F}}$ on $\mathcal{H}$ we know that $-\langle \hat{h}-h_{\lambda}, E \rangle \geq 0 $. Therefore:
	\begin{align}
		\lambda^2 \Vert \hat{h}- h_{\lambda} \Vert^2 \leq  \Vert  \nabla \widehat{\mathcal{F}}(h_{\lambda}) - \nabla \mathcal{F}( h_{\lambda} ) \Vert^2.
	\end{align}
\end{proof}

\section{Latent noise sampling and Smoothness of KALE}
\subsection{Latent space sampling}\label{sec:latent_space_sampling}
Here we prove \cref{prop:convergence_langevin} for which we make the assumptions more precise:
\begin{assumption}\label{assumption:latent_MCMC}
We make the following assumption:
 
	\begin{itemize}
		\item $\log \eta$ is strongly concave and admits a Lipschitz gradient.
		\item There exists a non-negative constant $L$ such that for any $x,x'\in \mathcal{X}$ and $z,z'\in \mathcal{Z}$:
				\begin{align}
					\vert E(x)-E(x') \vert &\leq \Vert x-x'\Vert, \qquad \Vert \nabla_x E(x)- \nabla_x E(x') \Vert \leq \Vert x-x'\Vert & \\
					\vert \B(z)-\B(z') \vert &\leq \Vert z-z'\Vert, \qquad \Vert \nabla_z \B(z)- \nabla_z \B(z') \Vert \leq \Vert z-z'\Vert &			
				\end{align}
	\end{itemize}
\end{assumption}
Throughout this section, we introduce  $U(z) :=   -\log(\eta(z)) + E(\B(z))$ for simplicity.

\begin{proof}[Proof of \cref{prop:sampling} ]\label{proof:prop:sampling}
	To sample from $\QQ_{\BB,E}$, we first need to identify the \textit{posterior latent} distribution $\nu_{\BB, E}$ used to produce those samples. We rely on \cref{eq:push_forward} which holds by definition of $\QQ_{\BB,E}$ for any test function $h$ on $\mathcal{X}$:  
\begin{align}\label{eq:push_forward}
	\int h(x) \diff \QQ(x) = \int h(\B(z)) f(\B(z)) \eta(z) \diff z,   
\end{align}
Hence, the posterior latent distribution is given by $\nu(z) =  \eta(z) f(\B(z)),$ and samples from GEBM are produced by first sampling from $\nu_{\BB, E}$, then applying the implicit map $\B$,
\begin{align}
	X\sim \QQ \quad \iff \quad X = \B(Z), \quad Z \sim \nu.
\end{align}

\end{proof}

\begin{proof}[Proof of \cref{prop:ebm_as_gebm}]\label{proof:prop:ebm_as_gebm}
the base distribution $\BB$ admits a density on the whole space denoted by $\exp(-r(x))$ and the energy $\tilde E$ is of the form $ \tilde E(x) = E(x)-r(x) $ for some parametric function $E$, it is easy to see that $\QQ$ has a density proportional to $\exp(-E)$ and is therefore  equivalent to a standard EBM with energy $E$. 

The converse holds as well, meaning that for any EBM with energy $E$, it is possible to construct a GEBM using an \textit{importance weighting} strategy. This is achieved by first choosing a base $\BB$, which is required to have an explicit density $ \exp(-r) $ up to a normalizing constant, then defining the energy of the GEBM to be $\tilde E(x) = E(x) - r(x)$ so that:
\begin{align}\label{eq:EBM_as_GEBM}
	\diff \QQ(x) \propto \exp(-\tilde E(x))\diff \BB_{\theta}(x)\propto \exp(-E(x))\diff x
\end{align}
Equation \cref{eq:EBM_as_GEBM} effectively depends only on $E(x)$ and not on $\BB$ since the factor $\exp(r)$ exactly compensates for the density of $\BB$.
The requirement that the base also admits a tractable implicit map $\B$ can be met by choosing $\BB$ to be a \textit{normalizing flow} \citep{Rezende:2015} and does not restrict the class of possible EBMs that can be expressed as GEBMs.
\end{proof}

\begin{proof}[Proof of \cref{prop:convergence_langevin}]\label{poof:prop:convergence_langevin}
Let $\pi_t$ be the probability distribution of $(z_t,v_t)$ at time $t$ of the diffusion in  \cref{eq:latent_langevin_diffusion}, which we recall that
\begin{align}
	\begin{aligned}
	dz_t &= v_t dt,\qquad dv_t &= -\left(\gamma v_t  + u\nabla U(z_t)\right) + \sqrt{2\lambda u}dw_t,
	\end{aligned}
\end{align}
We call $\pi_{\infty}$ its corresponding invariant distribution given by
\begin{align}
	\pi_{\infty}(z,v) \propto  \exp{\left(-U(z)-\frac{1}{2}\Vert v\Vert^2  \right)}
\end{align}
 By \cref{lemm:dissipative_energy} we know that $U$ is dissipative, bounded from below, and has a Lipschitz gradient. This allows to directly apply \citep{Eberle:2018}(Corollary 2.6.) which implies that
\begin{align}\label{eq:rate_latent}
	W_2(\pi_t,\pi_{\infty})\leq C \exp(-tc),
\end{align}
where $c$ is a positive constant and $C$ only depends on $\pi_{\infty}$ and the initial distribution $\pi_0$. Moreover, the constant $c$ is given explicitly in  \cite[Theorem 2.3]{Eberle:2018} and is of order $0( e^{-q} )$ where $q$ is the dimension of the latent space $\mathcal{Z}$.
  
We now consider an optimal coupling $\Pi_t$ between $\pi_t$ and $\pi_0$. Given joints samples $((z_t,v_t),(z,v))$ from $\Pi_t$, we consider the following samples in input space $ (x_t,x) :=  (\B(z_t),\B(z))$. Since $z_t$ and $z$ have marginals $\pi_t$ and $\pi_{\infty}$, it is easy to see that $x_t \sim \mathbb{P}_t $ and $x\sim \QQ$. Therefore, by definition of the $W_2$ distance, we have the following bound:

\begin{align}
	W^2_2(\mathbb{P}_t,\QQ)&\leq \mathbb{E}\left[  \Vert x_t-x \Vert^2 \right]\\
	 &\leq \int \Vert \B(z_t) - \B(z)  \Vert^2 d\Pi_t(z_t,z)\\
	 &\leq L^2 \int \Vert z_t- z  \Vert^2 d\Pi_t(z_t,z)\\ 
	 &\leq L^2W_2^2(\pi_t,\pi_{\infty})\leq C^2L^2\exp(-2tc).
\end{align}
The second line uses the definition of $(x_t,x)$ as joint samples obtained by mapping $(z_t,z)$. The third line uses the assumption that $B$ is $L$-Lipschitz. Finally, the last line uses that $\Pi_t$ is an optimal coupling between $\pi_t$ and $\pi_{\infty}$.
\end{proof}

\begin{lemma}\label{lemm:dissipative_energy}
	Under \cref{assumption:latent_MCMC}, there exists $A>0$ and $\lambda \in  (0,\frac{1}{4}]$ such that
\begin{align}\label{eq:dissipative_energy}
		\frac{1}{2}z^{\top{t}}\nabla U(z) \geq  \lambda \left( U(z) + \frac{\gamma^2}{4u} \Vert z \Vert^2  \right) - A , \qquad \forall z \in \mathcal{Z},
\end{align}
where $\gamma$ and $u$ are the coefficients appearing in  \cref{eq:latent_langevin_diffusion}. Moreover, $U$ is bounded bellow and has a Lipschitz gradient.
\end{lemma}
\begin{proof}
For simplicity, let's call $u(z)=  -\log \eta(z)$,   $w(z) = E^{\star}\circ B_{\theta^{\star}}(z),$  and denote by $M$ an upper-bound on the Lipschitz constant of $w$ and $\nabla w$ which is guaranteed to be finite by assumption. Hence $U(z) = u(z)  + w(z)$. Equation \cref{eq:dissipative_energy} is equivalent to having
	\begin{align}\label{eq:intermediate}
		z^{\top}\nabla u(z)-2\lambda u(z) - \frac{\gamma^2}{2u} \Vert z \Vert^2   \geq 2\lambda w(z) - z^{\top}\nabla w(z) - 2A.
	\end{align}
Using that $w$ is Lipschitz, we have that $ w(z)\leq w(0) +M\Vert z \Vert $ and $-z^{\top}\nabla w(z) \leq M\Vert z\Vert $. Hence,   $2\lambda w(z) - z^{\top}\nabla w(z) - 2A  \leq 2\lambda w(0) +(2\lambda+1)M\Vert z\Vert -2A$.  Therefore, a sufficient condition for \cref{eq:intermediate} to hold is
\begin{align}\label{eq:intermediate_2}
	z^{\top}\nabla u(z)-2\lambda u(z) - \frac{\gamma^2}{2u} \Vert z \Vert^2 \geq +(2\lambda+1)M\Vert z\Vert -2A +2\lambda w(0).
\end{align}
We will now rely on the strong convexity of $u,$ which holds by assumption, and implies the existence of a positive constant $m>0$ such that
\begin{align}
	-u(z)&\geq -u(0) - z^{\top}\nabla u(z) + \frac{m}{2}\Vert z\Vert^2,\\
	z^{\top}\nabla u(z)
	&\geq - \Vert z\Vert\Vert \nabla u(0) \Vert + m\Vert z \Vert^2.
\end{align}
This allows to write the following inequality,
\begin{align}
	z^{\top}\nabla u(z)-2\lambda u(z) - \frac{\gamma^2}{2u} &\geq (1-2\lambda)z^{\top}\nabla u(z) + \lambda(m+ \frac{\gamma^2}{2u} )\Vert z \Vert^2 - 2\lambda u(0) \\
	&\geq 
	(1-\lambda (m + \frac{\gamma^2}{2u} ) ) \Vert z \Vert^2 - (1-2\lambda)\Vert z\Vert \Vert \nabla u(0) \Vert - 2\lambda u(0).
\end{align}
Combining the previous inequality with \cref{eq:intermediate_2} and  denoting $M' = \Vert \nabla u(0) \Vert$ , it is sufficient to find $A$ and $\lambda$ satisfying
\begin{align}
	\left(1-\lambda \left(m + \frac{\gamma^2}{2u} \right) \right) \Vert z \Vert^2 - \left(M+M'+2\lambda(M-M')\right)\Vert z\Vert  - 2\lambda (u(0)+w(0)) +2A \geq 0.
	\end{align}
The l.h.s. in the above equation is a quadratic function in $\Vert z\Vert$ and admits a global minimum when $\lambda <  \left(m+\frac{\gamma^2}{2u}\right)^{-1} $. The global minimum is always positive provided that $A$ is large enough.

To see that $U$ is bounded below, it suffice to note, by Lipschitzness of $w$, that $w(z)\geq w(0)- M\Vert z \Vert$ and by strong convexity of $u$ that
\begin{align}
	u(z)\geq u(0)+ M'\Vert z \Vert + \frac{m}{2}\Vert z\Vert^2.
\end{align}
Hence, $U$ is lower-bounded by a quadratic function in $\Vert z\Vert$ with positive leading coefficient $\frac{m}{2}$, hence it must be lower-bounded by a constant. Finally, by assumption, $u$  and $w$ have  Lipschitz gradients, which directly implies that $U$ has a Lipschitz gradient.
\end{proof}

\begin{proof}[Proof of \cref{prop:kl_improvement}]\label{proof:prop:kl_improvement}
	By assumption $KL(\mathbb{P}||\BB)<+\infty$, this implies that $\mathbb{P}$ admits a density w.r.t. $\BB$ which we call $r(x)$. As a result  $\mathbb{P}$ admits also a density w.r.t. $\QQ$ given by:
	\begin{align}
		Z\exp(E^{\star}(x))r(x).
	\end{align}
	We can then compute the $KL(\mathbb{P}||\QQ)$ explicitly:
	\begin{align}
		KL(\mathbb{P}||\QQ)&= \mathbb{E}_{\mathbb{P}}[ E  ] + \log(Z) + \mathbb{E}_{\mathbb{P}}[\log(r)]\\ 
		&= -\mathcal{L}_{\mathbb{P},\BB}(E^{\star}) + KL(\mathbb{P}||\BB).
	\end{align}
	Since $0$ belongs to $\mathcal{E}$ and by optimality of $E^{\star}$, we know that  $\mathcal{L}_{\mathbb{P},\BB}(E^{\star}) \geq \mathcal{L}_{\mathbb{P},\BB}(0)=0$. The result then follows directly.
\end{proof}

\subsection{Topological and smoothness properties of KALE}\label{sec:proof_smoothness_kale}
{\bf Topological properties of KALE.}
Denseness and smoothness of  the energy class $\mathcal{E}$ are the key to guarantee that  KALE is a reliable criterion for measuring convergence. We thus make the following assumptions on $\mathcal{E}$:
\begin{assumplist}
\item  \label{densness}  For all $ E\in\mathcal{E}$, $-E\in \mathcal{E}$ and there is $C_{E}>0$ such that $ cE\in \mathcal{E}$ for $ 0\leq c\leq C_{E}$. For any continuous function $g$, any compact support $K$ in $\mathcal{X}$ and any precision $\epsilon>0$,  there exists a finite linear combination of energies $ G= \sum_{i=1}^r a_i E_i $ such that $ \sup_{x\in K}\vert f(x)-G(x)\vert \leq \epsilon.$
\item  \label{smoothness} All energies $E$ in $\mathcal{E}$ are Lipschitz in their input with the same Lipschitz constant $L>0$.
\end{assumplist}
\cref{densness} holds in particular when $\mathcal{E}$ contains feedforward networks with a given number of parameters. In fact networks with a single neuron are enough, as shown in \cite[Theorem 2.3]{Zhang:2017}. \cref{smoothness} holds when additional  regularization of the energy is enforced during training by methods such as  \textbf{spectral normalization} \cite{Miyato:2018} or \textbf{gradient penalty} \cite{Gulrajani:2017} as done in \cref{sec:experiments}.
\cref{prop:kale_extension} states the topological properties of KALE ensuring that it can be used as a criterion for weak convergence. A proof is given in \cref{proof:prop:kale_smooth} and is a consequence of \cite[Theorem B.1]{Zhang:2017}.  %
\begin{proposition}\label{prop:kale_extension_2} 
	Under \cref{densness,smoothness} it holds that: 
	\begin{enumerate}
		\item $\text{KALE}(\PP||\BB)\geq 0$ with $\text{KALE}(\PP||\BB) = 0 $ if and only if $ \PP = \BB $.
		\item $\text{KALE}(\PP||\BB^{n}) \rightarrow 0$ if and only if $ \BB^{n}\rightarrow \PP$ under the weak topology.
	\end{enumerate}	
\end{proposition} 
\subsubsection{Topological properties of KALE}
In this section we prove \cref{prop:kale_extension}. We first start by recalling the required assumptions and make them more precise:
\begin{assumption} \label{assumption:topological_properties}
Assume the following holds:
	\begin{itemize}
	\item The set $\mathcal{X}$ is compact.
	\item  For all $ E\in\mathcal{E}$, $-E\in \mathcal{E}$ and there is $C_{E}>0$ such that $ cE\in \mathcal{E}$ for $ 0\leq c\leq C_{E}$. For any continuous function $g$, any compact support $K$ in $\mathcal{X}$ and any precision $\epsilon>0$,  there exists a finite linear combination of energies $ G= \sum_{i=1}^r a_i E_i $ such that $ \vert f(x)-G(x)\vert \leq \epsilon$ on $K$.  
	\item   All energies $E$ in $\mathcal{E}$ are Lipschitz in their input with the same Lipschitz constant $L>0$.
	\end{itemize} 
\end{assumption}
For simplicity we consider the set $\mathcal{H}= \mathcal{E}+\mathbb{R}$, i.e.: $\mathcal{H}$ is the set of functions $h$ of the form $h = E+c$ where $E\in \mathcal{E}$ and $c\in \mathbb{R}$. 
In all what follows $\mathcal{P}_1$ is the set of probability distributions with finite first order moments. We consider the notion of weak convergence on $\mathcal{P}_1$ as defined in \cite[Definition 6.8]{Villani:2009} which is equivalent to convergence in the Wasserstein-1 distance $W_1$.

\begin{proof}[Proof of \cref{prop:kale_extension} ]\label{proof:prop:kale_smooth}
We proceed by proving the \textbf{separation} properties ($1^{st}$ statement), then the \textbf{metrization of the weak topology} ($2^{nd}$ statement).

{\bf Separation.}  We have by  \cref{assumption:topological_properties} that $ 0 \in \mathcal{E}  $, hence by definition $\text{KALE}(PP||\BB) \geq  \mathcal{F}_{\PP,\BB}(0) = 0 $. On the other hand, whenever $\PP=\BB$, it holds that:
\begin{align}
	\mathcal{F}_{\PP,\BB}(h) = -\int \left(\exp(-h)+h-1\right)d\PP,\qquad \forall h\in \mathcal{H}.
\end{align}
Moreover, by convexity of the exponential, we know that $\exp(-x)+x-1\geq 0 $ for all $x\in \mathbb{R}$. Hence, $\mathcal{F}_{\PP,\BB}(h) \leq \mathcal{F}_{\PP,\BB}(0)=0$ for all $h\in \mathcal{H}$. This directly implies that $\text{KALE}(\PP|\BB)=0$.
For the converse, we will use the same argument as in the proof of \cite[Theorem B.1]{Zhang:2017}. Assume that $\text{KALE}(\PP|\BB)=0$ and let  $h$ be in $ \mathcal{H}$. By \cref{assumption:topological_properties}, there exists $C_h >0$ such that $  ch\in \mathcal{H}$ and we have:
\begin{align}
	\mathcal{F}(ch) \leq \text{KALE}(\PP||\BB)=0.
\end{align}
Now dividing by $c$ and taking the limit to $0$, it is easy to see that $-\int h \diff \PP + \int h\diff \BB \leq 0.$ Again, by \cref{assumption:topological_properties}, we also know that $-h\in \mathcal{H}$, hence, $\int h \diff \PP - \int h\diff \BB \leq 0$. This necessarily implies that $\int h \diff \PP - \int h\diff \BB = 0$ for all $h\in \mathcal{H}$. By the density of $\mathcal{H}$ in the set continuous functions on compact sets, we can conclude that the equality holds for any continuous and bounded function, which in turn implies that $\PP=\BB$.

{\bf Metrization of the weak topology.}  
We first show that for any $\PP$ and $\BB$ with finite first moment, it holds that $\text{KALE}(\PP|\BB) \leq LW_1(\PP,\BB)$, where $W_1(\PP,\BB)$ is the Wasserstein-1 distance between $\PP$ and $\BB$. 
For any $h\in\mathcal{H}$ the following holds:
\begin{align}
	\mathcal{F}(h) =& -\int h d\pdist -\int \exp(-h)d\qdist +1\\
					=& \int h(x)d\qdist(x)-h(x') d\pdist(x') \\
					&-\int \underbrace{\left(\exp(-h)+h-1\right)}_{\geq 0}d\qdist \\
					\leq &   \int h(x)d\qdist(x)-h(x') d\pdist(x')\leq LW_1(\pdist,\qdist)
\end{align}
	The first inequality results from the convexity of the exponential while the last one is a consequence of $h$ being $L$-Lipschitz.
	This allows to conclude that $\text{KALE}(\pdist||\qdist)\leq LW_1(\pdist,\qdist)$ after taking the supremum over all $h\in \mathcal{H}$.
Moreover, since $W_1$ metrizes the weak convergence on $\mathcal{P}_1$ \cite[Theorem 6.9]{Villani:2009}, it holds that whenever a sequence $\BB^n$  converges weakly towards $\PP$ in $\mathcal{P}_1$ we also have $W_1(\PP,\BB^n)\rightarrow 0$ and thus $\text{KALE}(\PP||\BB^n)\rightarrow 0$. The converse is a direct consequence of \cite[Theorem 10]{Liu:2017b} since by assumption $\mathcal{X}$ is compact.
\paragraph{Well-defined learning.}
Assume that for any $\epsilon>0$ and any $h$ and $h'$ in $\mathcal{E}$ there exists $f$ in $2\mathcal{E}$ such that $\Vert h+h'- f \Vert_{\infty} \leq \epsilon$ then there exists a constant $C$ such that:
\begin{align}
	\text{KALE}(\PP,\QQ)\leq C \text{KALE}(\PP,\BB)
\end{align}
This means that the proposed learning procedure which first finds the optimal energy $E^{\star}$ given a base $\BB$ by maximum likelihood then minimizes $\text{KALE}(\PP,\BB)$ ensures ends up minimizing the distance between the data end the generalized energy-based model $\QQ$.
\begin{align}
	\text{KALE}(\PP,\QQ) &= \sup_{h\in\mathcal{E}} \mathcal{L}_{\PP,\QQ_{\BB}}(h)  \\
	&=  -\text{KALE}(\PP,\BB) + \sup_{h\in\mathcal{E}} \mathcal{L}_{\PP,\BB}(h+E^{\star})
\end{align}
Let's choose $\epsilon= KALE(\PP,\BB)$ and let $h\in 2\mathcal{E}$ such that $ \Vert h+E^{\star}- f \Vert_{\infty} \leq \epsilon $. We have by concavity of the function $(\alpha,\beta )\mapsto\mathcal{L}_{\PP,\BB}(\alpha (h+E^{\star}-f)+\beta f )$ we have that:
\begin{align}
\mathcal{L}_{\PP,\BB}(h+E^{\star}) \leq 2\mathcal{L}_{\PP,\BB}(\frac{1}{2}f) - \mathcal{L}_{\PP,\BB}(h+E^{\star}-f)
\end{align}
By assumption, we have that $ \Vert h+E^{\star}- f \Vert_{\infty} \leq \epsilon $, thus $\vert \mathcal{L}_{\PP,\BB}(h+E^{\star}-f)\vert \leq 2\epsilon$. Moreover, we have that  $\mathcal{L}_{\PP,\BB}(\frac{1}{2}f)\leq KALE(\PP,\BB)$ since $\frac{1}{2}f\in \mathcal{E}$. This ensures that:
\begin{align}
	\mathcal{L}_{\PP,\BB}(h+E^{\star})\leq 3 \text{KALE}(\PP,\BB).
\end{align}
Finally, we have shown that:
\begin{align}
	\text{KALE}(\PP,\QQ)\leq 2\text{KALE}(\PP,\BB).
\end{align}
 Hence, minimizing $\text{KALE}(\PP,\BB)$ directly minimizes $\text{KALE}(\PP,\QQ)$.
 
\end{proof}

\subsubsection{Smoothness properties of KALE}
We will now prove \cref{thm:kale_gan_convergence}. We begin by stating the assumptions that will be used in this section: 
\begin{assumplist2}
	\item\label{assump:compactness} $\mathcal{E}$ is parametrized by a compact set of parameters $\Psi$.
	\item\label{assump:l_smooth_critic} Functions in $\mathcal{E}$ are jointly continuous w.r.t. $(\psi,x)$  and are $L$-lipschitz and $L$-smooth w.r.t. the input $x$:
		\begin{align}
			\Vert E_{\psi}(x)-E_{\psi}(x') \Vert &\leq L_{e}\Vert x-x' \Vert,\\ 
			\Vert \nabla_x E_{\psi}(x)-\nabla_x E_{\psi}(x') \Vert&\leq L_{e}\Vert x-x' \Vert.
		\end{align}
	\item\label{assump:l_smooth_gen} 
	$(\theta,z)\mapsto \B_{\theta}(z)$ is jointly continuous in $\theta$ and $z$, with $z\mapsto \B_{\theta}(z)$ uniformly Lipschitz w.r.t. $z$:
		\begin{align}
			\Vert \B_{\theta}(z) - \B_{\theta}(z') \Vert \leq L_b \Vert z- z' \Vert,\qquad \forall z,z'\in \mathcal{Z}, \theta \in \Theta.  
		\end{align}	
	There exists non-negative functions $a$ and $b$ defined from $\mathcal{Z}$ to $\mathbb{R}$   such that $\theta \mapsto \B_{\theta}(z)$ are $a$-Lipschitz and $b$-smooth in the following sense:
			\begin{align}
				\Vert \B_{\theta}(z)-\B_{\theta'}(z) \Vert&\leq  a(z) \Vert \theta - \theta' \Vert,\\
				\Vert \nabla_{\theta} \B_{\theta}(z) - \nabla_{\theta} \B_{\theta'}(z)  \Vert &\leq b(z) \Vert  \theta -\theta' \Vert.  
			\end{align}
			Moreover,  $a$ and $b$ are integrable in the following sense:
			\begin{align}
				\int  a(z)^2 \exp(2 L_eL_b \Vert  z \Vert ) d\eta(z) <\infty,\qquad   \int  \exp( L_eL_b \Vert  z \Vert ) d\eta(z) <\infty,\\
			\end{align}
			\begin{align}
							\int b(z) \exp(L_eL_b \Vert  z \Vert ) d\eta(z) < \infty.	
			\end{align}
\end{assumplist2}
To simplify notation, we will denote by $\mathcal{L}_{\theta}(f)$ the expected $\BB_{\theta}$ log-likelihood under $\mathbb{P}$. In other words,
\begin{align}
	\mathcal{L}_{\theta}(E):= \mathcal{L}_{\PP,\BB_\theta}(E) = -\int E d \mathbb{P} - \log \int \exp(-E)d\BB_{\theta}.
\end{align}
We also denote by $p_{E,\theta}$ the density of the model w.r.t. $\qdist_{\theta}$,
\begin{align}
	p_{E,\theta}=  \frac{\exp(-E)}{Z_{\BB_{\theta},E}},\qquad Z_{\BB_{\theta},E}= \int  \exp(-E)d \qdist_{\theta}.
\end{align}
We write $ \mathcal{K}(\theta) := \text{KALE}(\pdist||\qdist_{\theta})$ to emphasize the dependence on $\theta$.
\begin{proof}[Proof of \cref{thm:kale_gan_convergence}]
To show that sub-gradient methods converge to local optima, we only need to show that $\mathcal{K}$ is Lipschitz continuous and weakly convex. This directly implies convergence to local optima for sub-gradient methods, according to \cite{Davis:2018,Thekumparampil:2019}. Lipschitz continuity ensures that $\mathcal{K}$ is differentiable for almost all $\theta\in \Theta,$ and weak convexity simply means that there exits some positive constant $C\geq 0$ such that $\theta \mapsto \mathcal{K}(\theta) + C\Vert \theta \Vert^2$ is convex. We now proceed to show these two properties. 

We will first prove that $ \theta \mapsto \mathcal{K}(\theta)$ is weakly convex in $\theta$. 
By \cref{lem:smooth_functional}, we know that for any $E\in \mathcal{E}$, the function $\theta\mapsto \mathcal{L}_{\theta}(E)$ is $M$-smooth for the same positive constant $M$. This directly implies that it is also weakly convex and the following inequality holds: 
\begin{align}
\mathcal{L}_{\theta_t}(E) \leq  t \mathcal{L}_{\theta}(E) + (1-t)\mathcal{L}_{\theta'}(E) + \frac{M}{2}t(1-t)\Vert\theta-\theta'\Vert^2.	
\end{align} 
Taking the supremum w.r.t. $E$, it follows that
\begin{align}
	\mathcal{K}(\theta_t)\leq t \mathcal{K}(\theta) + (1-t)\mathcal{K}(\theta') + \frac{M}{2}t(1-t)\Vert\theta-\theta'\Vert^2.
\end{align}
This  means precisely that $\mathcal{K}$ is weakly convex in $\theta$. 

To prove that $\mathcal{K}$ is Lipschitz, we will also use \cref{lem:smooth_functional}, which states that $\mathcal{L}_{\theta}(E)$ is Lipschitz  in $\theta$ uniformly on $\mathcal{E}$. Hence, the following holds:
\begin{align}
	\mathcal{L}_{\theta}(E) \leq \mathcal{L}_{\theta}(E)  +LC \Vert \theta - \theta' \Vert.
\end{align}
Again, taking the supremum over $E$, it follows directly that
\begin{align}
	\mathcal{K}(\theta)\leq \mathcal{K}(\theta') +LC \Vert \theta - \theta' \Vert.
\end{align}
We conclude that $\mathcal{K}$ is Lipschitz by exchanging the roles of $\theta$ and $\theta'$ to get the other side of the inequality. Hence, by the Rademacher theorem, $\mathcal{K}$ is differentiable for almost all $\theta$. 

We will now provide an expression for the gradient of $\mathcal{K}$.
By \cref{lem:dom} we know that $\psi \mapsto \mathcal{L}_{\theta}(E_{\psi})$ is continuous and by \cref{assump:compactness} $\Psi$ is compact. Therefore, the supremum $\sup_{E\in \mathcal{E}}\mathcal{L}_{\theta}(E)$ is achieved for some function $E^{\star}_{\theta}$. 
Moreover, we know by \cref{lem:smooth_functional} that $\mathcal{L}_{\theta}(E)$ is smooth uniformly on $\mathcal{E}$, therefore the family $(\partial_{\theta}\mathcal{L}_{\theta}(E))_{E\in \mathcal{E}}$ is equi-differentiable. We are in position to apply \cite{Milgrom:2002}(Theorem 3) which ensures that $\mathcal{K}(\theta)$ admits left and right partial derivatives given by
\begin{equation}\label{eq:left_right_limts}
  \begin{aligned}
  	\partial_{e}^{+}\mathcal{K}(\theta) = \lim_{\substack{t>0\\ t\rightarrow 0}} \partial_{\theta}\mathcal{L}_{\theta}(E^{\star}_{\theta + te})^{\top}e,\\
	\partial_{e}^{-}\mathcal{K}(\theta) = \lim_{\substack{t<0\\ t\rightarrow 0}} \partial_{\theta}\mathcal{L}_{\theta}(E^{\star}_{\theta + te})^{\top}e,
  \end{aligned}
\end{equation}
where $e$ is a given direction in $\mathbb{R}^r$. Moreover, the theorem also states that $\mathcal{K}(\theta)$ is differentiable iff $t\mapsto E^{\star}_{\theta+te}$ is continuous at $t=0$. 
Now, recalling that $\mathcal{K}(\theta)$ is actually differentiable for almost all $\theta$, it must hold that $E^{\star}_{\theta+te}\rightarrow_{t\rightarrow 0}  E^{\star}_{\theta}$ and $\partial_{e}^{+}\mathcal{K}(\theta)= \partial_{e}^{-}\mathcal{K}(\theta)$ for almost all $\theta$.  This implies that the two limits in \cref{eq:left_right_limts} are actually equal to $\partial_{\theta}\mathcal{L}_{\theta}(E^{\star}_{\theta })^{\top}e$. The gradient of $\mathcal{K}$, whenever defined, in therefore given by 
\begin{align}
	\nabla_{\theta} \mathcal{K}(\theta) = Z_{\BB_{\theta},E^{\star}_{\theta}}^{-1}\int \nabla_x E^{\star}_{\theta}(\B_{\theta}(z)) \nabla_{\theta} \B_{\theta}(z) \exp(-E^{\star}_{\theta}(\B_{\theta}(z))) \eta(z)\diff z.
\end{align}
\end{proof}
\begin{lemma}\label{lem:smooth_functional}
	Under \cref{assump:compactness,assump:l_smooth_critic,assump:l_smooth_gen}, the functional $\mathcal{L}_{\theta}(E)$ is Lipschitz and smooth in $\theta$ uniformly on $\mathcal{E}$:
	\begin{align}
		\vert \mathcal{L}_{\theta}(E) - \mathcal{L}_{\theta'}(E)  \vert 
		&\leq LC \Vert \theta- \theta'\Vert,\\
		 \Vert \partial_{\theta}\mathcal{L}_{\theta}(E) -  \partial_{\theta}\mathcal{L}_{\theta'}(E))  \Vert 
		&\leq 2CL(1+L)\Vert \theta- \theta'\Vert.
	\end{align}
\end{lemma}
\begin{proof}
	By \cref{lem:dom}, we have that $\mathcal{L}_{\theta}(E)$ is differentiable, and that
	\begin{align}\label{eq:grad_kale_appendix}
		\partial_{\theta}\mathcal{L}_{\theta}(E) := \int \left(\nabla_x E\circ \B_{\theta} \right)\nabla_{\theta}\B_{\theta}\left(p_{E,\theta}\circ \B_{\theta}\right) d\eta.
	\end{align}
\cref{lem:dom} ensures that $\Vert\partial_{\theta}\mathcal{L}_{\theta}(E)\Vert$ is bounded by some positive constant $C$ that is independent from $E$ and $\theta$.
This implies in particular that $\mathcal{L}_{\theta}(E)$ is Lipschitz with a constant $C$. We will now  show that it is also smooth. For this, we need to control the difference 
	\[D:= \Vert \partial_{\theta}\mathcal{L}_{\theta}(E) - \partial_{\theta}\mathcal{L}_{\theta'}(E) \Vert. \]
We have by triangular inequality:
	\begin{align}
		D \leq & \underbrace{\int \left \Vert \nabla_x E\circ \B_{\theta}-\nabla_x E\circ \B_{\theta'}  \right\Vert \Vert \nabla_{\theta} \B_{\theta} \Vert \left(p_{E,\theta}\circ \B_{\theta}\right)d\eta}_{I}\\ 
			&+\underbrace{\int \Vert \nabla_x E\circ \B_{\theta} \Vert\Vert \nabla_{\theta}\B_{\theta} - \nabla_{\theta}\B_{\theta'}\Vert \left(p_{E,\theta}\circ \B_{\theta}\right)d\eta}_{II}\\
			&+ \underbrace{\int \Vert \nabla_x E\circ \B_{\theta}\nabla_{\theta}\B_{\theta}  \Vert \vert p_{E,\theta}\circ \B_{\theta} - p_{E,\theta'}\circ \B_{\theta'}\vert d\eta}_{III}.
	\end{align}
The first term can be upper-bounded using $L_e$-smoothness of $E$ and the fact that $\B_{\theta}$ is Lipschitz in $\theta$:
\begin{align}
	I &\leq L_e\Vert\theta-\theta' \Vert \int \vert a \vert^2 (p_{E,\theta}\circ \B_{\theta}) d\eta\\
		&\leq 
		L_eC\Vert\theta-\theta' \Vert.
\end{align}
The last inequality was obtained by \cref{lem:integrability}.
Similarly, using that $\nabla_{\theta}\B_{\theta}$ is Lipschitz, it follows by \cref{lem:integrability} that
\begin{align}
	II & \leq L_e\Vert\theta-\theta' \Vert\int \vert b\vert (p_{E,\theta}\circ \B_{\theta}) d\eta\\
	& \leq L_eC\Vert\theta-\theta' \Vert.
\end{align}
Finally, for the last term $III$, we first consider a path $\theta_t   =  t\theta + (1-t)\theta' $ for $t\in[0,1],$ and introduce the function $s(t):= p_{E,\theta_t}\circ \B_{\theta_t}$. We will now control the difference $p_{E,\theta}\circ \B_{\theta} - p_{E,\theta'}\circ \B_{\theta'},$ also equal to $s(1)-s(0)$. Using the fact that $s_t$ is absolutely continuous we have that $s(1)-s(0) = \int_{0}^1 s'(t)dt$. The derivative $s'(t)$ is simply given by $s'(t) =(\theta-\theta')^{\top}(M_t-\bar{M}_t)s(t)$ where $M_t =   (\nabla_x E\circ B_{\theta_t})\nabla_{\theta}\B_{\theta_t} $ and $\bar{M}_t = \int M_t p_{E,\theta_t}\circ \B_{\theta_t}  d\eta$. Hence, 
\begin{align}
	 s(1)-s(0) =& (\theta-\theta')^{\top} \int_0^1 (M_t-\bar{M_t}) s(t)dt.
\end{align}
We also know that $M_t$ is upper-bounded by $L a(z),$ which implies
\begin{align}
	III &\leq L_e^2\Vert \theta-\theta'\Vert\int_0^1 \left(\int  \vert a(z) \vert^2 s(t)(z) d\eta(z)+ \left(\int a(z) s(t)(z) d\eta(z)\right)^2  \right) \\
		& \leq L_e^2(C+C^2)\Vert \theta- \theta'\Vert,
\end{align}
where the last inequality is obtained using \cref{lem:integrability}. 
This allows us to conclude that $\mathcal{L}_{\theta}(E)$ is smooth for any $E\in\mathcal{E}$ and $\theta\in \Theta$.
\end{proof}
\begin{lemma}\label{lem:dom}
	Under \cref{assump:l_smooth_critic,assump:l_smooth_gen}, it holds that $ \psi \mapsto \mathcal{L}_{\theta}(E_{\psi})$ is continuous,  and that $ \theta \mapsto \mathcal{L}_{\theta}(E_{\psi})$ is differentiable in $\theta$ with gradient given by
	\begin{align}
		\partial_{\theta}\mathcal{L}_{\theta}(E) := \int \left(\nabla_x E\circ \B_{\theta} \right)\nabla_{\theta}\B_{\theta}\left(p_{E,\theta}\circ \B_{\theta}\right) d\eta.
	\end{align}
	Moreover, the gradient is bounded uniformly in $\theta$ and $E$:
\begin{align}
	  		  	\Vert \nabla_{\theta}\mathcal{L}_{\theta}(E)\Vert \leq  L_e \left( \int  \exp(-L_eL_b \Vert z \Vert ) \diff \eta(z) \right)^{-1} \int a(z) \exp( L_eL_b \Vert z \Vert  ) \diff \eta(z).
	  \end{align}  
	
\end{lemma}
\begin{proof}
	To show that $ \psi \mapsto \mathcal{L}_{\theta}(E_{\psi})$ is continuous, we will use the dominated convergence theorem. We fix $\psi_{0}$ in the interior of $\Psi$ and consider a compact neighborhood $W$ of $\psi_0$. By assumption, we have that $(\psi,x) \mapsto E_{\psi}(x)$ and  $(\psi,z) \mapsto E_{\psi}(\B_{\theta}(z))$ are jointly continuous. Hence, $\vert E_{\psi}(0)\vert $ and $\vert E_{\psi}(\B_{\theta}(0))\vert $ are bounded on $W$ by some constant $C$.
	Moreover, by Lipschitz continuity of $x\mapsto E_{\psi}$, we have
	\begin{align}
		\vert E_{\psi}(x)\vert &\leq \vert E_{\psi}(0)\vert + L_e\Vert x\Vert \leq C +  L_e\Vert x\Vert,\\
		\exp(-E(\B_{\theta}(z))) &\leq \exp(-E(\B_{\theta}(0))) \exp(L_eL_b \Vert z \Vert)\leq \exp(C) \exp(L_eL_b \Vert z \Vert).
	\end{align}
	Recalling that $\PP$ admits a first order moment and that by \cref{assump:l_smooth_gen}, $\exp(L_eL_b \Vert z \Vert)$ is integrable w.r.t. $\eta$, it follows by the dominated convergence theorem and by composition of continuous functions that $ \psi \mapsto \mathcal{L}_{\theta}(E_{\psi})$ is continuous in $\psi_0$.
	
	To show that $ \theta \mapsto \mathcal{L}_{\theta}(E_{\psi})$ is differentiable in $\theta$, we will use the differentiation lemma in  \cite[Theorem 6.28]{Klenke:2008}. We first fix $\theta_0$ in the interior of $\Theta$, and  consider a compact neighborhood $V$ of $\theta_0$. Since $\theta\mapsto \vert E(\B_{\theta}(0))\vert  $ is continuous on the compact neighborhood $V$ it admits a maximum value $C$; hence we have using \cref{assump:l_smooth_critic,assump:l_smooth_gen} that
	\begin{align}
		\exp(-E(\B_{\theta}(z)))\leq \exp(-E(\B_{\theta}(0))) \exp(L_eL_b \Vert z \Vert)\leq \exp(C) \exp(L_eL_b \Vert z \Vert).
	\end{align} 
	Along with the integrability assumption in \cref{assump:l_smooth_gen}, this ensures that $z\mapsto \exp( -E(\B_{\theta}(z)))$ is integrable w.r.t $\eta$ for all $\theta$ in $V$. We also have that $\exp( -E(\B_{\theta}(z)))$ is differentiable, with gradient given by
	\begin{align}
		\nabla_{\theta}  \exp( -E(\B_{\theta}(z)))= \nabla_x E(\B_{\theta}(z)) \nabla_{\theta} \B_{\theta}(z)   \exp( -E(\B_{\theta}(z))).
	\end{align}
	 Using that $E$ is Lipschitz in its inputs and $\B_{\theta}(z)$ is Lipschitz in $\theta$,  and combining  with the previous inequality, it follows that
	 \begin{align}
	 	\Vert \nabla_{\theta}  \exp( -E(\B_{\theta}(z)))\Vert \leq \exp(C)L_e a(z)\exp(L_eL_b\Vert  z\Vert ),
	 \end{align}  
	  where $a(z)$ is the location dependent Lipschitz constant introduced in \cref{assump:l_smooth_gen}.  The r.h.s. of the above inequality is integrable by \cref{assump:l_smooth_gen} and is independent of $\theta$ on the neighborhood $V$. Thus \cite[Theorem 6.28]{Klenke:2008} applies, and it follows that
	  \begin{align}
	  	\nabla_{\theta} \int  \exp( -E(\B_{\theta_0}(z))) \diff \eta(z) = \int \nabla_x E(\B_{\theta_0}(z)) \nabla_{\theta} \B_{\theta_0}(z)   \exp( -E(\B_{\theta_0}(z))) \diff \eta(z).
	  \end{align}
	  We can now directly compute the gradient of $\mathcal{L}_{\theta}(E)$,
	  \begin{align}
	  	\nabla_{\theta}\mathcal{L}_{\theta}(E) =  \left( \int  \exp( -E(\B_{\theta_0})) \diff \eta \right)^{-1} \int \nabla_x E(\B_{\theta_0}) \nabla_{\theta} \B_{\theta_0}   \exp( -E(\B_{\theta_0})) \diff \eta.
	  \end{align}
	  Since $E$ and $\B_{\theta}$ are Lipschitz in $x$ and $\theta$ respectively, it follows that $ \Vert \nabla_x E(\B_{\theta_0}(z)) \Vert \leq L_e$ and $ \Vert \nabla_{\theta} \B_{\theta_0}(z)\Vert \leq a(z)$. Hence, we have
	  \begin{align}
	  	\Vert \nabla_{\theta}\mathcal{L}_{\theta}(E)\Vert \leq  L_e \int a(z)  (p_{E,\theta}\circ \B_{\theta}(z)) d\eta(z).
	  \end{align}
	  Finally, \cref{lem:integrability} allows us to conclude that $\Vert \nabla_{\theta}\mathcal{L}_{\theta}(E)\Vert$ is bounded by a positive constant $C$ independently from $\theta$ and $E$.
\end{proof}
\begin{lemma}\label{lem:integrability}
Under \cref{assump:l_smooth_critic,assump:l_smooth_gen}, there exists a constant $C$ independent from $\theta$ and $E$ such that
	\begin{align}\label{eq:inequalities_a_b}
	&\int  a^i(z)  (p_{E,\theta}\circ \B_{\theta}(z)) d\eta(z) < C,\\
		&\int b(z)(p_{E,\theta}\circ \B_{\theta}(z)) d\eta(z) <C,
	\end{align}
	for $i \in {1,2}$.
\end{lemma}
\begin{proof}
	By Lipschitzness of $E$ and $\B_{\theta}$, we have   $ \exp( -L_eL_b \Vert z \Vert)  \leq \exp( E(\B_{\theta}(0))-E(\B_{\theta}(z)) \leq \exp( L_eL_b \Vert z \Vert )$, thus introducing the factor $\exp(E(B_{\theta_0}(0))$  in \cref{eq:inequalities_a_b} we get 
		  \begin{align}
	  		  	\int  a^i(z)  (p_{E,\theta}\circ \B_{\theta}(z)) d\eta(e) &\leq  L_e \left( \int  \exp(-L_eL_b \Vert z \Vert ) \diff \eta(z) \right)^{-1} \int a(z)^i \exp( L_eL_b \Vert z \Vert  ) \diff \eta(z),\\
	  		  	\int b(z)(p_{E,\theta}\circ \B_{\theta}(z)) d\eta(z) &\leq L_e \left( \int  \exp(-L_eL_b \Vert z \Vert ) \diff \eta(z) \right)^{-1} \int b(z) \exp( L_eL_b \Vert z \Vert  ) \diff \eta(z).
	  \end{align}
	  The r.h.s. of both inequalities is independent of $\theta$ and $E,$ and finite by the integrability assumptions in \cref{assump:l_smooth_gen}.
\end{proof}

\section{Image Generation}\label{sec:image_generation_appendix}

\cref{fig:image_samples_lmc_1,fig:image_samples_lmc_2} show sample trajectories using \cref{alg:langevin} with no friction $\gamma=0$ for the 4 datasets. It is clear that along the same MCMC chain, several image modes are explored. We also notice the transition from a mode to another happens almost at the same time for all chains and corresponds to the gray images. This is unlike Langevin or when the friction coefficient $\gamma$ is large as in \cref{fig:image_samples_all}. In that case each chain remains within the same mode. 

\cref{tab:further_exps} shows further comparisons with other methods on Cifar10 and ImageNet 32x32.
\begin{table}
\center
\begin{tabular}{ll}
\midrule
Model                                   & FID   \\ \midrule
\textbf{Cifar10 Unsupervised}                   &       \\ \midrule
PixelCNN  \cite{Oord:2016}      & 65.93 \\
PixelIQN \cite{Ostrovski:2018}      & 49.46 \\
EBM  \cite{Radford:2015a}            & 38.2  \\
WGAN-GP \cite{Gulrajani:2017}        & 36.4  \\
NCSN  \cite{Ho:2016}          & 25.32 \\
SNGAN \cite{Miyato:2018}          & 21.7  \\
GEBM (ours)                                   & \textbf{19.31}   \\ \midrule
\textbf{Cifar10 Supervised}                     &       \\ \midrule
BigGAN {\small \cite{Donahue:2019} }                                & 14.73 \\ 
SAGAN \cite{Zenke:2017} & 13.4  \\\midrule
\textbf{ImageNet Conditional}              &       \\ \midrule
PixelCNN                                & 33.27 \\
PixelIQN                                & 22.99 \\
EBM                                     & 14.31 \\ \midrule
\textbf{ImageNet Supervised}            &       \\ \midrule
SNGAN                                   & 20.50 \\
GEBM  (ours)                            & \textbf{13.94}\\\midrule
\end{tabular}
\caption{FID scores on ImageNet and CIFAR-10.}
\label{tab:further_exps}
\end{table}

\begin{figure}[h]
\centering
  	\includegraphics[width=0.49\linewidth]{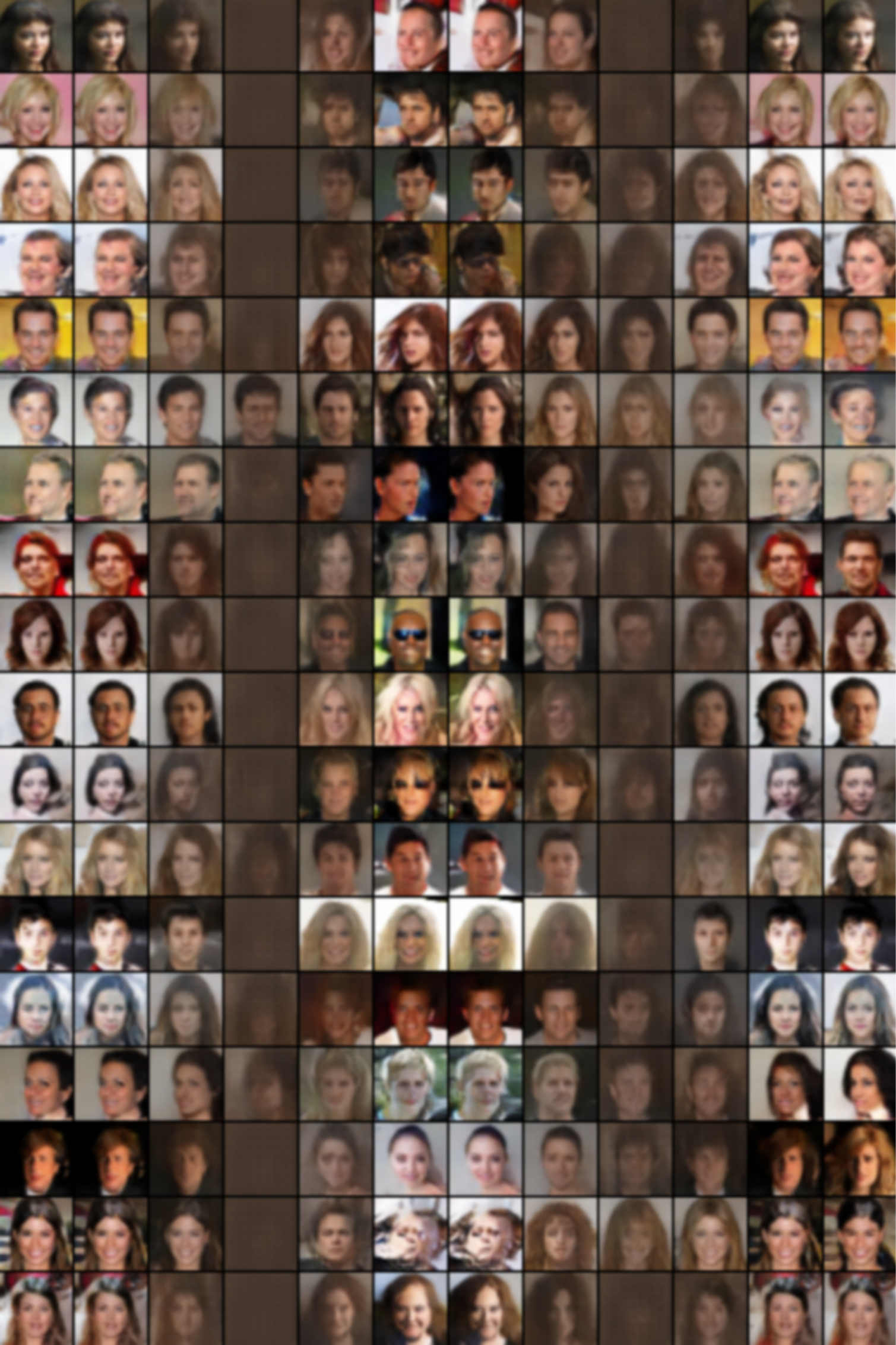}
	\includegraphics[width=0.49\linewidth]{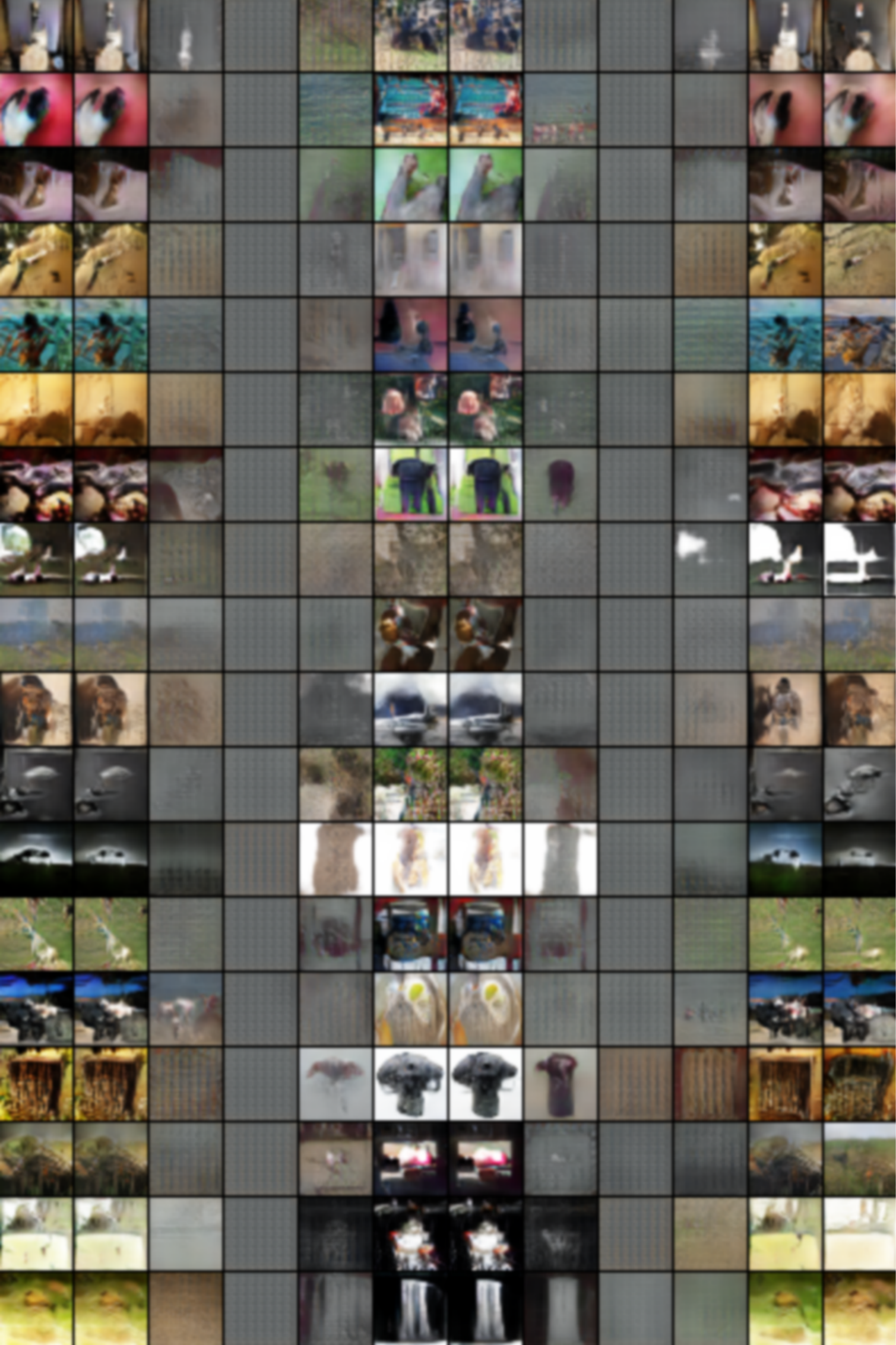}
\caption{Samples from the GEBM at different stages of sampling using \cref{alg:langevin} and inverse temperature $\beta=1$, on CelebA (Left), Imagenet (Right). Each row represents a sampling trajectory from early stages (leftmost images) to later stages (rightmost images).}
\label{fig:image_samples_lmc_1}
\end{figure}

  \begin{figure}[h]
\centering
  	\includegraphics[width=0.49\linewidth]{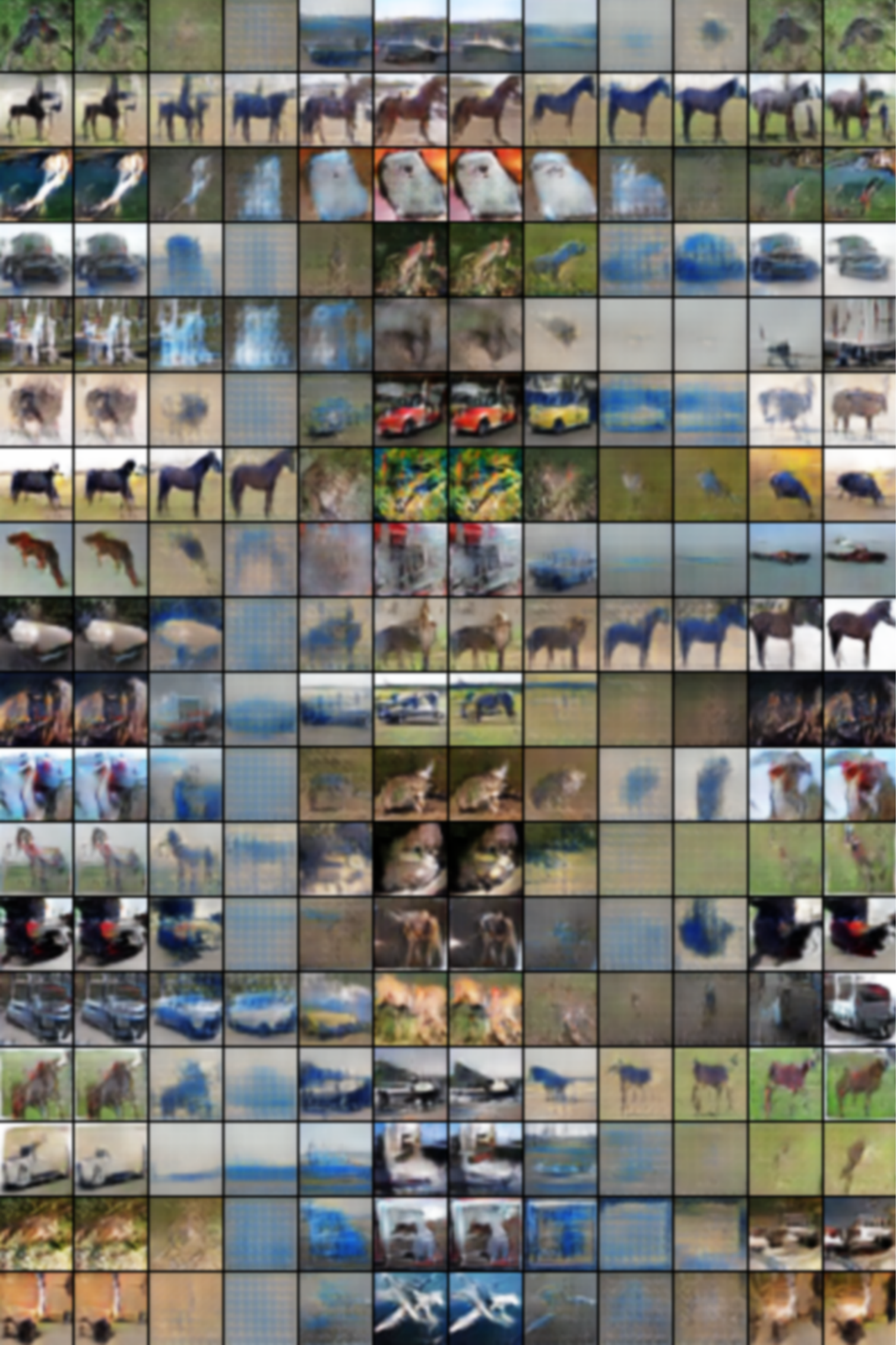}
	\includegraphics[width=0.49\linewidth]{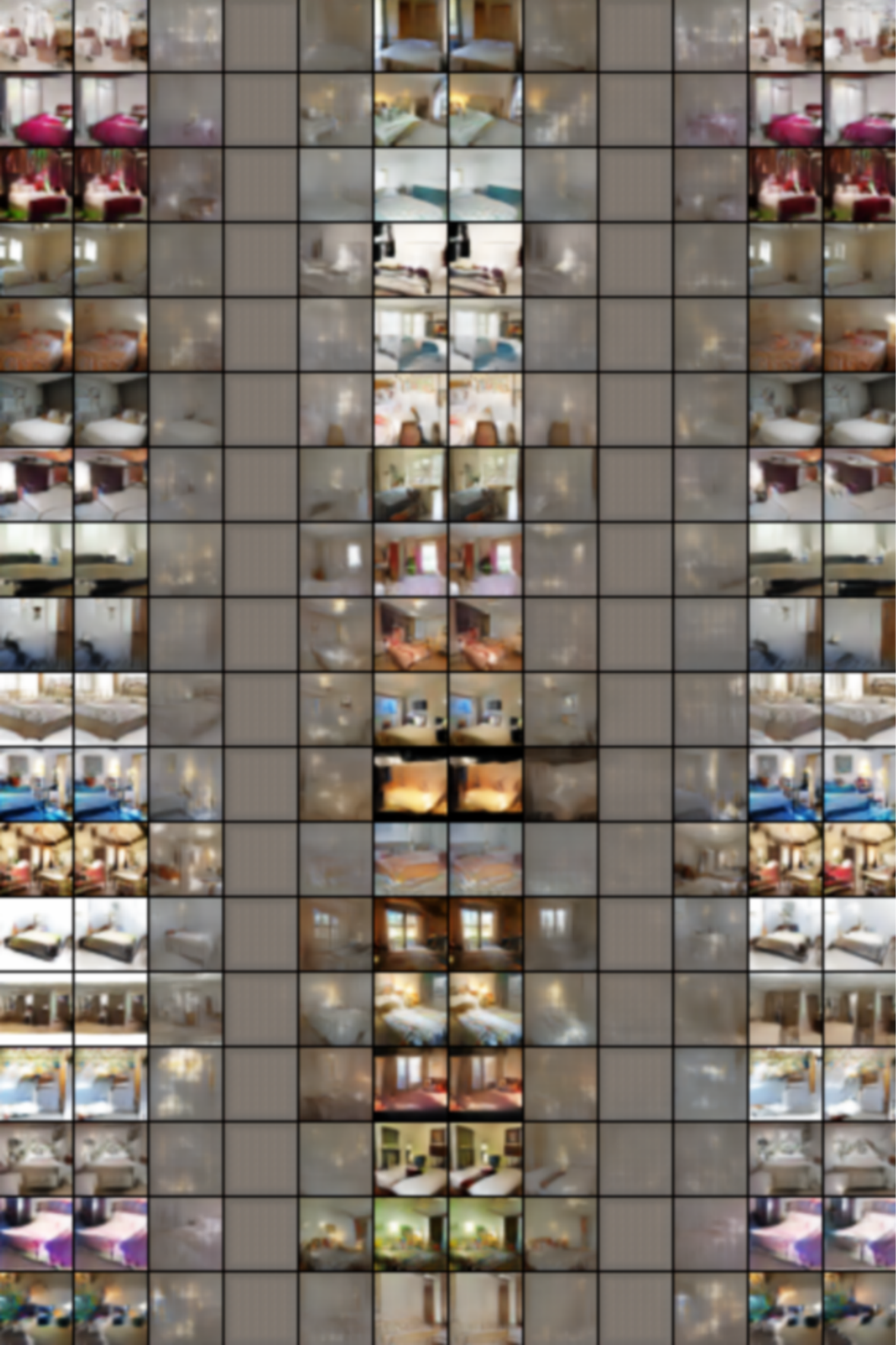}
\caption{Samples from the GEBM at different stages of sampling using \cref{alg:langevin} and inverse temperature $\beta=1$, on Cifar10 and LSUN (Right). Each row represents a sampling trajectory from early stages (leftmost images) to later stages (rightmost images).}
\label{fig:image_samples_lmc_2}
\end{figure}

\begin{figure}[h]
\centering
  	\includegraphics[width=0.24\linewidth]{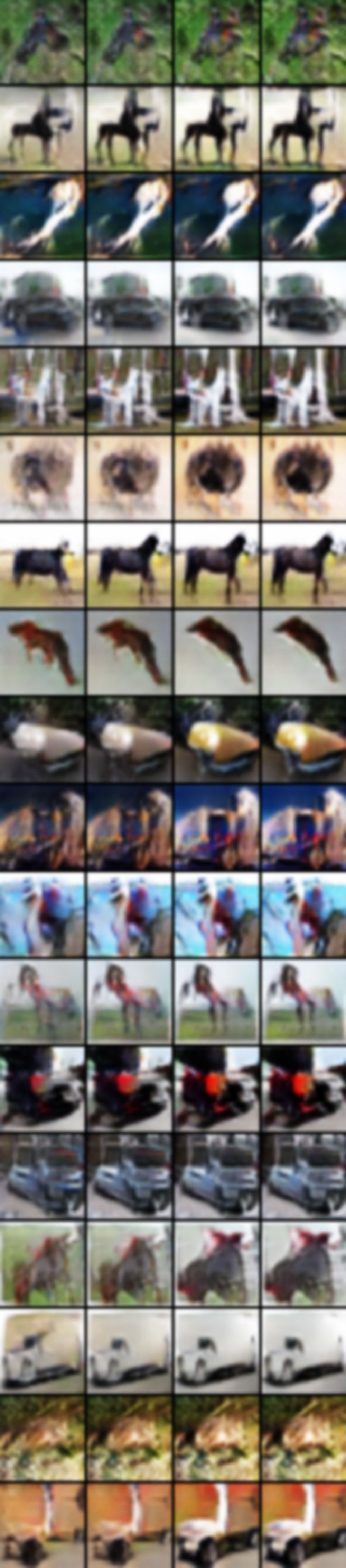}
	\includegraphics[width=0.24\linewidth]{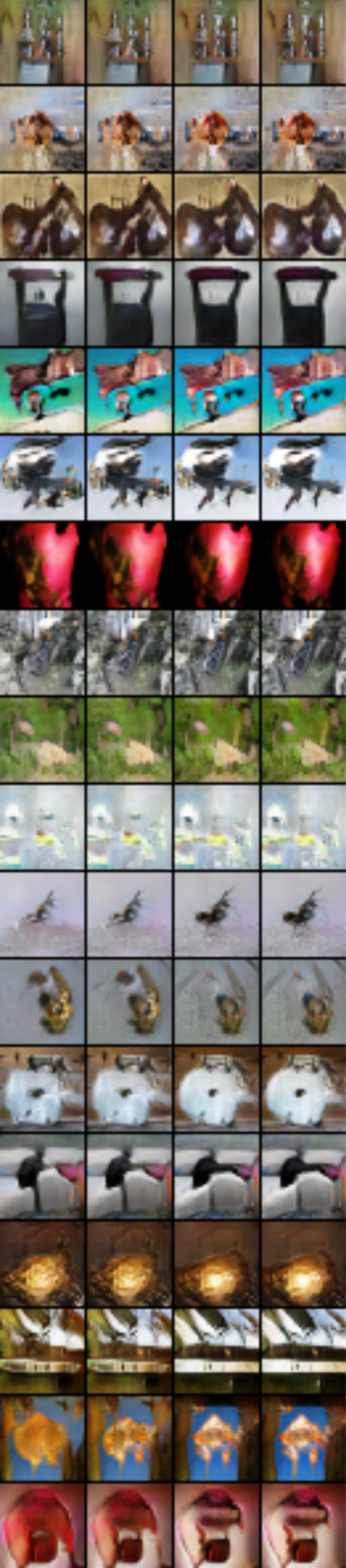}
	\includegraphics[width=0.24\linewidth]{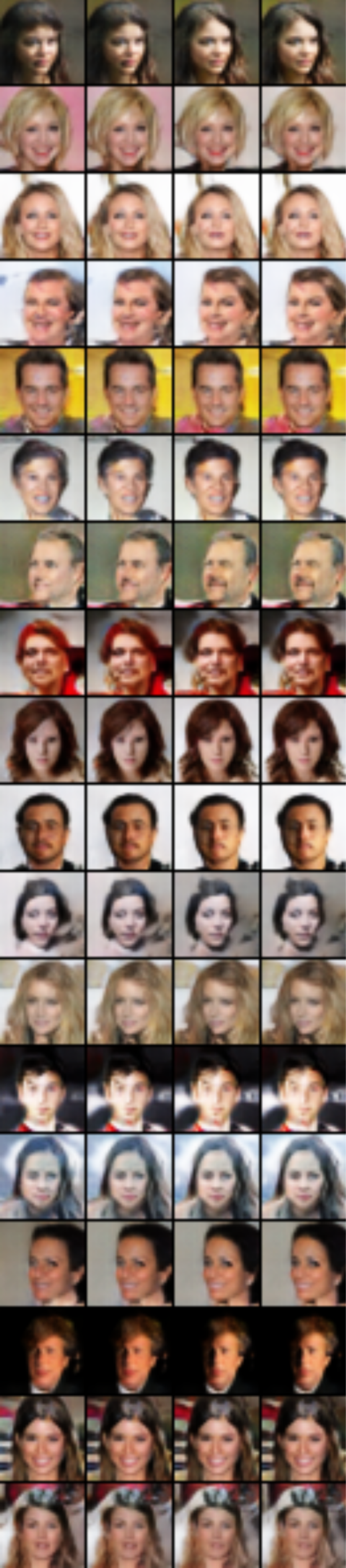}
	\includegraphics[width=0.24\linewidth]{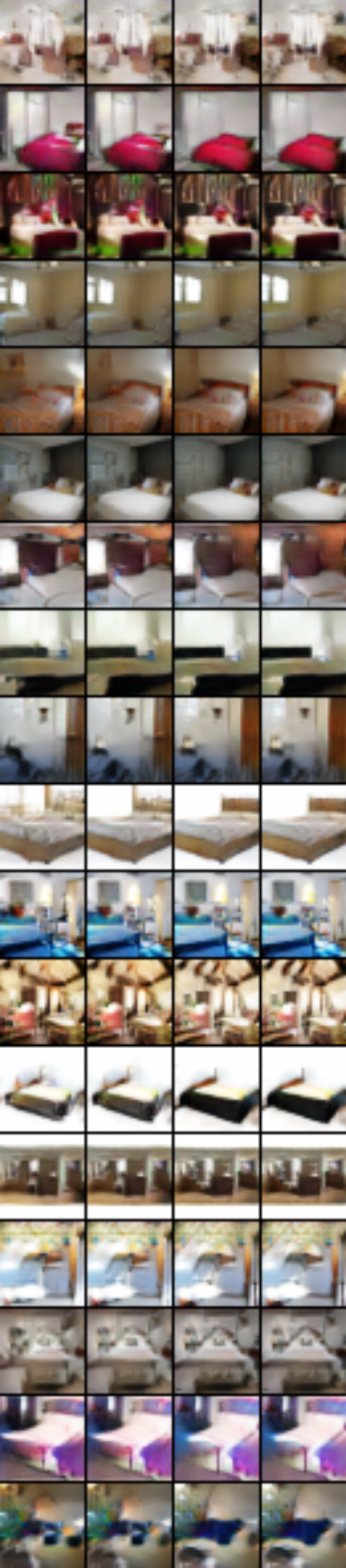}
\caption{Samples from the tempered GEBM at different stages of sampling using langevin and inverse temperature $\beta=100$, on Cifar10 (Left), Imagenet (Middle-left), CelebA (Middle-Right) and LSUN (Right). Each row represents a sampling trajectory from early stages (leftmost images) to later stages (rightmost images).}
\label{fig:image_samples_all}
\end{figure}

\begin{figure}[h]
\centering
  	\includegraphics[width=.7\linewidth]{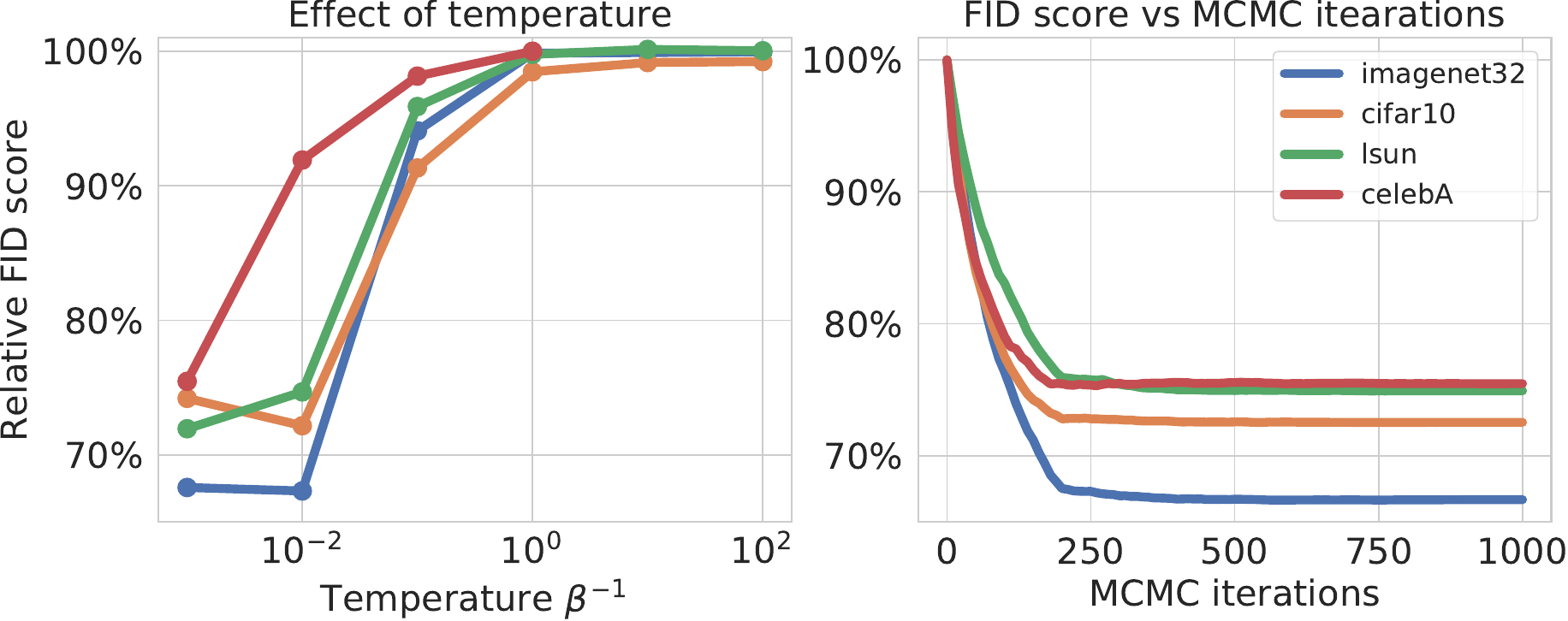}
\caption{ Relative FID score: ratio between  FID score of the GEBM $\QQ_{\BB,E}$ and its base $\BB$. (Left) Evolution of the ratio for increasing temperature on the 4 datasets after 1000 iterations of \cref{eq:latent_langevin_diffusion}. 
(Right) Evolution of the same ratio during MCMC iteration using  \cref{eq:latent_langevin_diffusion}. 
}
\label{fig:sampling}
\end{figure}

\section{Density Estimation}\label{sec:density_estimation_appendix}

Figure \cref{fig:UCI} (left) shows the error in the estimation of the log-partition function using both methods (KALE-DV and KALE-F). KALE-DV estimates the negative log-likelihood on each batch of size $100$ and therefore has much more variance than KALE-F which maintains the amortized estimator of the log-partition function.

Figure \cref{fig:UCI} (right) shows the evolution of the negative log-likelihood (NLL) on both training and test sets per epochs for RedWine and Whitewine datasets. The error decreases steadily in the case of KALE-DV and KALE-F while the error gap between the training and test set remains controlled. Larger gaps are observed for both direct maximum likelihood estimation and Contrastive divergence although the training NLL tends to decrease faster than for KALE.

\begin{figure}[h]
\centering
  	\includegraphics[width=.9\linewidth]{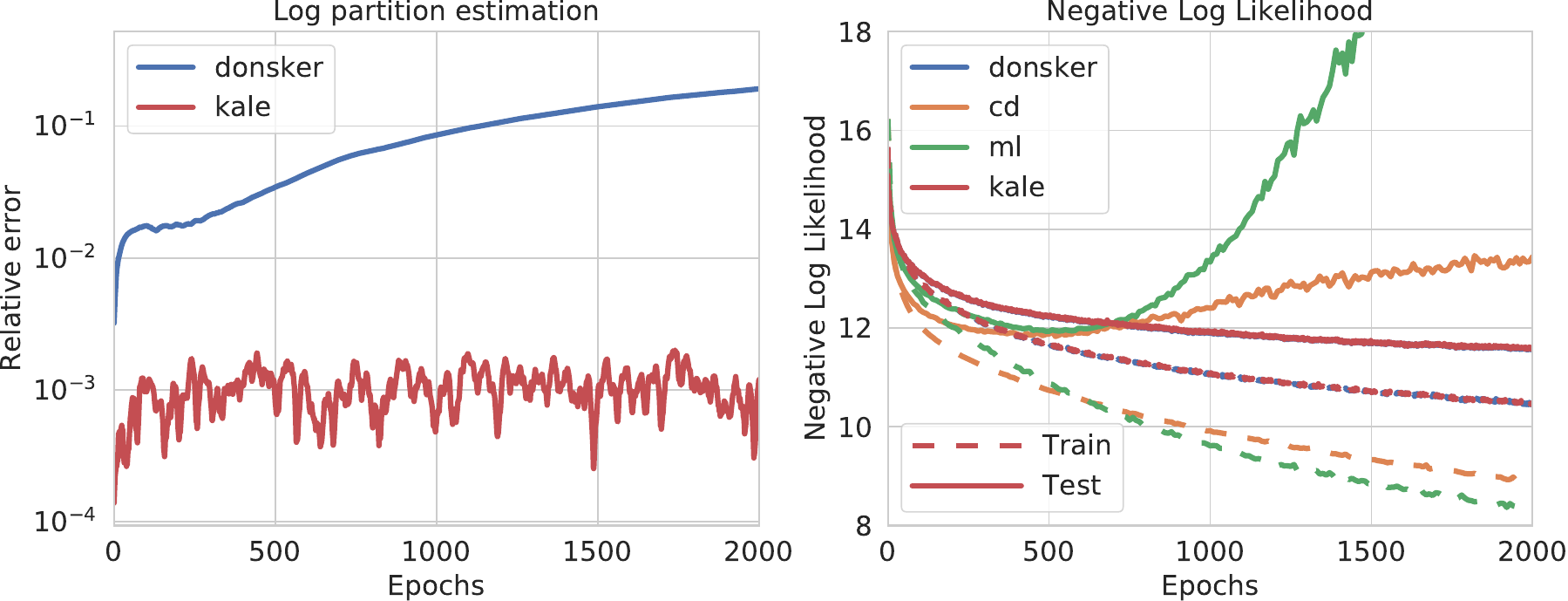}
  	
  	\includegraphics[width=.9\linewidth]{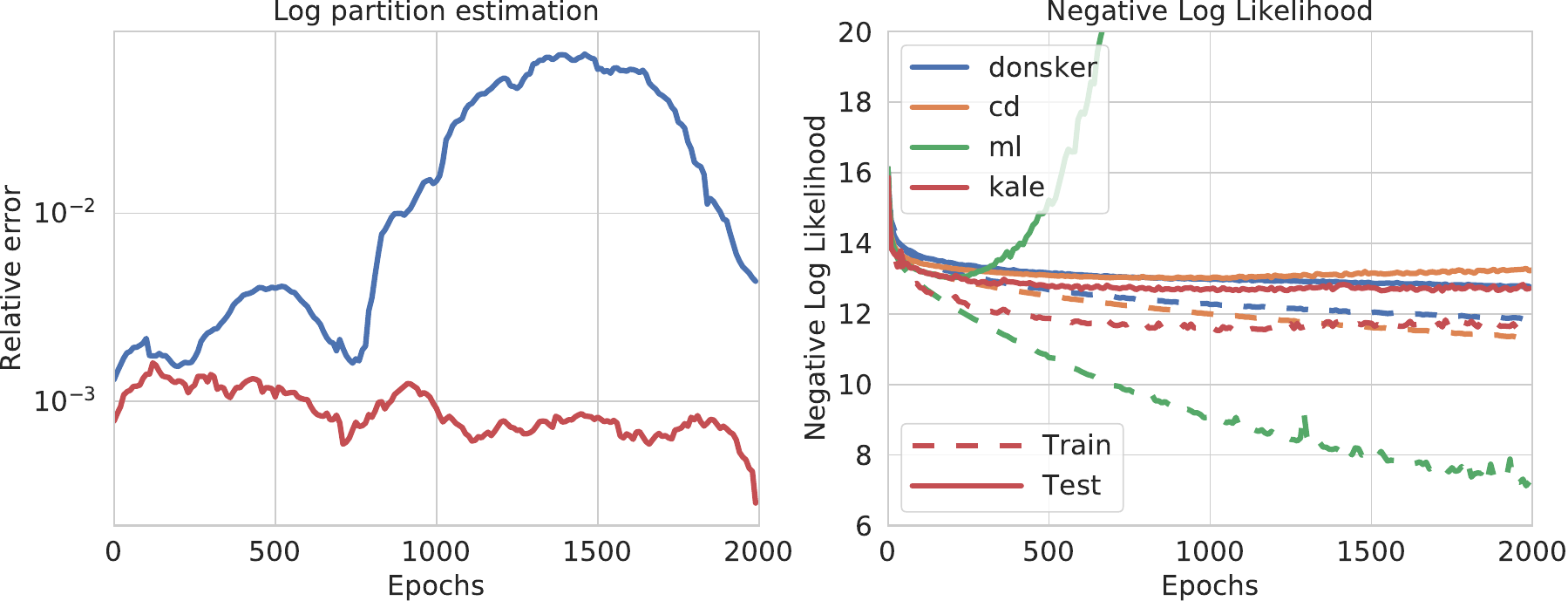}
\caption{(Left): Relative error  $\frac{\vert\hat{c} -c^{\star} \vert}{\vert \hat{c}\vert + \vert c^{\star}\vert }$ on the estimation of the ground truth log-partition function $c^{*}$ by $\hat{c}$ using either KALE-DV or KALE-F vs training Epochs on RedWine (Top) and WhiteWine (Bottom) datasets. (Right): Negative log likelihood  vs training epochs on both training and test set for 4 different learning methods (KALE-DV,KALE-F, CD and ML) on RedWine dataset.}
\label{fig:UCI}
\end{figure}

\section{Algorithms}\label{sec:Algorithms}
{\bf Estimating the variational parameter.} Optimizing \cref{eq:estimator_variational_lower_bound} exactly over $A$ yields \cref{eq:DV}, with the optimal $A$ equal to $\tilde{A} = \log(\frac{1}{M} \sum_{m=1}^M \exp(-E(Y_m))  ) $. However, to maintain an amortized estimator of the log-partition we propose to optimize  \cref{eq:estimator_variational_lower_bound} iteratively using second order updates:
\begin{align}\label{eq:second_order_update}
 A_{k+1}= A_k - \lambda ( \exp(A_k -\tilde{A}_{k+1})  -1  ),\qquad	A_0 = \tilde{A}_{0}
\end{align}
where $\lambda$ is a learning rate and $\tilde{A}_{k+1}$ is the empirical log-partition function estimated from a batch of new samples.
 By leveraging updates from previous iterations,  $A$ can  yield much more accurate estimates of the log-partition function as confirmed empirically in \cref{fig:UCI} of \cref{sec:density_estimation_appendix}. 

{\bf Tempered GEBM.} It can be preferable to sample from a \textit{tempered} version of the model by rescaling the energy $E$ by an \textit{inverse temperature} parameter $\beta$, thus effectively sampling from $\QQ$. \textit{High temperature} regimes ($\beta \rightarrow 0 $) recover the base model $\BB$  while \textit{low temperature} regimes ($ \beta \rightarrow \infty $) essentially sample from minima of the energy $E$. As shown in \cref{sec:experiments}, low temperatures tend to produce better sample quality for natural image generation tasks.
\paragraph{Training} In \cref{alg:kale_gan}, we describe the general algorithm for training a GEBM which alternates between gradient steps on the energy and the generator. An additional regularization, denoted by $I(\psi)$ is used to ensure conditions of \cref{prop:kale_extension,thm:kale_gan_convergence} hold. $I(\psi)$  can include $L_2$ regularization over the parameters $\psi$, a gradient penalty as in \cite{Gulrajani:2017} or Spectral normalization \cite{Miyato:2018}. The energy can be trained either using the estimator in \cref{eq:DV}  (KALE-DV) or the one in \cref{eq:estimator_variational_lower_bound} (KALE-F) depending on the variable $\mathcal{C}$.
  
\paragraph{Sampling} In \cref{alg:langevin}, we describe the MCMC sampler proposed in \cite{Sachs:2017}  which is a time discretization of \cref{eq:latent_langevin_diffusion}.

\begin{algorithm}[H]
\caption{Overdamped Langevin Algorithm}\label{alg:overdamped_langevin}
	\begin{algorithmic}[1]
		\STATE \textbf{Input}  $\lambda$, $\gamma$, $u$,$\eta$,$E$,$\B$		
		\STATE \textbf{Ouput} $X_T$
		\STATE $Z_0\sim \eta$  \textcolor{blue}{ // Sample Initial latent from $\eta$. }
		\FOR{$t=0,\dots, T$}
			\STATE $ Y_{t+1}   \leftarrow  \nabla_z \log \eta(Z_{t}) -  \nabla_z E\circ B (Z_{t})  $ \textcolor{blue}{ // Evaluating $\nabla_z \log(\nu(Z_{t+1})) $ using \cref{eq:posterior_latent}.}
			\STATE $W_{t+1}\sim \mathcal{N}(0,I) $  \textcolor{blue}{ // Sample standard Gaussian noise }
			\STATE $ Z_{t+1}\leftarrow Z_{t} +\lambda Y_{t+1} +\sqrt{2\lambda} W_{t+1}$
		\ENDFOR
		\STATE $X_T  \leftarrow \B(Z_T) $
	\end{algorithmic}
\end{algorithm}

\begin{algorithm}[H]
\caption{Kinetic Langevin Algorithm}\label{alg:langevin}
	\begin{algorithmic}[1]
		\STATE \textbf{Input}  $\lambda$, $\gamma$, $u$,$\eta$,$E$,$\B$		
		\STATE \textbf{Ouput} $X_T$
		\STATE $Z_0\sim \eta$  \textcolor{blue}{ // Sample Initial latent from $\eta$. }
		\FOR{$t=0,\dots, T$}
			\STATE $Z_{t + 1}  \leftarrow Z_t + \frac{\lambda}{2} V_{t}$
			\STATE $ Y_{t+1}   \leftarrow  \nabla_z \log \eta(Z_{t+1}) -  \nabla_z E\circ B (Z_{t+1})  $ \textcolor{blue}{ // Evaluating $\nabla_z \log(\nu(Z_{t+1})) $ using \cref{eq:posterior_latent}.}
			\STATE $V_{t +1} \leftarrow V_t  +\frac{u\lambda}{2}Y_{t+1} $.
			\STATE $W_{t+1}\sim \mathcal{N}(0,I) $  \textcolor{blue}{ // Sample standard Gaussian noise }
			\STATE $ \tilde{V}_{t+1} \leftarrow \exp(-\gamma\lambda ) V_{t+\frac{1}{2}} + \sqrt{u\left(1-\exp(-2\gamma \lambda )\right)} W_{t+1}$
			\STATE $  V_{t+1} \leftarrow \tilde{V}_{t+1} + \frac{u\lambda}{2}Y_{t+1} $
			\STATE $ Z_{t+1}\leftarrow Z_{t+1} + \frac{\lambda}{2}V_{t+1}$
		\ENDFOR
		\STATE $X_T  \leftarrow \B(Z_T) $
	\end{algorithmic}
\end{algorithm}
\section{Experimental details}\label{sec:exp_details}
In all experiments, we use \textbf{regularization} which is a combination of $L_2$ norm and a variant of the gradient penalty \cite{Gulrajani:2017}. For the image generation tasks, we also employ spectral normalization \cite{Miyato:2018}. This is to ensure that the conditions in \cref{prop:kale_extension,thm:kale_gan_convergence} hold. We \textbf{pre-condition} the gradient as proposed in \cite{Simsekli:2020} to stabilize training, and to avoid taking large noisy gradient steps due to the exponential terms in \cref{eq:DV,eq:estimator_variational_lower_bound}. We also use the second-order updates in \cref{eq:second_order_update} for the variational constant $c$ whenever it is learned.

\subsection{Image generation}\label{appendix:image_gen}

\paragraph{Network Architecture}

\cref{tab:sngan_convnet_base} and \cref{tab:sngan_convnet_energy} show the network architectures used for the GEBM in the case of SNGAN ConvNet. \cref{tab:sngan_convnet_base} and \cref{tab:sngan_convnet_energy} show the network architectures used for the GEBM in the case of SNGAN ResNet. The residual connections of each residual block consists of two convolutional layers proceeded by a BatchNormalization and ReLU activation:  \textbf{BN+ReLU+Conv+BN+ReLU+Conv} as in \cite[Figure 8]{Miyato:2018}. 
\begin{table}
\center
\begin{minipage}{0.4\linewidth}
\begin{tabular}{c}
\hhline{=}
$ z\in \mathbb{R}^{100}\sim \mathcal{N}(0,I) $ \\ \hline
dense $\rightarrow$  $M_g\times M_g\times 512$ \\ \hline
$4\times 4$, stride$=2$ deconv. BN 256 ReLU    \\ \hline
$4\times 4$, stride$=2$ deconv. BN 128 ReLU    \\ \hline
$4\times 4$, stride$=2$ deconv. BN 64 ReLU    \\ \hline
$3\times 3$, stride$=1$  conv. 3 Tanh          \\ \hhline{=}
\end{tabular}
\caption{Base/Generator of SNGAN ConvNet: $M_g=4$.}
\label{tab:sngan_convnet_base}
\end{minipage}\hspace{.15\linewidth}
\begin{minipage}{.35\linewidth}
\begin{tabular}{c}
\hhline{=} 
RGB image $x\in \mathbb{R}^{M\times M \times 3}$ \\ \hline
$3\times 3$, stride$=1$ conv $64$ lReLU          \\
$4\times 4$, stride$=2$ conv $64$ lReLU          \\ \hline
$3\times 3$, stride$=1$ conv $128$ lReLU         \\
$4\times 4$, stride$=2$ conv $128$ lReLU         \\ \hline
$3\times 3$, stride$=1$ conv $256$ lReLU         \\
$4\times 4$, stride$=2$ conv $256$ lReLU         \\ \hline
$3\times 3$, stride$=1$ conv $512$ lReLU         \\ \hline
dense $\rightarrow 1$.                           \\ \hhline{=}
\end{tabular}
\caption{Energy/Discriminator of SNGAN ConvNet: $M=32$.}
\label{tab:sngan_convnet_energy}
\end{minipage}
\end{table}

\begin{table}[]
\center
\begin{minipage}{.3\linewidth}
\begin{tabular}{c}
\hhline{=}
RGB image $x\in \mathbb{R}^{M\times M \times 3}$ \\ \hline
ResBlock down $128$                              \\ \hline
ResBlock down $128$                              \\ \hline
ResBlock $128$                                   \\ \hline
ResBlock $128$                                   \\ \hline
ReLu                                             \\ \hline
Global sum pooling                               \\ \hline
dense $\rightarrow 1$                            \\ \hhline{=}
\end{tabular}
\caption{Energy/Discriminator of SNGAN ResNet.}
\label{tab:sngan_Resnet_energy}
\end{minipage}\hspace{.15\linewidth}
\begin{minipage}{.3\linewidth}
\begin{tabular}{l}
\hhline{=}
$z\in \mathbb{R}^{100}\sim\mathcal{N}(0,I)$ \\ \hline
dense, $4\times 4\times 256$                \\ \hline
ResBlock up $256$                           \\ \hline
ResBlock up $256$                           \\ \hline
ResBlock  up $256$                          \\ \hline
BN, ReLu, $3\times 3$ conv, Tanh            \\ \hhline{=}
\end{tabular}
\caption{Base/Generator of SNGAN ResNet.}
\label{tab:sngan_Resnet_base}
\end{minipage}
\end{table}

\paragraph{Training:} We train both base and energy by alternating $5$ gradient steps to learn the energy vs $1$ gradient step to learn the base. For the first two gradient iterations and after every $500$ gradient iterations on base, we train the energy for $100$ gradient steps instead of $5$.   We then train the model up to $150000$ gradient iterations on the base using a batch-size of $128$ and Adam optimizer \cite{Kingma:2014} with initial learning rate of $10^{-4}$ and parameters $(0.5,.999)$ for both energy and base. 

\paragraph{Scheduler:} We decrease the learning rate using a scheduler that monitors the FID score in a similar way as in \cite{Binkowski:2018,Arbel:2018}. More precisely,  every $2000$ gradient iterations on the base, we evaluate the FID score on the training set using $50000$ generated samples from the base and check if the current score is larger than the score $20000$ iterations before. The learning rate is decreased by a factor of $0.8$ if the FID score fails to decrease for $3$ consecutive times.

\paragraph{Sampling:} 
For (DOT) \cite{Tanaka:2019}, we use the following objective:
\begin{align}\label{eq:dot}
	z\mapsto \Vert  z- z_y  + \epsilon  \Vert + \frac{1}{k_{eff}} E\circ \B(z)
\end{align}
where $z_y$ is sampled from a standard Gaussian, $\epsilon$ is a perturbation meant to stabilize sampling and $k_{eff}$ is the estimated Lipschitz constant of $E\circ B$.  Note that \cref{eq:dot} uses a flipped sign for the $E\circ B $ compared to \cite{Tanaka:2019}. This is because $E$ plays the role of $-D$ where $D$ is the discriminator in \cite{Tanaka:2019}. Introducing the minus sign in \cref{eq:dot} leads to a degradation in performance.
We perform $1000$ gradient iterations with a step-size of $0.0001$ which is also decreased by a factor of $10$ every $200$ iterations as done for the proposed method. As suggested by the authors of  \cite{Tanaka:2019} we perform the following projection for the gradient before applying it:
\begin{align}
	g \leftarrow g- \frac{(g^{\top}z)}{\sqrt{q}}z.
\end{align}
We set the perturbation $\epsilon$ to $0.001$ and $k_{eff}$ to $1$ which was also shown in \cite{Tanaka:2019} to perform well. In fact, we found that estimating the Lipschitz constant  by taking the maximum value of $ \Vert \nabla E\circ \B(z) \Vert $ over $1000$ latent samples according to $\eta$  lead to higher values for $k_{eff}$: ( Cifar10: $9.4$, CelebA : $7.2$, ImageNet: $4.9$, Lsun: $3.8$). However, those higher values did not perform as well as setting $k_{eff}=1$.

For (IHM) \cite{Turner:2019a} we simply run the MCMC chain for $1000$ iterations.

\subsection{Density estimation}\label{sec:density_estimation}

\paragraph{Pre-processing} We use code and pre-processing steps from \cite{Wenliang:2018} which we describe here for completeness. 
For RedWine and WhiteWine, we added uniform noise with support equal to the median distances between two adjacent values. That is to avoid instabilities due to the quantization of the datasets.  For Hepmass and MiniBoone, we removed ill-conditioned dimensions as also done in \cite{Papamakarios:2017}. We split all datasets, except HepMass into three splits. The test split consists of $10\%$ of the total data. For the validation set, we use $10\%$ of the remaining data with an upper limit of $1000$ to reduce the cost of validation at each iteration. For HepMass, we used the sample splitting as done in \cite{Papamakarios:2017}. Finally, the data is whitened before fitting and the whitening matrix was computed on at most $10000$ data points.  
\paragraph{Regularization:} We set the regularization parameter to $0.1$ and use a combination of $L_2$ norm and a variant of the gradient penalty \cite{Gulrajani:2017}:
\begin{align}\label{eq:reg_penalty}
	I(\psi)^{2} = \frac{1}{d_{\psi}}\Vert \psi \Vert^2 + \mathbb{E}\left[\Vert \nabla_x f_{\psi}(\widetilde{X})\Vert^2\right]
\end{align}
\paragraph{Network Architecture.} For both base and energy, we used an NVP \cite{Dinh:2016} with 5 NVP layers each consisting of a shifting and scaling layer with two hidden layers of $100$ neurons. We do not use Batch-normalization.
\paragraph{Training:} In all cases we use Adam optimizer with learning rate of $0.001$ and momentum parameters $(0.5,0.9)$.
For both KALE-DV and KALE-F, we used a batch-size of $100$ data samples vs $2000$ generated samples from the base in order to reduce the variance of the estimation of the energy. 
 We alternate $50$ gradient steps on the energy vs $1$ step on the base and further perform $50$ additional steps on the energy for the first two gradient iterations and after every $500$ gradient iterations on base.
For Contrastive divergence, each training step is performed by first producing $100$ samples from the model using $100$ Langevin iterations with a step-size of $10^{-2}$ and starting from a batch of $100$ data-samples. The resulting samples are then used to estimate the gradient of the of the loss. 

For (CD), we used 100 Langevin iterations for each learning step to sample from the EBM. This translates into an improved performance at the expense of increased computational cost compared to the other methods. 
All methods are trained for 2000 epochs with batch-size of 100 (1000 on Hepmass and Miniboone datasets) and fixed learning rate $0.001$, which was sufficient for convergence.

\subsection{Illustrative example in \cref{fig:image_samples} }\label{sec:illustrative_example}

We consider parametric functions $G_{\theta}^{(1)}$ and $G_{\theta}^{(2)}$ from $\mathbb{R}$ to $\mathbb{R}$ of the form:
\begin{align}
	G_{\theta}^{(1)}(x) = sin(8\pi Wx)/(1+4\pi Bx), \qquad G_{\theta}^{(2)}(x)= 4\pi W'x+b
\end{align}
with $\theta = (W,B,W',b)$. we also call $\theta^{\star} = (1,1,1,0)$.
In addition, we consider a sigmoid like function $h$ from $[0,1]$ to $[0,1]$ of the form:
\begin{align}
	\widetilde{z} = tan( \pi (z -\frac{1}{2}  ),\qquad
	h(z) = \frac{1}{2} \left( z + \frac{1}{1+\exp(-9\widetilde{z})}  \right).
\end{align}

\paragraph{Data generation}:
To generate a data point $X=(X_1,X_2)$, we consider the following simple generative model:
	\begin{itemize}
		\item Sample a uniform r.v. $Z$ from $[0,1]$.
		\item Apply the distortion function $h$ to get a latent sample $Y= h(Z)$.
		\item Generate point $X$ using $X_1 =G_{\theta^{\star}}^{(1)}(Y) $ and $X_2 =G_{\theta^{\star}}^{(2)}(Y) $.
	\end{itemize} 
Hence, the data are supported on the $1$-d line defined by the equation $X_2 = G_{\theta^{\star}}^{(2)}(X_1)$.

\paragraph{GAN}
 For the generator we sample $Z$ uniformly from $[0,1]$ then generate  a sample $(X_1,X_2) = ( G_{\theta}^{(1)}(Z),     G_{\theta}^{(2)}(Z) )$. The goal is to learn $\theta$.

For the discriminator, we used an MLP with $6$ layers and $10$ hidden units.
\paragraph{GEBM}
For the base we use the same generator as in the GAN model. For the  energy we use the same MLP as discriminator of the GAN model.
\paragraph{EBM}
To ensure tractability of the likelihood, we use the following model:
\begin{align}
	X_2|X_1\sim  \mathcal{N}(G_{\theta}^{(2)}(X_1),\sigma_0)\\
	X_1 \sim MoG((\mu_1,\sigma_1),(\mu_2,\sigma_2))
\end{align}
$MoG((\mu_1,\sigma_1),(\mu_2,\sigma_2))$ refers to a Mixture of two gaussians with mean and variances $\mu_i$ and $\sigma_i$. We learn each of the parameters $(\theta, \sigma_0,\mu_1,\sigma_1,\mu_2,\sigma_2)$ by maximizing the likelihood.

Both GAN and GEBM have the capacity to recover the the exact support by finding the optimal parameter $\theta^{\star}$. 
For the EBM, when $\theta=\theta^{\star}$, the mean $G_{\theta^{\star}}(X_1)$ of the conditional gaussian $X_2|X_1$  draws a line which matches the data support exactly, i.e.: $X_2 = G_{\theta^{\star}}^{(2)}(X_1)$.

\subsection{Base/Generator complexity}\label{sec:gen_complexity}

To investigate the effect of model complexity of the performance gap between GANs and GEBMs, we performed additional experiments using the setting of \cref{fig:image_samples}. 
Now we allow the generator/base network to better model the hidden transformation $h$ that produces the first coordinate $X_1$ given the latent noise. We choose $G_{\theta}^{(1)}$ to be either a one hidden layer network or an MLP with 3 hidden layers both with leaky ReLU activation, 
instead of a simple linear transform as previously done in \cref{sec:illustrative_example}.
The network has universal approximation capability that depends on the number of units. This provides a direct control over the complexity of the generator/base. We then varied the number of hidden units from $1$ to $5*10^4$ units for the one hidden layer network and from $10$ to $5*10^3$ units per layer for the MLP. Note that the MLP with $5*10^3$ units per layer stores a matrix of size $2.5*10^7$ and thus contains 2 orders of magnitudes more parameters than the widest shallow network with $5*10^4$ units.
 We then compared the performance of the GAN and GEBM using the Sinkhorn divergence \cite{Feydy:2019} between each model and the data distribution. In all  experiments, we used the same discriminator/energy network described in \cref{sec:illustrative_example}. Results  are provided in \cref{fig:gen_complexity}.
  \begin{figure}[h]
\centering
  	\includegraphics[width=0.49\linewidth]{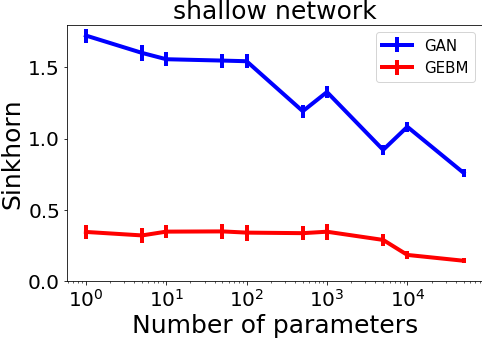}
  	\includegraphics[width=0.49\linewidth]{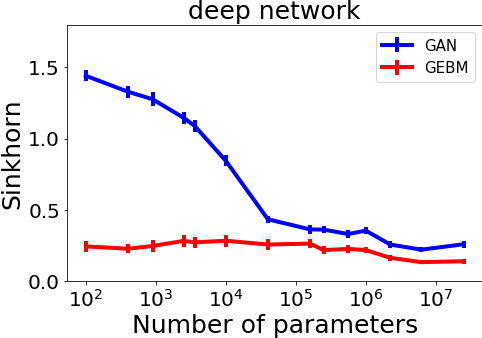}
\caption{ Sinkhorn divergence between data generating distribution and the trained model (either GAN or GEBM) vs number of hidden units in $G_{\theta}^{(1)}$. The left figure represents the one hidden layer network and right one is for the MLP. Each data point represents the average over 20 independent runs for each choice of number of hidden units.} 
\label{fig:gen_complexity}   
 
\end{figure}

\paragraph{Estimating the Sinkhorn divergence.}
The Sinkhorn is computed using $6000$ samples from the data and the model, with squared euclidean distance as a ground cost and using a regularization $\epsilon=1e-3$. We then repeat the procedure $5$ times and average the result to get the final estimate of the Sinkhorn distance for a given run.
\paragraph{Training.}
Each run optimizes the parameters of the model using Adam optimizer $(\beta_1 = .5, \beta_2=.99)$, learning rate $lr= 1e-4$ for the energy/discriminator and $lr=1e-5$ for the base/generator and weight decay of $1e-2$ for the base/generator. Training is performed using KALE for $2000$ epochs using a batch size of $5000$ and $10$   gradient iterations for the energy/discriminator per base/generator iteration. We use the gradient penalty for the energy/discriminator with a penalty parameter of $0.01$. We then perform early stopping and retain the best performing model on a validation set.
 
\paragraph{Observations}
We make the following observations from  \cref{fig:gen_complexity}: the GAN generator indeed improves when we increase the number of hidden units. The performance of the GEBM remains stable as the number of hidden units increases. The performance of the GEBM is always better than the GAN, although we can see the GAN converging towards the GEBM.   GEBM with a simpler base already outperforms the GAN with more powerful generators.  The gap between the GEBM and the GAN reduces as the GAN becomes more expressive. Using a deeper network further reduces the gap compared to a shallow network.

These observations support the prior discussion that the energy witnesses a remaining difference between the generator and training samples, as long as it is not flat. This information allows the GEBM to perform better than a GAN that ignores it. The performance gap between the GEBM and the GAN reduces as the generator becomes more powerful and forces the energy to be more flat. This is consistent with the result in \cref{prop:kl_improvement}.
\end{document}